\documentclass[a4paper,journal,onecoulumn,a4paper]{article}
 \usepackage{xr-hyper} 
\usepackage[colorlinks=true,pdfstartview=FitH,citecolor=blue,urlcolor=blue,linkcolor=blue]{hyperref}
 \usepackage{tikz,pgfplots}
\usetikzlibrary{hobby,matrix,patterns,calc,decorations.fractals,arrows,external,pgfplots.colormaps}
\usepackage{graphicx} 
\usepackage{array}
\usepackage{arydshln}
\usepackage{xcolor,colortbl}
\usepackage{cite}
\usepackage{appendix}
\usepackage{mathrsfs}
\usepackage{footmisc}
\usepackage{mathtools}
\usepackage{comment}
\usepackage{tikz}
\usepackage{tikz-cd}
\usepackage{accents}
\usepackage{amsmath}
\usepackage{amssymb}
\usepackage{amsfonts}
\usepackage{amsthm}
\usepackage{mathrsfs}
\usepackage{xspace}
\usepackage{bm}
\usepackage{fancyref}
\usepackage{textcomp}
\usepackage[Symbol]{upgreek}
\usepackage{mleftright}
\usepackage{enumerate}
\usepackage{mathtools}

\makeatletter
\if@twocolumn

\else

\fi
\makeatother

\newcommand{\safemath}[2]{\newcommand{#1}{\ensuremath{#2}\xspace}}

\newcommand{\diam}{\operatorname{diam}}

\newcommand{\ssa}{\mathsf{a}}
\newcommand{\ssb}{\mathsf{b}}
\newcommand{\ssc}{\mathsf{c}}
\newcommand{\ssd}{\mathsf{d}}
\newcommand{\sse}{\mathsf{e}}
\newcommand{\ssf}{\mathsf{f}}
\newcommand{\ssg}{\mathsf{g}}
\newcommand{\ssh}{\mathsf{h}}
\newcommand{\ssi}{\mathsf{i}}
\newcommand{\ssj}{\mathsf{j}}
\newcommand{\ssk}{\mathsf{k}}
\newcommand{\ssl}{\mathsf{l}}
\newcommand{\ssm}{\mathsf{m}}
\newcommand{\ssn}{\mathsf{n}}
\newcommand{\sso}{\mathsf{o}}
\newcommand{\ssp}{\mathsf{p}}
\newcommand{\ssq}{\mathsf{q}}
\newcommand{\ssr}{\mathsf{r}}
\newcommand{\sss}{\mathsf{s}}
\newcommand{\sst}{\mathsf{t}}
\newcommand{\ssu}{\mathsf{u}}
\newcommand{\ssv}{\mathsf{v}}
\newcommand{\ssw}{\mathsf{w}}
\newcommand{\ssx}{\mathsf{x}}
\newcommand{\ssy}{\mathsf{y}}
\newcommand{\ssz}{\mathsf{z}}

\safemath{\bmsa}{\bm{\ssa}}
\safemath{\bmsb}{\bm{\ssb}}
\safemath{\bmsc}{\bm{\ssc}}
\safemath{\bmsd}{\bm{\ssd}}
\safemath{\bmse}{\bm{\sse}}
\safemath{\bmsf}{\bm{\ssf}}
\safemath{\bmsg}{\bm{\ssg}}
\safemath{\bmsh}{\bm{\ssh}}
\safemath{\bmsi}{\bm{\ssi}}
\safemath{\bmsj}{\bm{\ssj}}
\safemath{\bmsk}{\bm{\ssk}}
\safemath{\bmsl}{\bm{\ssl}}
\safemath{\bmsm}{\bm{\ssm}}
\safemath{\bmsn}{\bm{\ssn}}
\safemath{\bmso}{\bm{\sso}}
\safemath{\bmsp}{\bm{\ssp}}
\safemath{\bmsq}{\bm{\ssq}}
\safemath{\bmsr}{\bm{\ssr}}
\safemath{\bmss}{\bm{\sss}}
\safemath{\bmst}{\bm{\sst}}
\safemath{\bmsu}{\bm{\ssu}}
\safemath{\bmsv}{\bm{\ssv}}
\safemath{\bmsw}{\bm{\ssw}}
\safemath{\bmsx}{\bm{\ssx}}
\safemath{\bmsy}{\bm{\ssy}}
\safemath{\bmsz}{\bm{\ssz}}

\bmdefine{\bmualphad}{\upalpha}
\bmdefine{\bmubetad}{\upbeta}
\bmdefine{\bmuchid}{\upchi}
\bmdefine{\bmudeltad}{\updelta}
\bmdefine{\bmuepsilond}{\upepsilon}
\bmdefine{\bmuvarepsilond}{\upvarepsilon}
\bmdefine{\bmuetad}{\upeta}
\bmdefine{\bmugammad}{\upgamma}
\bmdefine{\bmuiotad}{\upiota}
\bmdefine{\bmukappad}{\upkappa}
\bmdefine{\bmulambdad}{\uplambda}
\bmdefine{\bmumud}{\upmu}
\bmdefine{\bmunud}{\upnu}
\bmdefine{\bmuomegad}{\upomega}
\bmdefine{\bmuphid}{\upphi}
\bmdefine{\bmuvarphid}{\upvarphi}
\bmdefine{\bmupid}{\uppi}
\bmdefine{\bmuvarpid}{\upvarpi}
\bmdefine{\bmupsid}{\uppsi}
\bmdefine{\bmurhod}{\uprho}
\bmdefine{\bmuvarrhod}{\upvarrho}
\bmdefine{\bmusigmad}{\upsigma}
\bmdefine{\bmuvarsigmad}{\upvarsigma}
\bmdefine{\bmutaud}{\uptau}
\bmdefine{\bmuthetad}{\uptheta}
\bmdefine{\bmuvarthetad}{\upvartheta}
\bmdefine{\bmuupsilond}{\upupsilon}
\bmdefine{\bmuxid}{\upxi}
\bmdefine{\bmuzetad}{\upzeta}

\safemath{\bmua}{\mathbf{a}}
\safemath{\bmub}{\mathbf{b}}
\safemath{\bmuc}{\mathbf{c}}
\safemath{\bmud}{\mathbf{d}}
\safemath{\bmue}{\mathbf{e}}
\safemath{\bmuf}{\mathbf{f}}
\safemath{\bmug}{\mathbf{g}}
\safemath{\bmuh}{\mathbf{h}}
\safemath{\bmui}{\mathbf{i}}
\safemath{\bmuj}{\mathbf{j}}
\safemath{\bmuk}{\mathbf{k}}
\safemath{\bmul}{\mathbf{l}}
\safemath{\bmum}{\mathbf{m}}
\safemath{\bmun}{\mathbf{n}}
\safemath{\bmuo}{\mathbf{o}}
\safemath{\bmup}{\mathbf{p}}
\safemath{\bmuq}{\mathbf{q}}
\safemath{\bmur}{\mathbf{r}}
\safemath{\bmus}{\mathbf{s}}
\safemath{\bmut}{\mathbf{t}}
\safemath{\bmuu}{\mathbf{u}}
\safemath{\bmuv}{\mathbf{v}}
\safemath{\bmuw}{\mathbf{w}}
\safemath{\bmux}{\mathbf{x}}
\safemath{\bmuy}{\mathbf{y}}
\safemath{\bmuz}{\mathbf{z}}

\safemath{\bmualpha}{\bmualphad}
\safemath{\bmubeta}{\bmubetad}
\safemath{\bmuchi}{\bumchid}
\safemath{\bmudelta}{\bmudeltad}
\safemath{\bmuepsilon}{\bmuepsilond}
\safemath{\bmuvarepsilon}{\bmuvarepsilond}
\safemath{\bmueta}{\bmuetad}
\safemath{\bmugamma}{\bmugammad}
\safemath{\bmuiota}{\bmuiotad}
\safemath{\bmukappa}{\bmukappad}
\safemath{\bmulambda}{\bmulambdad}
\safemath{\bmumu}{\bmumud}
\safemath{\bmunu}{\bmunud}
\safemath{\bmuomega}{\bmuomegad}
\safemath{\bmuphi}{\bmuphid}
\safemath{\bmuvarphi}{\bmuvarphid}
\safemath{\bmupi}{\bmupid}
\safemath{\bmuvarpi}{\bmuvarpid}
\safemath{\bmupsi}{\bmupsid}
\safemath{\bmurho}{\bmurhod}
\safemath{\bmuvarrho}{\bmuvarrhod}
\safemath{\bmusigma}{\bmusigmad}
\safemath{\bmuvarsigma}{\bmuvarsigmad}
\safemath{\bmutau}{\bmutaud}
\safemath{\bmutheta}{\bmuthetad}
\safemath{\bmuvartheta}{\bmuvarthetad}
\safemath{\bmuupsilon}{\bmuupsilond}
\safemath{\bmuxi}{\bmuxid}
\safemath{\bmuzeta}{\bmuzetad}

\bmdefine{\bmiad}{a}
\bmdefine{\bmibd}{b}
\bmdefine{\bmicd}{c}
\bmdefine{\bmidd}{d}
\bmdefine{\bmied}{e}
\bmdefine{\bmifd}{f}
\bmdefine{\bmigd}{g}
\bmdefine{\bmihd}{h}
\bmdefine{\bmiid}{i}
\bmdefine{\bmijd}{j}
\bmdefine{\bmikd}{k}
\bmdefine{\bmild}{l}
\bmdefine{\bmimd}{m}
\bmdefine{\bmind}{n}
\bmdefine{\bmiod}{o}
\bmdefine{\bmipd}{p}
\bmdefine{\bmiqd}{q}
\bmdefine{\bmird}{r}
\bmdefine{\bmisd}{s}
\bmdefine{\bmitd}{t}
\bmdefine{\bmiud}{u}
\bmdefine{\bmivd}{v}
\bmdefine{\bmiwd}{w}
\bmdefine{\bmixd}{x}
\bmdefine{\bmiyd}{y}
\bmdefine{\bmizd}{z}

\bmdefine{\bmialphad}{\alpha}
\bmdefine{\bmibetad}{\beta}
\bmdefine{\bmichid}{\chi}
\bmdefine{\bmideltad}{\delta}
\bmdefine{\bmiepsilond}{\epsilon}
\bmdefine{\bmivarepsilond}{\varepsilon}
\bmdefine{\bmietad}{\eta}
\bmdefine{\bmigammad}{\gamma}
\bmdefine{\bmiiotad}{\iota}
\bmdefine{\bmikappad}{\kappa}
\bmdefine{\bmivarkappad}{\varkappa}
\bmdefine{\bmilambdad}{\lambda}
\bmdefine{\bmimud}{\mu}
\bmdefine{\bminud}{\nu}
\bmdefine{\bmiomegad}{\omega}
\bmdefine{\bmiphid}{\phi}
\bmdefine{\bmivarphid}{\varphi}
\bmdefine{\bmipid}{\pi}
\bmdefine{\bmivarpid}{\varpi}
\bmdefine{\bmipsid}{\psi}
\bmdefine{\bmirhod}{\rho}
\bmdefine{\bmivarrhod}{\varrho}
\bmdefine{\bmisigmad}{\sigma}
\bmdefine{\bmivarsigmad}{\varsigma}
\bmdefine{\bmitaud}{\tau}
\bmdefine{\bmithetad}{\theta}
\bmdefine{\bmivarthetad}{\vartheta}
\bmdefine{\bmiupsilond}{\upsilon}
\bmdefine{\bmixid}{\xi}
\bmdefine{\bmizetad}{\zeta}

\safemath{\bmia}{\bmiad}
\safemath{\bmib}{\bmibd}
\safemath{\bmic}{\bmicd}
\safemath{\bmid}{\bmidd}
\safemath{\bmie}{\bmied}
\safemath{\bmif}{\bmifd}
\safemath{\bmig}{\bmigd}
\safemath{\bmih}{\bmihd}
\safemath{\bmii}{\bmiid}
\safemath{\bmij}{\bmijd}
\safemath{\bmik}{\bmikd}
\safemath{\bmil}{\bmild}
\safemath{\bmim}{\bmimd}
\safemath{\bmin}{\bmind}
\safemath{\bmio}{\bmiod}
\safemath{\bmip}{\bmipd}
\safemath{\bmiq}{\bmiqd}
\safemath{\bmir}{\bmird}
\safemath{\bmis}{\bmisd}
\safemath{\bmit}{\bmitd}
\safemath{\bmiu}{\bmiud}
\safemath{\bmiv}{\bmivd}
\safemath{\bmiw}{\bmiwd}
\safemath{\bmix}{\bmixd}
\safemath{\bmiy}{\bmiyd}
\safemath{\bmiz}{\bmizd}

\safemath{\bmialpha}{\bmialphad}
\safemath{\bmibeta}{\bmibetad}
\safemath{\bmichi}{\bmichid}
\safemath{\bmidelta}{\bmideltad}
\safemath{\bmiepsilon}{\bmiepsilond}
\safemath{\bmivarepsilon}{\bmivarepsilond}
\safemath{\bmieta}{\bmietad}
\safemath{\bmigamma}{\bmigammad}
\safemath{\bmiiota}{\bmiiotad}
\safemath{\bmikappa}{\bmikappad}
\safemath{\bmivarkappa}{\bmivarkappad}
\safemath{\bmilambda}{\bmilambdad}
\safemath{\bmimu}{\bmimud}
\safemath{\bminu}{\bminud}
\safemath{\bmiomega}{\bmiomegad}
\safemath{\bmiphi}{\bmiphid}
\safemath{\bmivarphi}{\bmivarphid}
\safemath{\bmipi}{\bmipid}
\safemath{\bmivarpi}{\bmivarpid}
\safemath{\bmipsi}{\bmipsid}
\safemath{\bmirho}{\bmirhod}
\safemath{\bmivarrho}{\bmivarrhod}
\safemath{\bmisigma}{\bmisigmad}
\safemath{\bmivarsigma}{\bmivarsigmad}
\safemath{\bmitau}{\bmitaud}
\safemath{\bmitheta}{\bmithetad}
\safemath{\bmivartheta}{\bmivarthetad}
\safemath{\bmiupsilon}{\bmiupsilond}
\safemath{\bmixi}{\bmixid}
\safemath{\bmizeta}{\bmizetad}

\bmdefine{\bmuDeltad}{\Updelta}
\bmdefine{\bmuGammad}{\Upgamma}
\bmdefine{\bmuLambdad}{\Uplambda}
\bmdefine{\bmuOmegad}{\Upomega}
\bmdefine{\bmuPhid}{\Upphi}
\bmdefine{\bmuPid}{\Uppi}
\bmdefine{\bmuPsid}{\Uppsi}
\bmdefine{\bmuSigmad}{\Upsigma}
\bmdefine{\bmuThetad}{\Uptheta}
\bmdefine{\bmuUpsilond}{\Upupsilon}
\bmdefine{\bmuXid}{\Upxi}

\safemath{\bmuA}{\mathbf{A}}
\safemath{\bmuB}{\mathbf{B}}
\safemath{\bmuC}{\mathbf{C}}
\safemath{\bmuD}{\mathbf{D}}
\safemath{\bmuE}{\mathbf{E}}
\safemath{\bmuF}{\mathbf{F}}
\safemath{\bmuG}{\mathbf{G}}
\safemath{\bmuH}{\mathbf{H}}
\safemath{\bmuI}{\mathbf{I}}
\safemath{\bmuJ}{\mathbf{J}}
\safemath{\bmuK}{\mathbf{K}}
\safemath{\bmuL}{\mathbf{L}}
\safemath{\bmuM}{\mathbf{M}}
\safemath{\bmuN}{\mathbf{N}}
\safemath{\bmuO}{\mathbf{O}}
\safemath{\bmuP}{\mathbf{P}}
\safemath{\bmuQ}{\mathbf{Q}}
\safemath{\bmuR}{\mathbf{R}}
\safemath{\bmuS}{\mathbf{S}}
\safemath{\bmuT}{\mathbf{T}}
\safemath{\bmuU}{\mathbf{U}}
\safemath{\bmuV}{\mathbf{V}}
\safemath{\bmuW}{\mathbf{W}}
\safemath{\bmuX}{\mathbf{X}}
\safemath{\bmuY}{\mathbf{Y}}
\safemath{\bmuZ}{\mathbf{Z}}

\safemath{\bmuZero}{\mathbf{0}}
\safemath{\bmuOne}{\mathbf{1}}

\safemath{\bmuDelta}{\bmuDeltad}
\safemath{\bmuGamma}{\bmuGammad}
\safemath{\bmuLambda}{\bmuLambdad}
\safemath{\bmuOmega}{\bmuOmegad}
\safemath{\bmuPhi}{\bmuPhid}
\safemath{\bmuPi}{\bmuPid}
\safemath{\bmuPsi}{\bmuPsid}
\safemath{\bmuSigma}{\bmuSigmad}
\safemath{\bmuTheta}{\bmuThetad}
\safemath{\bmuUpsilon}{\bmuUpsilond}
\safemath{\bmuXi}{\bmuXid}

\bmdefine{\bmiAd}{A}
\bmdefine{\bmiBd}{B}
\bmdefine{\bmiCd}{C}
\bmdefine{\bmiDd}{D}
\bmdefine{\bmiEd}{E}
\bmdefine{\bmiFd}{F}
\bmdefine{\bmiGd}{G}
\bmdefine{\bmiHd}{H}
\bmdefine{\bmiId}{I}
\bmdefine{\bmiJd}{J}
\bmdefine{\bmiKd}{K}
\bmdefine{\bmiLd}{L}
\bmdefine{\bmiMd}{M}
\bmdefine{\bmiOd}{N}
\bmdefine{\bmiPd}{O}
\bmdefine{\bmiQd}{P}
\bmdefine{\bmiRd}{R}
\bmdefine{\bmiSd}{S}
\bmdefine{\bmiTd}{T}
\bmdefine{\bmiUd}{U}
\bmdefine{\bmiVd}{V}
\bmdefine{\bmiWd}{W}
\bmdefine{\bmiXd}{X}
\bmdefine{\bmiYd}{Y}
\bmdefine{\bmiZd}{Z}

\bmdefine{\bmiDeltad}{\Delta}
\bmdefine{\bmiGammad}{\Gamma}
\bmdefine{\bmiLambdad}{\Lambda}
\bmdefine{\bmiOmegad}{\Omega}
\bmdefine{\bmiPhid}{\Phi}
\bmdefine{\bmiPid}{\Pi}
\bmdefine{\bmiPsid}{\Psi}
\bmdefine{\bmiSigmad}{\Sigma}
\bmdefine{\bmiThetad}{\Theta}
\bmdefine{\bmiUpsilond}{\Upsilon}
\bmdefine{\bmiXid}{\Xi}

\safemath{\bmiA}{\bmiAd}
\safemath{\bmiB}{\bmiBd}
\safemath{\bmiC}{\bmiCd}
\safemath{\bmiD}{\bmiDd}
\safemath{\bmiE}{\bmiEd}
\safemath{\bmiF}{\bmiFd}
\safemath{\bmiG}{\bmiGd}
\safemath{\bmiH}{\bmiHd}
\safemath{\bmiI}{\bmiId}
\safemath{\bmiJ}{\bmiJd}
\safemath{\bmiK}{\bmiKd}
\safemath{\bmiL}{\bmiLd}
\safemath{\bmiM}{\bmiMd}
\safemath{\bmiN}{\bmiNd}
\safemath{\bmiO}{\bmiOd}
\safemath{\bmiP}{\bmiPd}
\safemath{\bmiQ}{\bmiQd}
\safemath{\bmiR}{\bmiRd}
\safemath{\bmiS}{\bmiSd}
\safemath{\bmiT}{\bmiTd}
\safemath{\bmiU}{\bmiUd}
\safemath{\bmiV}{\bmiVd}
\safemath{\bmiW}{\bmiWd}
\safemath{\bmiX}{\bmiXd}
\safemath{\bmiY}{\bmiYd}
\safemath{\bmiZ}{\bmiZd}

\safemath{\bmiDelta}{\bmiDeltad}
\safemath{\bmiGamma}{\bmiGammad}
\safemath{\bmiLambda}{\bmiLambdad}
\safemath{\bmiOmega}{\bmiOmegad}
\safemath{\bmiPhi}{\bmiPhid}
\safemath{\bmiPi}{\bmiPid}
\safemath{\bmiPsi}{\bmiPsid}
\safemath{\bmiSigma}{\bmiSigmad}
\safemath{\bmiTheta}{\bmiThetad}
\safemath{\bmiUpsilon}{\bmiUpsilond}
\safemath{\bmiXi}{\bmiXid}

\safemath{\evA}{\mathcal{A}}
\safemath{\evB}{\mathcal{B}}
\safemath{\evC}{\mathcal{C}}
\safemath{\evD}{\mathcal{D}}
\safemath{\evE}{\mathcal{E}}
\safemath{\evF}{\mathcal{F}}
\safemath{\evG}{\mathcal{G}}
\safemath{\evH}{\mathcal{H}}
\safemath{\evI}{\mathcal{I}}
\safemath{\evJ}{\mathcal{J}}
\safemath{\evK}{\mathcal{K}}
\safemath{\evL}{\mathcal{L}}
\safemath{\evM}{\mathcal{M}}
\safemath{\evN}{\mathcal{N}}
\safemath{\evO}{\mathcal{O}}
\safemath{\evP}{\mathcal{P}}
\safemath{\evQ}{\mathcal{Q}}
\safemath{\evR}{\mathcal{R}}
\safemath{\evS}{\mathcal{S}}
\safemath{\evT}{\mathcal{T}}
\safemath{\evU}{\mathcal{U}}
\safemath{\evV}{\mathcal{V}}
\safemath{\evW}{\mathcal{W}}
\safemath{\evX}{\mathcal{X}}
\safemath{\evY}{\mathcal{Y}}
\safemath{\evZ}{\mathcal{Z}}

\safemath{\setA}{\mathcal{A}}
\safemath{\setB}{\mathcal{B}}
\safemath{\setC}{\mathcal{C}}
\safemath{\setD}{\mathcal{D}}
\safemath{\setE}{\mathcal{E}}
\safemath{\setF}{\mathcal{F}}
\safemath{\setG}{\mathcal{G}}
\safemath{\setH}{\mathcal{H}}
\safemath{\setI}{\mathcal{I}}
\safemath{\setJ}{\mathcal{J}}
\safemath{\setK}{\mathcal{K}}
\safemath{\setL}{\mathcal{L}}
\safemath{\setM}{\mathcal{M}}
\safemath{\setN}{\mathcal{N}}
\safemath{\setO}{\mathcal{O}}
\safemath{\setP}{\mathcal{P}}
\safemath{\setQ}{\mathcal{Q}}
\safemath{\setR}{\mathcal{R}}
\safemath{\setS}{\mathcal{S}}
\safemath{\setT}{\mathcal{T}}
\safemath{\setU}{\mathcal{U}}
\safemath{\setV}{\mathcal{V}}
\safemath{\setW}{\mathcal{W}}
\safemath{\setX}{\mathcal{X}}
\safemath{\setY}{\mathcal{Y}}
\safemath{\setZ}{\mathcal{Z}}
\safemath{\emptySet}{\varnothing}

\safemath{\colA}{\mathscr{A}}
\safemath{\colB}{\mathscr{B}}
\safemath{\colC}{\mathscr{C}}
\safemath{\colD}{\mathscr{D}}
\safemath{\colE}{\mathscr{E}}
\safemath{\colF}{\mathscr{F}}
\safemath{\colG}{\mathscr{G}}
\safemath{\colH}{\mathscr{H}}
\safemath{\colI}{\mathscr{I}}
\safemath{\colJ}{\mathscr{J}}
\safemath{\colK}{\mathscr{K}}
\safemath{\colL}{\mathscr{L}}
\safemath{\colM}{\mathscr{M}}
\safemath{\colN}{\mathscr{N}}
\safemath{\colO}{\mathscr{O}}
\safemath{\colP}{\mathscr{P}}
\safemath{\colQ}{\mathscr{Q}}
\safemath{\colR}{\mathscr{R}}
\safemath{\colS}{\mathscr{S}}
\safemath{\colT}{\mathscr{T}}
\safemath{\colU}{\mathscr{U}}
\safemath{\colV}{\mathscr{V}}
\safemath{\colW}{\mathscr{W}}
\safemath{\colX}{\mathscr{X}}
\safemath{\colY}{\mathscr{Y}}
\safemath{\colZ}{\mathscr{Z}}

\safemath{\opA}{\operatorname{A}}
\safemath{\opB}{\operatorname{B}}
\safemath{\opC}{\operatorname{C}}
\safemath{\opD}{\operatorname{D}}
\safemath{\opE}{\operatorname{E}}
\safemath{\opF}{\operatorname{F}}
\safemath{\opG}{\operatorname{G}}
\safemath{\opH}{\operatorname{H}}
\safemath{\opI}{\operatorname{I}}
\safemath{\opJ}{\operatorname{J}}
\safemath{\opK}{\operatorname{K}}
\safemath{\opL}{\operatorname{L}}
\safemath{\opM}{\operatorname{M}}
\safemath{\opN}{\operatorname{N}}
\safemath{\opO}{\operatorname{O}}
\safemath{\opP}{\operatorname{P}}
\safemath{\opQ}{\operatorname{Q}}
\safemath{\opR}{\operatorname{R}}
\safemath{\opS}{\operatorname{S}}
\safemath{\opT}{\operatorname{T}}
\safemath{\opU}{\operatorname{U}}
\safemath{\opV}{\operatorname{V}}
\safemath{\opW}{\operatorname{W}}
\safemath{\opX}{\operatorname{X}}
\safemath{\opY}{\operatorname{Y}}
\safemath{\opZ}{\operatorname{Z}}
\safemath{\opZero}{\operatorname{O}}
\safemath{\identityop}{\opI}

\safemath{\sca}{a}
\safemath{\scb}{b}
\safemath{\scc}{c}
\safemath{\scd}{d}
\safemath{\sce}{e}
\safemath{\scf}{f}
\safemath{\scg}{g}
\safemath{\sch}{h}
\safemath{\sci}{i}
\safemath{\scj}{j}
\safemath{\sck}{k}
\safemath{\scl}{l}
\safemath{\scm}{m}
\safemath{\scn}{n}
\safemath{\sco}{o}
\safemath{\scp}{p}
\safemath{\scq}{q}
\safemath{\scr}{r}
\safemath{\scs}{s}
\safemath{\sct}{t}
\safemath{\scu}{u}
\safemath{\scv}{v}
\safemath{\scw}{w}
\safemath{\scx}{x}
\safemath{\scy}{y}
\safemath{\scz}{z}

\safemath{\scA}{A}
\safemath{\scB}{B}
\safemath{\scC}{C}
\safemath{\scD}{D}
\safemath{\scE}{E}
\safemath{\scF}{F}
\safemath{\scG}{G}
\safemath{\scH}{H}
\safemath{\scI}{I}
\safemath{\scJ}{J}
\safemath{\scK}{K}
\safemath{\scL}{L}
\safemath{\scM}{M}
\safemath{\scN}{N}
\safemath{\scO}{O}
\safemath{\scP}{P}
\safemath{\scQ}{Q}
\safemath{\scR}{R}
\safemath{\scS}{S}
\safemath{\scT}{T}
\safemath{\scU}{U}
\safemath{\scV}{V}
\safemath{\scW}{W}
\safemath{\scX}{X}
\safemath{\scY}{Y}
\safemath{\scZ}{Z}

\safemath{\scalpha}{\alpha}
\safemath{\scbeta}{\beta}
\safemath{\scchi}{\chi}
\safemath{\scdelta}{\delta}
\safemath{\scepsilon}{\epsilon}
\safemath{\scvarepsilon}{\varepsilon}
\safemath{\sceta}{\eta}
\safemath{\scgamma}{\gamma}
\safemath{\sciota}{\iota}
\safemath{\sckappa}{\kappa}
\safemath{\scvarkappa}{\varkappa}
\safemath{\sclambda}{\lambda}
\safemath{\scmu}{\mu}
\safemath{\scnu}{\nu}
\safemath{\scomega}{\omega}
\safemath{\scphi}{\phi}
\safemath{\scvarphi}{\varphi}
\safemath{\scpi}{\pi}
\safemath{\scvarpi}{\varpi}
\safemath{\scpsi}{\psi}
\safemath{\scrho}{\rho}
\safemath{\scvarrho}{\varrho}
\safemath{\scsigma}{\sigma}
\safemath{\scvarsigma}{\varsigma}
\safemath{\sctau}{\tau}
\safemath{\sctheta}{\theta}
\safemath{\scvartheta}{\vartheta}
\safemath{\scupsilon}{\upsilon}
\safemath{\scxi}{\xi}
\safemath{\sczeta}{\zeta}

\safemath{\veca}{{\boldsymbol{a}}}
\safemath{\vecb}{{\boldsymbol{b}}}
\safemath{\vecc}{{\boldsymbol{c}}}
\safemath{\vecd}{{\boldsymbol{d}}}
\safemath{\vece}{{\boldsymbol{e}}}
\safemath{\vecf}{{\boldsymbol{f}}}
\safemath{\vecg}{{\boldsymbol{g}}}
\safemath{\vech}{{\boldsymbol{h}}}
\safemath{\veci}{{\boldsymbol{i}}}
\safemath{\vecj}{{\boldsymbol{j}}}
\safemath{\veck}{{\boldsymbol{k}}}
\safemath{\vecl}{{\boldsymbol{l}}}
\safemath{\vecm}{{\boldsymbol{m}}}
\safemath{\vecn}{{\boldsymbol{n}}}
\safemath{\veco}{{\boldsymbol{o}}}
\safemath{\vecp}{{\boldsymbol{p}}}
\safemath{\vecq}{{\boldsymbol{q}}}
\safemath{\vecr}{{\boldsymbol{r}}}
\safemath{\vecs}{{\boldsymbol{s}}}
\safemath{\vect}{{\boldsymbol{t}}}
\safemath{\vecu}{{\boldsymbol{u}}}
\safemath{\vecv}{{\boldsymbol{v}}}
\safemath{\vecw}{{\boldsymbol{w}}}
\safemath{\vecx}{{\boldsymbol{x}}}
\safemath{\vecy}{{\boldsymbol{y}}}
\safemath{\vecz}{{\boldsymbol{z}}}
\safemath{\veczero}{{\boldsymbol{0}}}
\safemath{\vecone}{{\boldsymbol{1}}}

\safemath{\vecalpha}{\upalpha}
\safemath{\vecbeta}{\upbeta}
\safemath{\vecchi}{\upchi}
\safemath{\vecdelta}{\updelta}
\safemath{\vecepsilon}{\upepsilon}
\safemath{\vecvarepsilon}{\upvarepsilon}
\safemath{\veceta}{\upeta}
\safemath{\vecgamma}{\upgamma}
\safemath{\veciota}{\upiota}
\safemath{\veckappa}{\upkappa}
\safemath{\veclambda}{\uplambda}
\safemath{\vecmu}{\text{\textmu}}
\safemath{\vecnu}{\upnu}
\safemath{\vecomega}{\upomega}
\safemath{\vecphi}{\upphi}
\safemath{\vecvarphi}{\upvarphi}
\safemath{\vecpi}{\uppi}
\safemath{\vecvarpi}{\upvarpi}
\safemath{\vecpsi}{\uppsi}
\safemath{\vecrho}{\uprho}
\safemath{\vecvarrho}{\upvarrho}
\safemath{\vecsigma}{\upsigma}
\safemath{\vecvarsigma}{\upvarsigma}
\safemath{\vectau}{\uptau}
\safemath{\vectheta}{\uptheta}
\safemath{\vecvartheta}{\upvartheta}
\safemath{\vecupsilon}{\upupsilon}
\safemath{\vecxi}{\upxi}
\safemath{\veczeta}{\upzeta}

\safemath{\vecac}{a}
\safemath{\vecbc}{b}
\safemath{\veccc}{c}
\safemath{\vecdc}{d}
\safemath{\vecec}{e}
\safemath{\vecfc}{f}
\safemath{\vecgc}{g}
\safemath{\vechc}{h}
\safemath{\vecic}{i}
\safemath{\vecjc}{j}
\safemath{\veckc}{k}
\safemath{\veclc}{l}
\safemath{\vecmc}{m}
\safemath{\vecnc}{n}
\safemath{\vecoc}{o}
\safemath{\vecpc}{p}
\safemath{\vecqc}{q}
\safemath{\vecrc}{r}
\safemath{\vecsc}{s}
\safemath{\vectc}{t}
\safemath{\vecuc}{u}
\safemath{\vecvc}{v}
\safemath{\vecwc}{w}
\safemath{\vecxc}{x}
\safemath{\vecyc}{y}
\safemath{\veczc}{z}

\safemath{\matA}{{\boldsymbol{A}}}
\safemath{\matB}{{\boldsymbol{B}}}
\safemath{\matC}{{\boldsymbol{C}}}
\safemath{\matD}{{\boldsymbol{D}}}
\safemath{\matE}{{\boldsymbol{E}}}
\safemath{\matF}{{\boldsymbol{F}}}
\safemath{\matG}{{\boldsymbol{G}}}
\safemath{\matH}{{\boldsymbol{H}}}
\safemath{\matI}{{\boldsymbol{I}}}
\safemath{\matJ}{{\boldsymbol{J}}}
\safemath{\matK}{{\boldsymbol{K}}}
\safemath{\matL}{{\boldsymbol{L}}}
\safemath{\matM}{{\boldsymbol{M}}}
\safemath{\matN}{{\boldsymbol{N}}}
\safemath{\matO}{{\boldsymbol{O}}}
\safemath{\matP}{{\boldsymbol{P}}}
\safemath{\matQ}{{\boldsymbol{Q}}}
\safemath{\matR}{{\boldsymbol{R}}}
\safemath{\matS}{{\boldsymbol{S}}}
\safemath{\matT}{{\boldsymbol{T}}}
\safemath{\matU}{{\boldsymbol{U}}}
\safemath{\matV}{{\boldsymbol{V}}}
\safemath{\matW}{{\boldsymbol{W}}}
\safemath{\matX}{{\boldsymbol{X}}}
\safemath{\matY}{{\boldsymbol{Y}}}
\safemath{\matZ}{{\boldsymbol{Z}}}
\safemath{\matzero}{{\boldsymbol{0}}}

\safemath{\matDelta}{\Updelta}
\safemath{\matGamma}{\Upgammma}
\safemath{\matLambda}{\Uplambda}
\safemath{\matOmega}{\Upomega}
\safemath{\matPhi}{\Upphi}
\safemath{\matPi}{\Uppi}
\safemath{\matPsi}{\Uppsi}
\safemath{\matSigma}{\Upsigma}
\safemath{\matTheta}{\Uptheta}
\safemath{\matUpsilon}{\Upupsilon}
\safemath{\matXi}{\Upxi}

\safemath{\matidentity}{\matI}
\safemath{\vecunit}{\vece} 
\safemath{\matone}{\matO}

\safemath{\matAc}{a}
\safemath{\matBc}{b}
\safemath{\matCc}{c}
\safemath{\matDc}{d}
\safemath{\matEc}{e}
\safemath{\matFc}{f}
\safemath{\matGc}{g}
\safemath{\matHc}{h}
\safemath{\matIc}{i}
\safemath{\matJc}{j}
\safemath{\matKc}{k}
\safemath{\matLc}{l}
\safemath{\matMc}{m}
\safemath{\matNc}{n}
\safemath{\matOc}{o}
\safemath{\matPc}{p}
\safemath{\matQc}{q}
\safemath{\matRc}{r}
\safemath{\matSc}{s}
\safemath{\matTc}{t}
\safemath{\matUc}{u}
\safemath{\matVc}{v}
\safemath{\matWc}{w}
\safemath{\matXc}{x}
\safemath{\matYc}{y}
\safemath{\matZc}{z}

\safemath{\rnda}{\mathsf{a}}
\safemath{\rndb}{\mathsf{b}}
\safemath{\rndc}{\mathsf{c}}
\safemath{\rndd}{\mathsf{d}}
\safemath{\rnde}{\mathsf{e}}
\safemath{\rndf}{\mathsf{f}}
\safemath{\rndg}{\mathsf{g}}
\safemath{\rndh}{\mathsf{h}}
\safemath{\rndi}{\mathsf{i}}
\safemath{\rndj}{\mathsf{j}}
\safemath{\rndk}{\mathsf{k}}
\safemath{\rndl}{\mathsf{l}}
\safemath{\rndm}{\mathsf{m}}
\safemath{\rndn}{\mathsf{n}}
\safemath{\rndo}{\mathsf{o}}
\safemath{\rndp}{\mathsf{p}}
\safemath{\rndq}{\mathsf{q}}
\safemath{\rndr}{\mathsf{r}}
\safemath{\rnds}{\mathsf{s}}
\safemath{\rndt}{\mathsf{t}}
\safemath{\rndu}{\mathsf{u}}
\safemath{\rndv}{\mathsf{v}}
\safemath{\rndw}{\mathsf{w}}
\safemath{\rndx}{\mathsf{x}}
\safemath{\rndy}{\mathsf{y}}
\safemath{\rndz}{\mathsf{z}}

\safemath{\rndA}{\bmiA}
\safemath{\rndB}{\bmiB}
\safemath{\rndC}{\bmiC}
\safemath{\rndD}{\bmiD}
\safemath{\rndE}{\bmiE}
\safemath{\rndF}{\bmiF}
\safemath{\rndG}{\bmiG}
\safemath{\rndH}{\bmiH}
\safemath{\rndI}{\bmiI}
\safemath{\rndJ}{\bmiJ}
\safemath{\rndK}{\bmiK}
\safemath{\rndL}{\bmiL}
\safemath{\rndM}{\bmiM}
\safemath{\rndN}{\bmiN}
\safemath{\rndO}{\bmiO}
\safemath{\rndP}{\bmiP}
\safemath{\rndQ}{\bmiQ}
\safemath{\rndR}{\bmiR}
\safemath{\rndS}{\bmiS}
\safemath{\rndT}{\bmiT}
\safemath{\rndU}{\bmiU}
\safemath{\rndV}{\bmiV}
\safemath{\rndW}{\bmiW}
\safemath{\rndX}{\bmiX}
\safemath{\rndY}{\bmiY}
\safemath{\rndZ}{\bmiZ}

\safemath{\rndalpha}{\bmialpha}
\safemath{\rndbeta}{\bmibeta}
\safemath{\rndchi}{\bmichi}
\safemath{\rnddelta}{\bmidelta}
\safemath{\rndepsilon}{\bmiepsilon}
\safemath{\rndvarepsilon}{\bmivarepsilon}
\safemath{\rndeta}{\bmieta}
\safemath{\rndgamma}{\bmigamma}
\safemath{\rndiota}{\bmiiota}
\safemath{\rndkappa}{\bmikappa}
\safemath{\rndlambda}{\bmilambda}
\safemath{\rndmu}{\bmimu}
\safemath{\rndnu}{\bminu}
\safemath{\rndomega}{\bmiomega}
\safemath{\rndphi}{\bmiphi}
\safemath{\rndvarphi}{\bmivarphi}
\safemath{\rndpi}{\bmipi}
\safemath{\rndvarpi}{\bmivarpi}
\safemath{\rndpsi}{\bmipsi}
\safemath{\rndrho}{\bmirho}
\safemath{\rndvarrho}{\bmivarrho}
\safemath{\rndsigma}{\bmisigma}
\safemath{\rndvarsigma}{\bmivarsigma}
\safemath{\rndtau}{\bmitau}
\safemath{\rndtheta}{\bmitheta}
\safemath{\rndvartheta}{\bmivartheta}
\safemath{\rndupsilon}{\bmiupsilon}
\safemath{\rndxi}{\bmixi}
\safemath{\rndzeta}{\bmizeta}

\safemath{\rveca}{{\boldsymbol{\mathsf{a}}}}
\safemath{\rvecb}{{\boldsymbol{\mathsf{b}}}}
\safemath{\rvecc}{{\boldsymbol{\mathsf{c}}}}
\safemath{\rvecd}{{\boldsymbol{\mathsf{d}}}}
\safemath{\rvece}{{\boldsymbol{\mathsf{e}}}}
\safemath{\rvecf}{{\boldsymbol{\mathsf{f}}}}
\safemath{\rvecg}{{\boldsymbol{\mathsf{g}}}}
\safemath{\rvech}{{\boldsymbol{\mathsf{h}}}}
\safemath{\rveci}{{\boldsymbol{\mathsf{i}}}}
\safemath{\rvecj}{{\boldsymbol{\mathsf{j}}}}
\safemath{\rveck}{{\boldsymbol{\mathsf{k}}}}
\safemath{\rvecl}{{\boldsymbol{\mathsf{l}}}}
\safemath{\rvecm}{{\boldsymbol{\mathsf{m}}}}
\safemath{\rvecn}{{\boldsymbol{\mathsf{n}}}}
\safemath{\rveco}{{\boldsymbol{\mathsf{o}}}}
\safemath{\rvecp}{{\boldsymbol{\mathsf{p}}}}
\safemath{\rvecq}{{\boldsymbol{\mathsf{q}}}}
\safemath{\rvecr}{{\boldsymbol{\mathsf{r}}}}
\safemath{\rvecs}{{\boldsymbol{\mathsf{s}}}}
\safemath{\rvect}{{\boldsymbol{\mathsf{t}}}}
\safemath{\rvecu}{{\boldsymbol{\mathsf{u}}}}
\safemath{\rvecv}{{\boldsymbol{\mathsf{v}}}}
\safemath{\rvecw}{{\boldsymbol{\mathsf{w}}}}
\safemath{\rvecx}{{\boldsymbol{\mathsf{x}}}}
\safemath{\rvecy}{{\boldsymbol{\mathsf{y}}}}
\safemath{\rvecz}{{\boldsymbol{\mathsf{z}}}}

\safemath{\rvecalpha}{\bmualpha}
\safemath{\rvecbeta}{\bmubeta}
\safemath{\rvecchi}{\bmuchi}
\safemath{\rvecdelta}{\bmudelta}
\safemath{\rvecepsilon}{\bmuepsilon}
\safemath{\rvecvarepsilon}{\bmuvarepsilon}
\safemath{\rveceta}{\bmueta}
\safemath{\rvecgamma}{\bmugamma}
\safemath{\rveciota}{\bmuiota}
\safemath{\rveckappa}{\bmukappa}
\safemath{\rveclambda}{\bmulambda}
\safemath{\rvecmu}{\bmumu}
\safemath{\rvecnu}{\bmunu}
\safemath{\rvecomega}{\bmuomega}
\safemath{\rvecphi}{\bmuphi}
\safemath{\rvecvarphi}{\bmuvarphi}
\safemath{\rvecpi}{\bmupi}
\safemath{\rvecvarpi}{\bmuvarpi}
\safemath{\rvecpsi}{\bmupsi}
\safemath{\rvecrho}{\bmurho}
\safemath{\rvecvarrho}{\bmuvarrho}
\safemath{\rvecsigma}{\bmusigma}
\safemath{\rvecvarsigma}{\bmuvarsigma}
\safemath{\rvectau}{\bmutau}
\safemath{\rvectheta}{\bmutheta}
\safemath{\rvecvartheta}{\bmuvartheta}
\safemath{\rvecupsilon}{\bmuupsilon}
\safemath{\rvecxi}{\bmuxi}
\safemath{\rveczeta}{\bmuzeta}

\safemath{\rvecac}{\rnda}
\safemath{\rvecbc}{\rndb}
\safemath{\rveccc}{\rndc}
\safemath{\rvecdc}{\rndd}
\safemath{\rvecec}{\rnde}
\safemath{\rvecfc}{\rndf}
\safemath{\rvecgc}{\rndg}
\safemath{\rvechc}{\rndh}
\safemath{\rvecic}{\rndi}
\safemath{\rvecjc}{\rndj}
\safemath{\rveckc}{\rndk}
\safemath{\rveclc}{\rndl}
\safemath{\rvecmc}{\rndm}
\safemath{\rvecnc}{\rndn}
\safemath{\rvecoc}{\rndo}
\safemath{\rvecpc}{\rndp}
\safemath{\rvecqc}{\rndq}
\safemath{\rvecrc}{\rndr}
\safemath{\rvecsc}{\rnds}
\safemath{\rvectc}{\rndt}
\safemath{\rvecuc}{\rndu}
\safemath{\rvecvc}{\rndv}
\safemath{\rvecwc}{\rndw}
\safemath{\rvecxc}{\rndx}
\safemath{\rvecyc}{\rndy}
\safemath{\rveczc}{\rndz}

\safemath{\rmatA}{{\boldsymbol{\mathsf{A}}}}
\safemath{\rmatB}{{\boldsymbol{\mathsf{B}}}}
\safemath{\rmatC}{{\boldsymbol{\mathsf{C}}}}
\safemath{\rmatD}{{\boldsymbol{\mathsf{D}}}}
\safemath{\rmatE}{{\boldsymbol{\mathsf{E}}}}
\safemath{\rmatF}{{\boldsymbol{\mathsf{F}}}}
\safemath{\rmatG}{{\boldsymbol{\mathsf{G}}}}
\safemath{\rmatH}{{\boldsymbol{\mathsf{H}}}}
\safemath{\rmatI}{{\boldsymbol{\mathsf{I}}}}
\safemath{\rmatJ}{{\boldsymbol{\mathsf{J}}}}
\safemath{\rmatK}{{\boldsymbol{\mathsf{K}}}}
\safemath{\rmatL}{{\boldsymbol{\mathsf{L}}}}
\safemath{\rmatM}{{\boldsymbol{\mathsf{M}}}}
\safemath{\rmatN}{{\boldsymbol{\mathsf{N}}}}
\safemath{\rmatO}{{\boldsymbol{\mathsf{O}}}}
\safemath{\rmatP}{{\boldsymbol{\mathsf{P}}}}
\safemath{\rmatQ}{{\boldsymbol{\mathsf{Q}}}}
\safemath{\rmatR}{{\boldsymbol{\mathsf{R}}}}
\safemath{\rmatS}{{\boldsymbol{\mathsf{S}}}}
\safemath{\rmatT}{{\boldsymbol{\mathsf{T}}}}
\safemath{\rmatU}{{\boldsymbol{\mathsf{U}}}}
\safemath{\rmatV}{{\boldsymbol{\mathsf{V}}}}
\safemath{\rmatW}{{\boldsymbol{\mathsf{W}}}}
\safemath{\rmatX}{{\boldsymbol{\mathsf{X}}}}
\safemath{\rmatY}{{\boldsymbol{\mathsf{Y}}}}
\safemath{\rmatZ}{{\boldsymbol{\mathsf{Z}}}}
\safemath{\rmatDelta}{\bmuDelta}
\safemath{\rmatGamma}{\bmuGamma}
\safemath{\rmatLambda}{\bmuLambda}
\safemath{\rmatOmega}{\bmuOmega}
\safemath{\rmatPhi}{\bmuPhi}
\safemath{\rmatPi}{\bmuPi}
\safemath{\rmatPsi}{\bmuPsi}
\safemath{\rmatSigma}{\bmuSigma}
\safemath{\rmatTheta}{\bmuTheta}
\safemath{\rmatUpsilon}{\bmuUpsilon}
\safemath{\rmatXi}{\bmuXi}

\safemath{\rmatAc}{\rnda}
\safemath{\rmatBc}{\rndb}
\safemath{\rmatCc}{\rndc}
\safemath{\rmatDc}{\rndd}
\safemath{\rmatEc}{\rnde}
\safemath{\rmatFc}{\rndf}
\safemath{\rmatGc}{\rndg}
\safemath{\rmatHc}{\rndh}
\safemath{\rmatIc}{\rndi}
\safemath{\rmatJc}{\rndj}
\safemath{\rmatKc}{\rndk}
\safemath{\rmatLc}{\rndl}
\safemath{\rmatMc}{\rndm}
\safemath{\rmatNc}{\rndn}
\safemath{\rmatOc}{\rndo}
\safemath{\rmatPc}{\rndp}
\safemath{\rmatQc}{\rndq}
\safemath{\rmatRc}{\rndr}
\safemath{\rmatSc}{\rnds}
\safemath{\rmatTc}{\rndt}
\safemath{\rmatUc}{\rndu}
\safemath{\rmatVc}{\rndv}
\safemath{\rmatWc}{\rndw}
\safemath{\rmatXc}{\rndx}
\safemath{\rmatYc}{\rndy}
\safemath{\rmatZc}{\rndz}

\newenvironment{textbmatrix}{	\setlength{\arraycolsep}{2.5pt}%
								\big[\begin{matrix}}{\end{matrix}\big]%
								\raisebox{0.08ex}{\vphantom{M}}}

 \def\be{\begin{equation}}
 \def\ee{\end{equation}}
 \def\btm{\begin{textbmatrix}}
 \def\etm{\end{textbmatrix}}

\def\ba#1\ea{\begin{align}#1\end{align}}

\DeclareMathOperator{\adj}{adj}				

\safemath{\fun}{\scf}						

\safemath{\vrbl}{x}						
\safemath{\altvrbl}{y}						
\safemath{\aaltvrbl}{z}						

\safemath{\vvrbl}{\vecx}						
\safemath{\altvvrbl}{\vecy}						
\safemath{\aaltvvrbl}{\vecz}						

\safemath{\altfun}{\scg}
\safemath{\aaltfun}{\sch}
\safemath{\bel}{\sce}					
\safemath{\altbel}{\sce}					
\safemath{\frel}{g}					
\safemath{\altfrel}{g}					
\safemath{\dfrel}{\tilde{g}}					
\safemath{\altdfrel}{\tilde{g}}					

\safemath{\mat}{\matA}						
\safemath{\matc}{\matAc}						
\safemath{\altmat}{\matB}						
\safemath{\altmatc}{\matBc}						

\safemath{\vectr}{\vecu}						
\safemath{\vectrc}{\vecuc}						
\safemath{\altvectr}{\vecv}						
\safemath{\aaltvectr}{\vect}						
\safemath{\altvectrc}{\vecvc}						
\safemath{\genvar}{u}						
\safemath{\altgenvar}{v}						

\safemath{\rvectr}{\rvecu}						
\safemath{\rvectrc}{\rvecuc}						
\safemath{\raltvectr}{\rvecv}						
\safemath{\raaltvectr}{\rvect}						
\safemath{\raltvectrc}{\rvecvc}						
\safemath{\rgenvar}{\rndu}						
\safemath{\raltgenvar}{\rndv}						

\newcommand{\nullspace}{\setN}	 			

\newcommand{\abs}[1]{\mathchoice{{\left\lvert#1\right\rvert}}{{\bigl\lvert#1\bigr\rvert}}{{\left\lvert#1\right\rvert}}{{\left\lvert#1\right\rvert}}}		

\newcommand{\ind}[1]{\chi_{#1}}				
\newcommand{\conj}[1]{\ensuremath{#1^{*}}} 	
\newcommand{\tp}[1]{\ensuremath{#1^{\mathsf{T}}}} 		

\newcommand{\inv}[1]{\ensuremath{#1^{-1}}} 	

\safemath{\dirac}{\delta}					
\safemath{\diracp}{\dirac(\time)}			
\safemath{\krond}{\dirac}					
\safemath{\indfun}{I}						
\safemath{\stepfun}{u}						

\safemath{\upto}{\uparrow}
\safemath{\downto}{\downarrow}
\safemath{\iu}{\mathrm{i}}							
\safemath{\maj}{\succ}
\newcommand{\dftmat}[1]{\matF_{#1}}			
\safemath{\mdft}{\dftmat{}}					
\safemath{\runity}{\beta}					
\safemath{\eval}{\lambda}					

\safemath{\veval}{\veclambda}				
\safemath{\littleo}{\sco}					

\let\im\undefined
\safemath{\re}{\Re}				
\safemath{\im}{\Im}				

\safemath{\euclidspace}{\complexset}			
\safemath{\confunspace}{\setC}				
\newcommand{\banachseqspace}[1]{l^{#1}}		
\safemath{\hilseqspace}{\banachseqspace{2}}	
\newcommand{\banachfunspace}[1]{\setL^{#1}}	
\safemath{\hilfunspace}{\banachfunspace{2}}	
\safemath{\hilfunspacep}{\hilfunspace(\complexset)}	
\safemath{\schwarzspace}{\setS}				

\newcommand{\hadj}[1]{#1^{\star}}			

\safemath{\SNR}{\rho} 				
\safemath{\SINR}{\text{\sc sinr}} 				
\safemath{\No}{N_0}							
\safemath{\Es}{E_s}							
\safemath{\Eb}{E_b}							
\safemath{\EbNo}{\frac{\Eb}{\No}}
\safemath{\EsNo}{\frac{\Es}{\No}}
\safemath{\NoVar}{\variance}                 

\let\time\undefined
\safemath{\time}{\sct}						
\safemath{\dtime}{\sck}						
\safemath{\delay}{\sctau}					
\safemath{\ddelay}{\scl}						
\safemath{\doppler}{\scnu}					
\safemath{\ddoppler}{\scm}					
\safemath{\freq}{\scf}						
\safemath{\dfreq}{\scn}						
\safemath{\Dtime}{\Delta\time}
\safemath{\Dfreq}{\Delta\freq}
\safemath{\Ddtime}{\dtime}
\safemath{\Ddfreq}{\dfreq}
\safemath{\bandwidth}{\scB}
\safemath{\maxdoppler}{\doppler_{0}}			
\safemath{\maxdelay}{\delay_{0}}				
\safemath{\spread}{\Delta_{\CHop}}			
\DeclareMathOperator{\CHop}{\ensuremath{\opH}} 
\safemath{\kernel}{\rndk_{\CHop}}			
\safemath{\kernelp}{\kernel(\time,\time')}	
\safemath{\tvir}{\rndh_{\CHop}}				
\safemath{\tvirp}{\tvir(\time,\delay)}		
\safemath{\tvirc}{\conj{\rndh}_{\CHop}}		
\safemath{\tvtf}{\rndl_{\CHop}}				
\safemath{\tvtfp}{\tvtf(\time,\freq)}			
\safemath{\tvtfc}{\conj{\rndl}_{\CHop}}		
\safemath{\spf}{\rnds_{\CHop}}				
\safemath{\spfp}{\spf(\doppler,\delay)}		
\safemath{\spfc}{\conj{\rnds}_{\CHop}}		
\safemath{\bff}{\rndb_{\CHop}}				
\safemath{\bffp}{\bff(\doppler,\freq)}		

\safemath{\irc}{\scr_{\rndh}}				
\safemath{\tfc}{\scr_{\rndl}}				
\safemath{\spc}{\scr_{\rnds}}				
\safemath{\bfc}{\scr_{\rndb}}				

\safemath{\scaf}{\scc_{\rnds}}				
\safemath{\scafp}{\scaf(\doppler,\delay)}		
\safemath{\ccf}{\scc_{\rndl}}				
\safemath{\ccfp}{\ccf(\Dtime,\Dfreq)}			
\safemath{\cic}{\scc_{\rndh}}				
\safemath{\cicp}{\cic(\Dtime,\delay)}			

\safemath{\mi}{I}							
\safemath{\capacity}{C}					

\DeclareMathOperator{\Prob}{\opP}		

\safemath{\normal}{\mathcal{N}}			
\safemath{\jpg}{\mathcal{CN}}			
\safemath{\uniform}{\mathcal{U}}				
\safemath{\mchain}{\leftrightarrow}		

\safemath{\dB}{\,\mathrm{dB}}
\safemath{\dBm}{\,\mathrm{dBm}}
\safemath{\Hz}{\,\mathrm{Hz}}
\safemath{\kHz}{\,\mathrm{kHz}}
\safemath{\MHz}{\,\mathrm{MHz}}
\safemath{\GHz}{\,\mathrm{GHz}}
\safemath{\s}{\,\mathrm{s}}
\safemath{\ms}{\,\mathrm{ms}}
\safemath{\mus}{\,\mathrm{\text{\textmu}s}}
\safemath{\ns}{\,\mathrm{ns}}
\safemath{\ps}{\,\mathrm{ps}}
\safemath{\meter}{\,\mathrm{m}}
\safemath{\mm}{\,\mathrm{mm}}
\safemath{\cm}{\,\mathrm{cm}}
\safemath{\m}{\,\mathrm{m}}
\safemath{\W}{\,\mathrm{W}}
\safemath{\mW}{\, \mathrm{mW}}
\safemath{\J}{\,\mathrm{J}}
\safemath{\K}{\,\mathrm{K}}
\safemath{\bit}{\,\mathrm{bit}}
\safemath{\nat}{\,\mathrm{nat}}

\safemath{\define}{\triangleq}					

\safemath{\equivalent}{\sim}
\safemath{\distas}{\sim}					
\safemath{\sdiff}{\Delta}				
\safemath{\setdiff}{\setminus}				

\safemath{\reals}{\mathbb R}
\safemath{\positivereals}{\reals^{+}}
\safemath{\integers}{\mathbb Z}
\safemath{\posint}{\integers^{+}}
\safemath{\naturals}{\mathbb N}
\safemath{\posnaturals}{\naturals^{+}}
\safemath{\complexset}{\mathbb C}
\safemath{\rationals}{\mathbb Q}

\safemath{\iSet}{\setI}
\safemath{\rel}{\bowtie}					
\safemath{\eqrel}{\sim}					
\safemath{\rlord}{\leq}					
\safemath{\slord}{<}						
\safemath{\rpord}{\preceq}				
\safemath{\rrpord}{\succeq}				
\safemath{\spord}{\prec}					
\safemath{\sig}{\sigma}					
\safemath{\metric}{d}					

\safemath{\setfun}{\Phi}					
\safemath{\measure}{\mu}					
\safemath{\altmeasure}{\lambda}					
\newcommand{\outerm}[1]{#1^{\star}}		
\newcommand{\innerm}[1]{#1_{\star}}		
\safemath{\omeasure}{\outerm{\measure}}		
\safemath{\imeasure}{\innerm{\measure}}		
\safemath{\aecol}{\colS^{\star}_{\measure}} 
\safemath{\emeasure}{\bar{\measure}_{0}}	
\safemath{\rmeasure}{\tilde{\measure}}	
\safemath{\bmeasure}{\measure_{0}}		
\safemath{\glength}{\measure_{\altfun}}	
\safemath{\lebmea}{\lambda}				
\safemath{\blebmea}{\lebmea_{0}}			
\safemath{\sfun}{s}						
\safemath{\absintspace}{\colL^{1}}		
\safemath{\sqintspace}{\colL^{2}}		
\safemath{\abssumspace}{l^{1}}		
\safemath{\sqsumspace}{l^{2}}		

\safemath{\sfield}{\setF}				
\safemath{\vectors}{\setV}				
\safemath{\vecspace}{(\vectors,\sfield)}	
\safemath{\linspace}{\setV}				
\safemath{\clinspace}{(\linspace,\sfield)} 
\safemath{\nspace}{\setU}				
\safemath{\metspace}{\setM}				
\safemath{\bspace}{\setB}				
\safemath{\ipspace}{\setG}				
\safemath{\hilspace}{\setH}				
\safemath{\blospace}{\setG}				
\safemath{\lop}{\opT}					
\safemath{\altlop}{\opS}					
\safemath{\nullsp}{\nullspace(\lop)}		
\safemath{\lfun}{l}						
\safemath{\altlfun}{g}					
\newcommand{\dual}[1]{#1^{'}}			
\safemath{\dsum}{\oplus}					
\safemath{\funspace}{\colL}				
\renewcommand{\adj}[1]{#1^{\times}}		
\safemath{\adjlop}{\adj{\lop}}			
\safemath{\hadjlop}{\hadj{\lop}}			
\safemath{\tow}{\xrightarrow{w}}			
\safemath{\tows}{\xrightarrow{w^{*}}}		
\safemath{\cparam}{\lambda}
\safemath{\lopl}{\lop_{\cparam}}		
\safemath{\iop}{\opI}					
\safemath{\resolop}{\opR}				
\safemath{\resolvent}{\resolop_{\cparam}(\lop)}	
\safemath{\reset}{\setQ}
\safemath{\spectrum}{\setS}
\safemath{\resolset}{\reset(\lop)}		
\safemath{\lopspec}{\spectrum(\lop)}		
\safemath{\pspec}{\spectrum_{p}(\lop)}	
\safemath{\cspec}{\spectrum_{c}(\lop)}	
\safemath{\rspec}{\spectrum_{r}(\lop)}	
\safemath{\ev}{\cparam}					
\newcommand{\specrad}[1]{r_{#1}}			
\safemath{\lopsrad}{\specrad{\lop}}		
\safemath{\pop}{\opP}					

\safemath{\specfam}{\colE}				
\safemath{\specop}{\opE_{\cparam}}		
\safemath{\altspecop}{\opE_{\mu}}		
\safemath{\vmulti}{\vecone}				
\safemath{\unitaryop}{\opU}				
\safemath{\sval}{\sigma}					
\safemath{\corrcoef}{\rho}				
\safemath{\sangle}{\theta}				

\let\time\undefined
\safemath{\iset}{\setI}				
\safemath{\shift}{\nu}
\safemath{\scale}{\alpha}
\safemath{\time}{t}
\safemath{\specfreq}{\theta}	
\newcommand{\transopgen}[1]{\opT_{#1}} 
\safemath{\transop}{\transopgen{\delay}}
\newcommand{\modopgen}[1]{\opM_{#1}}	
\safemath{\modop}{\modopgen{\shift}}
\newcommand{\dilopgen}[1]{\opD_{#1}}	
\safemath{\dilop}{\dilopgen{\scale}}
\safemath{\fram}{\setF}				
\safemath{\dfram}{\dual{\fram}}		
\safemath{\ufb}{B}					
\safemath{\lfb}{A}					
\safemath{\sop}{\hadj{\aop}}				
\safemath{\aop}{\opT}			
\safemath{\fop}{\opS}				
\safemath{\daop}{\tilde\opT}			
\safemath{\dsop}{\hadj{\tilde\opT}}				

\safemath{\ifop}{\inv{\fop}}			
\safemath{\rifop}{\fop^{-1/2}}			
\safemath{\transeq}{\setT}			
\safemath{\nfun}{\Phi}				
\safemath{\funvec}{\vecf}			
\safemath{\altfunvec}{\vecg}

\safemath{\samplespace}{\Omega}
\safemath{\probspace}{(\samplespace,\sfield,\Prob)}	
\safemath{\ccoef}{\rho}			

\safemath{\infstate}{\vecpi}				
\safemath{\typset}{\setA_{\epsilon}^{(N)}}	
\safemath{\expequal}{\doteq}				
\safemath{\code}{C}						
\safemath{\dstringset}{\setD^{\star}}		
\safemath{\cwlength}{l}					
\safemath{\elength}{L}					
\safemath{\extension}{C^{\star}}			
\safemath{\approaches}{\rightarrow}		
\safemath{\evnt}{\setA}					
\safemath{\altevnt}{\setB}					
\safemath{\rv}{\rndx}					
\safemath{\altrv}{\rndy}					
\safemath{\complexrv}{\rndu}					
\safemath{\altcrv}{\rndv}				
\safemath{\rvec}{\rvecx}					
\safemath{\altrvec}{\rvecy}				
\safemath{\crvec}{\rvecu}				
\safemath{\altcrvec}{\rvecv}				
\safemath{\variance}{\sigma^{2}}			
\safemath{\map}{T}						
\safemath{\jacobian}{\matJ}					
\safemath{\wvec}{\rvecw}					
\safemath{\wrv}{\rndw}					
\safemath{\orthmat}{\matQ}				
\safemath{\evmat}{\matLambda}			
\safemath{\identity}{\matidentity}		
\safemath{\innovec}{\vecv}				
\safemath{\convas}{\xrightarrow{\text{a.s.}}}	
\safemath{\convr}{\xrightarrow{\text{r}}}	
\safemath{\convp}{\xrightarrow{\text{P}}}	
\safemath{\convd}{\xrightarrow{\text{D}}}	
\safemath{\ltis}{\opL}				
\safemath{\ir}{h}					
\safemath{\tf}{\MakeUppercase{\ir}}	

\newcommand*{\fancyrefparlabelprefix}{par}		
\newcommand*{\fancyrefremlabelprefix}{rem}		
\newcommand*{\fancyrefchalabelprefix}{cha}		
\newcommand*{\fancyrefapplabelprefix}{app}		
\newcommand*{\fancyrefthmlabelprefix}{thm}		
\newcommand*{\fancyreflemlabelprefix}{lem}		
\newcommand*{\fancyrefcorlabelprefix}{cor}		
\newcommand*{\fancyrefdeflabelprefix}{def}		
\newcommand*{\fancyrefproplabelprefix}{prop}		
\frefformat{vario}{\fancyrefparlabelprefix}{Part~#1}
\frefformat{vario}{\fancyrefremlabelprefix}{Remark~#1}
\frefformat{vario}{\fancyrefchalabelprefix}{Chapter~#1}
\frefformat{vario}{\fancyrefseclabelprefix}{Section~#1}

\frefformat{vario}{\fancyrefthmlabelprefix}{Theorem~#1}
\frefformat{vario}{\fancyreflemlabelprefix}{Lemma~#1}
\frefformat{vario}{\fancyrefcorlabelprefix}{Corollary~#1}
\frefformat{vario}{\fancyrefdeflabelprefix}{Definition~#1}
\frefformat{vario}{\fancyreffiglabelprefix}{Figure~#1}
\frefformat{vario}{\fancyrefapplabelprefix}{Appendix~#1}
\frefformat{vario}{\fancyrefeqlabelprefix}{(#1)}
\frefformat{vario}{\fancyrefproplabelprefix}{Property~#1}

\theoremstyle{definition}
\newtheorem{thm}{Theorem}
\newtheorem{lem}{Lemma}
\newtheorem{prp}{Proposition}
\newtheorem{cor}{Corollary}

\newtheorem{exa}{Example}
\newtheorem{dfn}{Definition}

\usepackage{pgfplots}
\usepackage{enumitem}
\usepackage{cleveref} 
\allowdisplaybreaks[3]
\definecolor{tblblue}{rgb}{0.93,0.93,1.0}
\definecolor{tblred}{rgb}{1,0.93,0.93}
\definecolor{darkblue}{rgb}{0,0,0.7} 
\definecolor{darkgreen}{RGB}{20,120,43} 
\definecolor{darkred}{rgb}{0.8,0,0} 
\definecolor{lightblue}{RGB}{101,124,191}
\definecolor{skyblue}{RGB}{135,206,235}
\definecolor{gold}{RGB}{204,168,66}
\definecolor{strongblue}{RGB}{60,146,228}
\definecolor{lightgray}{gray}{0.5}
\definecolor{verylightgray}{RGB}{101,124,191}
\definecolor{mistyrose}{RGB}{238,213,210}
\definecolor{firebrick3}{RGB}{205,38,38}
\crefname{dfn}{definition}{definitions}
\Crefname{dfn}{Definition}{Definitions}
\crefname{cor}{corollary}{corollaries}
\Crefname{cor}{Corollary}{Corollaries}
\crefname{lem}{lemma}{lemmata}
\Crefname{lem}{Lemma}{Lemmata}
\crefname{thm}{theorem}{theorems}
\Crefname{thm}{Theorem}{Theorems}
\crefname{prp}{proposition}{propositions}
\Crefname{prp}{Proposition}{Propositions}
\crefname{equation}{}{}
\crefname{enumi}{}{}
\crefformat{footnote}{#2\footnotemark[#1]#3}
 \title{Generating Rectifiable Measures through ReLU Neural Networks}
 \author{Erwin Riegler, Alex B\"uhler, Yang Pan, and Helmut B\"olcskei}
\begin{document}
\maketitle 
\begin{abstract}
We derive universal approximation results for the class of (countably) $m$-rectifiable measures. Specifically, we prove that  $m$-rectifiable  measures can be approximated as push-forwards 
 of the one-dimensional Lebesgue measure on $[0,1]$ using ReLU neural networks with arbitrarily small approximation error in terms of Wasserstein distance. Moreover, the weights in the networks under consideration are quantized and bounded and the number of ReLU neural networks required to achieve an approximation error of $\varepsilon$ is no larger than $2^{b(\varepsilon)}$ with 
 $b(\varepsilon)=\setO(\varepsilon^{-m}\log^2(\varepsilon))$. 
 This result improves  Lemma IX.4 in Perekrestenko et al. as it shows that  the rate at which $b(\varepsilon)$ tends to infinity as $\varepsilon$ tends to zero equals the rectifiability parameter $m$, which can be much smaller than the ambient dimension. 
 We extend this result to countably $m$-rectifiable  measures and show that  this rate still equals  
 the rectifiability parameter  $m$ provided that, among other technical assumptions, the measure decays  exponentially  on the individual components of the countably $m$-rectifiable support set. 
 \end{abstract}

\section{Introduction}
Generative models \cite{ngjo02} are used to approximate high-dimensional and very complex  data measures by   simplified  model measures, which allow to generate new samples in a simpler fashion. 
These models can be  applied to situations where it is difficult or expensive to gather a large amount of training data and   synthetic data  is generated to expand the dataset.
In the last decade, the focus was  on  deep generative models, where    neural networks are trained  to  generate the model measures \cite{arbo17,gopomixuwaozcobe14,kiwe14}.  
 Approximation results for deep generative models were obtained in  \cite{bate18,legemariar17,peebbo21,yaliwa22}, which are briefly described next.   First results along these lines appear in \cite{legemariar17}, where it was shown that deep neural networks can be used to approximate measures arising from the composition of Barron functions \cite{ba93}. 
  In \cite{bate18},  it was shown that the uniform measure on $[0,1]^n$ can be approximated arbitrarily well as the push-forward of f on  $[0,1]$ under a ReLU neural network, where the error is measured in terms of Wasserstein distance.  This result was generalized in  \cite{peebbo21} to arbitrary measures on  $[0,1]^n$. Concretely, the measure under consideration is approximated by a uniform mixture, which in turn can then be approximated by the push-forward of Lebesgue measure on  $[0,1]$ under a ReLU neural network.  The  number of parameters in the ReLU neural network that suffice to achieve an error no more than $\varepsilon$  is of order 
$\setO(\varepsilon^{-n})$  \cite[Theorem VIII.1]{peebbo21}.     This result was improved  in \cite{yaliwa22}, where it was shown that for arbitrary measures on $\reals^n$, the number of parameters in the ReLU neural network that suffice to achieve an error no more than $\varepsilon$  is of order 
$\setO(\varepsilon^{-s})$, where $s$ can be arbitrarily close to  the upper dimension of the measure \cite[Theorem 2.4]{yaliwa22}, provided that the measure satisfies a certain moment condition.  The proof idea is to construct a covering of an approximate support set of the measure, put mass points in the centers  of the  covering balls, and then connect these mass points by a piecewise linear curve, which can then be realized in terms of a ReLU neural network.

   In the above summary, the results in \cite{peebbo21} are of particular interest for two reasons. First, for a given measure on $[0,1]^n$,  the  ReLU neural network that achieves the desired approximation error is explicit. And second,    the ReLU neural networks have quantized and bounded weights, which in turn leads to a covering result for arbitrary measures on $[0,1]^n$. Specifically, it is shown in \cite[Lemma IX.4]{peebbo21}
that, for general measures on $[0,1]^n$, the number of ReLU neural networks that suffice  to achieve an approximation error of $\varepsilon$  is no larger than $2^{b(\varepsilon)}$ with 
 \begin{align}\label{eq:scalingdmytro}
 \inf\{s\in (0,\infty): \lim_{\varepsilon\to 0}\varepsilon^s b(\varepsilon)<\infty\} =n.
 \end{align} 
 In many applications, the data measures are supported on  low-dimensional objects in Euclidean space \cite{hidare97,lufahe98,letawi94,soze98,lz12}. Although \cite[Lemma IX.4]{peebbo21}  applies to such situations, the number of bits in \cref{eq:scalingdmytro} suffers from the curse of dimensionality 
 as it depends on the ambient dimension $n$ rather than the intrinsic dimension of the low-dimensional objects. 
In this paper, we consider  the broad class of  (countably) $m$-rectifiable  measures that are supported on (countable) unions of Lipschitz images of compact sets in $\reals^m$, which serve as a mathematically rigorous way to describe low-dimensional, non-smooth data structures embedded in high-dimensional spaces.

We show that   $m$-rectifiable  measures can be approximated as push-forwards 
 of the one-dimensional Lebesgue measure on $[0,1]$ using ReLU neural networks with arbitrarily small approximation error in terms of Wasserstein distance. What is more, the weights in the networks under consideration are quantized and bounded and the number of ReLU neural networks that suffice  to achieve an approximation of $\varepsilon$ is no larger than $2^{b(\varepsilon)}$ with 
 \begin{align}\label{eq:scalingdmytro2}
\inf\{s\in (0,\infty): \lim_{\varepsilon\to 0}\varepsilon^s b(\varepsilon)<\infty\} =m.
 \end{align} 
Comparing \cref{eq:scalingdmytro} and \cref{eq:scalingdmytro2}, we see that the ambient dimension $n$ is replaced by the intrinsic dimension of the objects, i.e., the  rectifiability parameter $m$, which can be much smaller than $n$. We also establish that the result in \cref{eq:scalingdmytro2} is metric entropy optimal and construct a matching lower bound. 

We then extend this result to  countably $m$-rectifiable  measures and show that $b(\varepsilon)$ still satisfies \cref{eq:scalingdmytro2} provided that, among other technical assumptions, the measure decays  exponentially on   the individual $m$-rectifiable components of the countably $m$-rectifiable set. 

The main idea of our proof technique for $m$-rectifiable measures is depicted in  \Cref{fig:illustration}. The $m$-rectifiable measure $\nu$ is supported on the $m$-rectifiable set $f(\setA)$ and can be realized as push-forward of a Radon measure $\mu$ that is supported on $\setA$ and the Lipschitz mapping $f$ can be approximated by a ReLU neural network $\Phi$.  We then  approximate  $\mu$ by $\Sigma\#\lambda$, where $\lambda$ denotes Lebesgue measure on $[0,1]$ and 
$\Sigma$ is a ReLU neural network, in a space-filling fashion. 
 The triangle inequality for Wasserstein distance then yields 
 \vspace*{.5cm}
\begin{align}
W_1( \nu, (\Phi\circ \Sigma)\#\lambda)
&\leq W_1( \nu,\Phi\#\mu)+ W_1( \Phi\#\mu, (\Phi\circ \Sigma)\#\lambda)\\
&\leq W_1( \nu,\Phi\#\mu)+ \operatorname{Lip}(\Phi)W_1( \mu, \Sigma\#\lambda).\label{eq:termtobond}
\end{align}
To upper-bound the individual terms in \Cref{eq:termtobond}, it is crucial to control the  Lipschitz constant $\operatorname{Lip}(\Phi)$  
of the ReLU neural network approximating the Lipschitz function $f$. 
We emphasize that we can explicitly construct 
 the ReLU neural network $\Phi$ approximating $f$ and,  if $\mu$ is known explicitly, which is, e.g., the case when  $f$ is injective (see \Cref{lem:density}), we can also explicitly  construct the space-filling ReLU neural network  $\Sigma$.

\begin{figure}[h]
\centering
\begin{tikzpicture}[use Hobby shortcut,closed=true,scale=0.25] 
\node at (7.5,7) {$\lambda$};
\draw[|-|,  thick, blue] (5,5) -- (10,5);
\node at (20,11) {$\reals^m$};
\node at (20,8.8) {$\Sigma\#\lambda\approx\mu$};
\node at (13,7) {$\Sigma$};
\draw [draw=black, thick] (15,0) rectangle (25,10); 
\draw[fill=blue!20,thick,shift={(26 cm,4.5 cm)},scale=1.5] (-4,2) to [closed, curve through = {(-2.6,1.2) (-2.5,1) (-2,0)  (-2.8,-0.2) (-4,-1.4) (-5.5,-0.8) (-4.7,0) (-5.4,1) (-4.2,1.5)}] (-4,2);
\node at (20,6) {$\setA$};
\draw[fill=green!20,thick, shift={(30 cm,9 cm)},rotate=140,scale=1.5] (-4,2) to [closed, curve through = {(-2.6,1.2) (-2.5,1) (-2,0)  (-2.8,-0.2) (-4,-2.8) (-5.5,-0.8) (-4.7,0) (-5.4,1) (-4.2,1.5)}] (-4,2); 
\node at (36,6) {$f(\setA)$}; 
\node at (35,11) {$\reals^n$};
\node at (35,9) {$\Phi\#\mu\approx \nu$};
\node at (27,7) {$\,\,\Phi\approx f$};
\draw [draw=black,  thick] (30,0) rectangle (40,10);
\draw[->, thick,shift={(-2.2 cm,0.7 cm)},rotate=0] (11,6) .. (14,7.5) .. (19,6);
\draw[->, thick,shift={(-2.2 cm,0.7 cm)},rotate=0] (25,6) .. (28,7.5) .. (33,6);

\end{tikzpicture}
\caption{Illustration of proof technique for the $m$-rectifiable measure $\nu$ supported on the $m$-rectifiable set $f(\setA)$ on $\reals^n$.\label{fig:illustration}}
\end{figure}

The paper is organized as follows. 
In \Cref{sec:1}, we introduce ReLU neural networks and state their main properties. 
\Cref{sec:2} is devoted to the approximation of Lipschitz functions by ReLU neural networks. 
In  \Cref{sec:3}, we introduce (countably) rectifiable measures, state their main properties, and derive our main approximation results. 
Finally, \Cref{sec:improve} contains   approximation properties by uniform mixtures and  quantitative statements on the space-filling approximation of uniform mixtures, which improve and simplify the results  in \cite{peebbo21}. 
In the appendices, we present the following material required in the current paper:  Properties of Lipschitz functions in \Cref{sec:LIP}, tools from general measure theory 
in \Cref{sec:GMT}, properties of Wasserstein distance in \Cref{sec:WD},   auxiliary results 
 in \Cref{sec:aux},    and properties of the sawtooth function in \Cref{sec:sawtooth}. 


\subsection{Notation}
The $i$-th unit vector in $\reals^m$ is denoted by $e_i$  and  $I_m$ stands for the $m\times m$ identity matrix. 
$\log$ refers to the logarithm of base $2$ and $[\,\cdot\,]\colon\reals\to\mathbb{Z}$ denotes rounding to the nearest integer. 
Sets  are designated by calligraphic letters,  e.g., $\setA$, with $\abs{\setA}$ denoting cardinality, $\overline{\setA}$ closure, and $\ind{\setA}$ the indicator function on $\setA$. For the sets $\setA,\setB\in\reals^m$, we write $\setA+\setB$ and $\setA-\setB$ for the Minkowski sum and the Minkowski difference, respectively. 
The diameter of the set $\setA\subseteq\reals^d$ is $\diam(\setA)=\sup_{x,y\in\setA}\lVert x-y\rVert_\infty$. 
For  the Lipschitz mapping $f$  and $p\in [1,\infty]$, we write $\operatorname{Lip}^{(p)}(f)$ for the Lipschitz constant with respect to the $p$-norm (see \Cref{dfn:Lip}) and designate  $\operatorname{Lip}(f):=\operatorname{Lip}^{(\infty)}(f)$.  
 $\colL^m$ and $\colH^m$ stand for the $m$-dimensional Lebesgue and Hausdorff measures, respectively.  For the 
 measure  $\mu$ on $\reals^m$ and the set $\setA\subseteq\reals^m$, we write $\mu|_\setA$ for the  restricted measure  on $\setA$ and $\mu_\setA$ for the trace measure on $\setA$.  
$W^{1,\infty}(\setA)$ is the Sobolev space of order one equipped with the infinity norm on the open set $\setA\subseteq \reals^d$ \cite[Definitions 3.1 and 3.2]{adfo03}. 
$W_1(\,\cdot\,,\,\cdot\,)$ denotes the $1$-Wasserstein distance  (see \Cref{dfn:Wasserstein}).  
By a measure $\mu$ we will always mean an   outer measure, i.e., defined on all subsets, and we write $\colM(\mu)$ for the sigma-algebra of $\mu$-measurable sets. For $\setX,\setY$  metric spaces, $f\colon \setX\to \setY$ Borel measurable, and $\mu$ a Borel measure on $\setX$, we write $f\#\mu$ for the push-forward measure on $\setY$ defined according to $(f\#\mu)(\setA)=\mu(f^{-1}(\setA))$ for all $\setA\subseteq\setY$. $H_\varepsilon(\setX,d)$ denotes the Kolmogorov $\varepsilon$-entropy of the metric space $(\setX,d)$.

\section{ReLU Neural Networks}\label{sec:1}
In this section, we provide a brief overview of neural networks with rectified linear unit (ReLU) activation function, following   the expositions in \cite{elpegrbo21,gukupe20}.
\begin{dfn}\label{dfn:ReLUrho}
Fix $m\in\naturals$. The ReLU function $\rho\colon \reals^m\to \reals^m$ is defined as
\begin{align}
\rho(x)=\tp{(\max\{0,x_1\},\max\{0,x_2\},\dots,\max\{0,x_m\})}.
\end{align}
\end{dfn}

\begin{dfn}
Fix  $m,n\in\naturals$.  A ReLU neural network of depth $L$  is a mapping $\Phi\colon \reals^{m}\to \reals^{n}$ given  by 
\begin{align}
\Phi 
=\begin{cases}
W_1&\text{if $L=1$}\\
W_2\circ \rho\circ W_1&\text{if $L=2$}\\
W_L\circ \rho \circ W_{L-1}\circ\rho \dots \rho \circ W_1 & \text{if $L\geq 3$,} 
\end{cases}
\end{align}
 where $N_0=m$, $N_L=n$, and, for  $\ell=1,\dots, L$, $W_\ell\colon \reals^{N_{\ell-1}}\to \reals^{N_{\ell}}$, $W_\ell(x)=A_\ell x +b_\ell$ 
is an affine transformation  with  
  $A_\ell\in \reals^{N_\ell\times N_{\ell-1}}$ and $b_\ell\in \reals^{N_\ell}$. 
  We refer to $\{L,N_1,\dots,N_L\}$ as the architecture of the ReLU neural network. 
We denote by $\setN_{m,n}$ the set of all ReLU neural networks  of input dimension $m$ and output dimension $n$ of arbitrary depth and arbitrary architecture and introduce  the following quantities. 
For every $\Phi\in \setN_{m,n}$, we designate the following quantities: 
\begin{enumerate}[itemsep=1ex,topsep=1ex]
\renewcommand{\theenumi}{(\roman{enumi})}
\renewcommand{\labelenumi}{(\roman{enumi})}
\item depth: $\setL(\Phi)=L$;
\item connectivity: 
\begin{align}
\mathcal{M}(\Phi)=\sum_{\ell=1}^L (\lVert A_\ell \rVert_0 + \lVert b_\ell \rVert_0);
\end{align}
\item width: $\setW(\Phi)=\max\{N_0,\dots,N_L\}$;
\item set of weights: $\setK(\Phi)=\setK_1(\Phi)\cup\setK_2(\Phi)\cup\dots\cup \setK_L(\Phi)$  
with 
\begin{align}
\setK_\ell(\Phi)=
\bigcup_{k_\ell=1}^{N_{\ell}}\Big\{ (b_\ell)_{k_\ell}\cup(A_\ell)_{k_\ell,1}\cup (A_\ell)_{k_\ell,2}\cup\dots\cup (A_\ell)_{k_\ell,N_{\ell-1}}\Big\}\end{align}
for $\ell=1,\dots,L$; 
\item weight magnitude: $\setB(\Phi)=\max\{|w|:w\in \setK\}$. 
\end{enumerate}
\end{dfn}

We need the following two properties of ReLU neural networks. 
\begin{lem}\label{lem:Id}
Set $W_1(x)=\tp{(I_m, -I_m)}x$ for all $x\in\reals^{m}$ and $W_2(x)=(I_m, -I_m)x$ for all $x\in\reals^{2m}$ and 
consider the ReLU neural network $i_m\in \setN_{m,m}$ defined according to $i_m=W_2\circ\rho\circ W_1$.
Then, $i_m(x)=x$ for all $x\in\reals^{m}$. 
\end{lem}
\begin{proof}
Follows from $\rho(z)-\rho(-z)=z$ for all $z\in\reals$.  
\end{proof}

\begin{lem}\label{lem:P}
Fix $\ell,L\in\naturals$. 
For $i=1,\dots,\ell$, let $m_i,n_i\in\naturals$  and suppose that $\Phi_i\in\setN_{m_i,n_i}$ with $\setL(\Phi_i)=L$. Further, set 
$m=m_1+m_2+\dots+m_\ell$ and $n=n_1+n_2+\dots+n_\ell$. Then,
$P(\Phi_1,\Phi_2,\dots,\Phi_\ell)\in \setN_{m,n}$ constructed in  \cref{eq:Pnetwork} satisfies 
\begin{align}
P(\Phi_1,\Phi_2,\dots,\Phi_\ell)(x)=\tp{(\Phi_1(x_1),\Phi_2(x_2),\dots,\Phi_\ell(x_\ell))}
\end{align}
for all $x=\tp{(x_1,x_2,\dots,x_\ell)}\in\reals^m$ and fulfills  the following properties: 
\begin{align}
\setL(P(\Phi_1,\Phi_2,\dots,\Phi_\ell))&=L\\[0.1em]
 \setM(P(\Phi_1,\Phi_2,\dots,\Phi_\ell)) &= \setM(\Phi_1)+ \setM(\Phi_2)+\dots+\setM(\Phi_\ell)\\[0.1em]
 \setW(P(\Phi_1,\Phi_2,\dots,\Phi_\ell))&= \setW(\Phi_1)+ \setW(\Phi_2)+\dots+\setW(\Phi_\ell)\\[0.1em]
 \setK(P(\Phi_1,\Phi_2,\dots,\Phi_\ell)) &=\setK(\Phi_1)\cup\setK(\Phi_2)\cup\dots\cup\setK(\Phi_\ell)\cup\{0\} \\[0.1em]
 \setB(P(\Phi_1,\Phi_2,\dots,\Phi_\ell)) &= \max\{\setB(\Phi_1),\setB(\Phi_2),\dots,\setB(\Phi_\ell)\}.
\end{align}
\end{lem}
\begin{proof}
By assumption, we have 
\begin{align}
\Phi_i 
=\begin{cases}
W^{(i)}_1&\text{if $L=1$}\\
W_2\circ \rho\circ W_1&\text{if $L=2$}\\
W^{(i)}_L\circ \rho \circ W^{(i)}_{L-1}\circ\rho \dots \rho \circ W^{(i)}_1 & \text{if $L\geq 3$,} 
\end{cases}
\end{align}
where $W^{(i)}_j(x)=A^{(i)}_jx +b^{(i)}_j$ for $i=1,\dots,\ell$ and  $j=1,\dots,L$. Now, set 
$A_j= A^{(1)}_j\bigoplus A^{(2)}_j\bigoplus\dots\bigoplus A^{(\ell)}_j$,  
$b_j=\tp{(b^{(1)}_j, b^{(2)}_j,\dots,b^{(\ell)}_j)}$, and 
$W_j(x)=A_jx +b_j$ for $j=1,\dots,L$. It is  then easily seen that 
\begin{align}\label{eq:Pnetwork}
P(\Phi_1,\Phi_2,\dots,\Phi_\ell)
=\begin{cases}
W_1&\text{if $L=1$}\\
W_2\circ \rho\circ W_1&\text{if $L=2$}\\
W_L\circ \rho \circ W_{L-1}\circ\rho \dots \rho \circ W_1 & \text{if $L\geq 3$,} 
\end{cases}
\end{align}
has the desired properties. 
\end{proof}



\section{Approximation of Lipschitz Functions}\label{sec:2}

In this section, we approximate  Lipschitz functions through  ReLU neural networks, which is required to upper-bound the individual terms  in \cref{eq:termtobond} (see second approximation step in \Cref{fig:illustration}). 
We start with the definition of the spike function, introduced in \cite[Appendix F]{kawaiw24}, which constitutes the basic building block for our construction. 

\begin{dfn}\label{dfn:spike}
The spike function $\phi\colon \reals^m\to\reals$ is defined according to 
\begin{align}\label{eq:spike}
\phi(x)=\max\{1+\min\{x_1,\dots,x_m,0 \}- \max\{x_1,\dots,x_m,0 \},0\}. 
\end{align}
\end{dfn}

We have the following result on the  representation of the spike function  by a ReLU neural network. 
\begin{lem}\cite[Lemma 5.1]{pahubo24}\label{lem:NNspike}
Let $\phi$  be the spike function in \Cref{dfn:spike}. Then, there exists  a  $\Phi\in\setN_{m,1}$ with $\Phi(x)=\phi(x)$ for all $x\in\reals^m$ such that 
\begin{align}
 \setL(\Phi)&= \lceil \log(m+1)\rceil+4\label{eq:maxphi1a}\\
 \setM(\Phi) & \leq 60m  -28\\
 \setW(\Phi) &\leq  6m\\
 \setK(\Phi) &\subseteq \{0,1,-1\} \label{eq:maxphi2aA}\\
 \setB(\Phi) &=1.\label{eq:maxphi2a}
 \end{align}
 Moreover, $\Phi$ can be written as  
  \begin{align}\label{eq:phipsi}
  \Phi = W_2\circ \rho\circ\Psi\circ \rho \circ W_1
  \end{align}
  with $W_1(x)= \tp{(I_m, -I_m)}x$, $W_2(x)=x$, and $\Psi\in \setN_{2m,1}$. 
\end{lem} 
Since the explicit form of $\Psi$ in \cref{eq:phipsi} is not needed, we refer to \cite[Equation (55)]{pahubo24} for this expression. 
We will make use of the following properties of the spike function: 
\begin{lem}\label{lem:spike}
The spike function $\phi$ in \cref{eq:spike} satisfies: 
\begin{enumerate}[itemsep=1ex,topsep=1ex]
\renewcommand{\theenumi}{(\roman{enumi})}
\renewcommand{\labelenumi}{(\roman{enumi})}
\item \label{i1:spike}$\phi(x)\in [0,1]$ for all $x\in (-1,1)^m$;
\item \label{i2:spike}$\phi(x)=0$ for all $x\in\reals^m\setminus (-1,1)^m$; 
\item \label{i2a:spike}$\operatorname{Lip}(\phi)\leq 2$;
\item \label{i3:spike}Let 
\begin{align}
\setS=\{ x\in [0,1]^m:  x_1\geq x_2 \geq \dots \geq x_m\}
\end{align} and define $\setM\subseteq\naturals_0^m$ according to 
\begin{align}\label{eq:setM}
\setM=\{0,e_1,e_1+e_2,\dots, e_1+\dots+e_m\}. 
\end{align} 
Then, 
\begin{align}\label{eq:phix0}
\phi(x-n)=0\quad\text{for all $x\in \setS$ and $n\in\mathbb{Z}^m\setminus \setM$.}
\end{align}
 Moreover, we have 
\begin{align}
\phi(x)&=1-x_1\label{eq:phix1}\\
\phi(x-e_1-e_2\dots-e_k) &= x_{k}-x_{k+1}\quad\text{for $k=1,\dots,m-1$}\\
\phi(x-e_1-e_2\dots-e_m)&=x_m\label{eq:phix3}
\end{align}
for all $x\in \setS$. 
\item \label{i4:spike}It holds that  
\begin{align}
\sum_{n\in\mathbb{Z}^m}\phi(x-n) =1\quad\text{for all $x\in\reals^m$.}
\end{align}
\end{enumerate}
\end{lem}
\begin{proof}
 \cref{i1:spike}, \cref{i2:spike}, and \cref{i4:spike} follow from \cite[Lemma 3.7]{pahubo24} and 
 \cref{i3:spike} is by \cite[(22) and (24)--(26)]{pahubo24} and \cref{i2:spike}. It remains to prove \cref{i2a:spike}. To this end, we write 
 \begin{align}
 \phi(x)=f_1\circ  (1+f_2 - f_3)
 \end{align}
 with $f_1(x)=\max\{x,0\}$, $f_2(x_1,\dots,x_d)=\min\{x_1,\dots,x_m,0 \}$, and $f_3(x_1,\dots,x_d)=\max\{x_1,\dots,x_m,0 \}$. 
 Since $\operatorname{Lip}(f_1)=1$  and $\operatorname{Lip}(f_2)=\operatorname{Lip}(f_3)$, we have 
 \begin{align}\label{eq:liphhi1}
 \operatorname{Lip}(\phi)\leq \operatorname{Lip}(f_1)( \operatorname{Lip}(f_2)+\operatorname{Lip}(f_3)) \leq 2\operatorname{Lip}(f_3).  
 \end{align}
 Now, set $g(x)=\max\{x_1,\dots,x_d\}$. It can be  shown by induction that $\operatorname{Lip}(g)\leq 1$ and hence 
 \begin{align}\label{eq:liphhi2}
 \operatorname{Lip}(f_3) = \operatorname{Lip}(f_1\circ g) \leq  \operatorname{Lip}(f_1)\operatorname{Lip}(g) \leq 1.
 \end{align}
 Using \cref{eq:liphhi2} in  \cref{eq:liphhi1} yields $\operatorname{Lip}(\phi)\leq 2$.
\end{proof}

An  immediate consequence of \Cref{lem:NNspike} is that, for every $N\in\naturals$,  the functions $\{ \phi(Nx-n)\}_{n\in\naturals}$ form a partition of unity on  $\reals^m$.  
\begin{cor}\label{cor:spike}
Let $N\in\naturals$. Then, we have 
\begin{align}
\psi_N(x):=\sum_{n\in\mathbb{Z}^m}\phi(Nx-n) =1\quad\text{for all $x\in\reals^m$}.\label{eq:PUI1}
\end{align}
\end{cor}
\begin{proof}
Since $\psi_N(x)=1$ for all $x\in\reals^m$ if and only if $\psi_N(x/N)=1$ for all $x\in\reals^m$, we can assume,  without loss of generality, that $N=1$. 
But $\psi_1(x)=1$ for all $x\in\reals^m$ owing to \Cref{i4:spike} in \Cref{lem:spike}. 
\end{proof}

We are now in a position to state our main neural network approximation result for Lipschitz functions.  

\begin{thm}\label{thm:1}
Let  $m,N\in\naturals$ and set $\setJ=\{0,1,\dots,N\}$.    Consider the  Lipschitz function $f\colon [0,1]^m\to \reals$ and set $\setD_{f}=\{x\in \reals:|x|\leq \lVert  f \rVert_{L_\infty([0,1]^m)}\}$.  
Let  $\delta\in [0,\infty)$ and let  $h\colon \setD_{f} \to \reals$ be an arbitrary function satisfying $\abs{h(x)-x}\leq \delta$ for all  $x\in \setD_{f}$.  Finally,  define $f_N\colon [0,1]^m \to \reals$ according to 
\begin{align}\label{eq:definitionfN}
f_N(x) = \sum_{n\in\setJ^m} h(f({n}/{N})) \phi(Nx-n) 
\end{align}
with  $\phi$ the spike function from \Cref{dfn:spike}. 
Then, the following properties hold: 
\begin{enumerate}[itemsep=1ex,topsep=1ex]
\renewcommand{\theenumi}{(\roman{enumi})}
\renewcommand{\labelenumi}{(\roman{enumi})}
\item $\lVert f-f_N \rVert_{L_\infty ([0,1]^m)}\leq \operatorname{Lip}(f)/N+\delta$; \label{i1:approxf}
\item 
$\operatorname{Lip}(f_N)/m\leq  \operatorname{Lip}^{(1)}(f_N) \leq \operatorname{Lip}^{(1)}(f)+2N\delta \leq  \operatorname{Lip}(f)+2N\delta.$  
 \label{i3:approxf}
\end{enumerate}
Further, 
\begin{enumerate}[itemsep=1ex,topsep=1ex]
\renewcommand{\theenumi}{(\alph{enumi})}
\renewcommand{\labelenumi}{(\alph{enumi})}
\item
let $\Psi$ be as in \Cref{lem:NNspike}, and set, for every $n\in \setJ^m$, 
\begin{align}
\Phi_{n,N} =   \Psi \circ \rho \circ  W_{n,N}
\end{align}
with $W_{n,N}(x)=\tp{N(I_m, -I_m)} x -\tp{(n, -n)} $;  
\item \label{item:phifN} let $\sigma\colon \{1,2,\dots (N+1)^m\}\to \setJ^m $ be a bijection, set  
\begin{align}
F= (h(f(\sigma(1)/N)), \dots, h(f(\sigma((N+1)^m)/N))),
\end{align}
define $ W\colon \reals^{\abs{\setJ^m}}\to \reals$ according to $ W(x)= Fx$,
and define the ReLU neural network $\Phi$ as follows:
\begin{align}\label{eq:dfnPhi}
 \Phi= W \circ \rho\circ  P(\Phi_{\sigma(1),N},\dots,  \Phi_{\sigma((N+1)^m ),N}).  
\end{align} 
\end{enumerate}
Then, we have 
\begin{align}
\Phi(x)=f_N(x)\quad\text{ for all $x\in [0,1]^m$}\label{eq:phifN}
\end{align}
 with 
\begin{align}
 \setL(\Phi)&= \lceil\log(m+1)\rceil+4\label{eq:NN3h}\\
 \setM(\Phi) & \leq (N+1)^m(62m  -28)\\
 \setW(\Phi) &\leq  (N+1)^m6m\label{eq:NN4h}\\
 \setB(\Phi) &\leq \max\{N, \lVert  f \rVert_{L_\infty([0,1]^m)}+\delta\}.\label{eq:NN6h}
 \end{align}
\end{thm}
\begin{proof}
First, define $\hat f_N\colon [0,1]^m \to \reals$ according to 
\begin{align}\label{eq:fN2a}
 \hat f_N(x) = \sum_{n\in\setJ^m}  f({n}/{N}) \phi(Nx-n). 
\end{align}
Next, note that 
\begin{align}
f(x) 
&= \sum_{n\in \mathbb{Z}^m}\phi(Nx-n) f(x) \label{eq:step1f}\\
&= \sum_{n\in \setJ^m}\phi(Nx-n) f(x)\quad\text{for all $x\in [0,1]^m$,}\label{eq:step2f}
\end{align}
where \cref{eq:step1f} follows from \Cref{cor:spike}  and in \cref{eq:step2f} we applied 
\cref{i2:spike} in \Cref{lem:spike}. 
 We therefore have 
\begin{align}
&\abs{f(x)-f_N(x)}\\
&\leq \abs{f(x)-\hat f_N(x)}  + \abs{\hat f_N(x) -f_N(x)} \\
&\leq \sum_{n\in\setJ^m} \Big(\abs{ f(x)- f({n}/{N})} 
+\abs{h( f(n/N))- f({n}/{N})} \Big)\phi(Nx-n)  \label{eq:step1infty}\\
&\leq \Big(\frac{\operatorname{Lip}(f)}{N}+\delta\Big)\sum_{n\in\setJ^m} \phi(Nx-n)\label{eq:step2infty}\\
&\leq \frac{\operatorname{Lip}(f)}{N}+\delta\quad\text{for all $x\in[0,1]^m$}, \label{eq:step3infty}
\end{align}
where 
\cref{eq:step1infty} is by  \cref{eq:step1f}--\cref{eq:step2f},   
 in \cref{eq:step2infty} we applied the Lipschitz property of $ f$ and $\abs{h(x)-x}\leq \delta$ for all $x\in\setD_f$, 
 and   \cref{eq:step3infty} is   again by \Cref{cor:spike}. This establishes \cref{i1:approxf}. 
We next prove \cref{i3:approxf}. The first and the third inequality in \cref{i3:approxf} follow from the norm inequalities $\lVert\, \cdot\,\rVert_1\leq m\lVert\, \cdot\,\rVert_\infty$ and $ \lVert\, \cdot\,\rVert_\infty \leq  \lVert\, \cdot\,\rVert_1$, both  on $\reals^m$. It remains to establish the second inequality in \cref{i3:approxf}. 
To this end, let $\Pi$ denote the set of all permutations $\pi\colon \{1,\dots,m\}\to \{1,\dots,m\}$ and  
define the  simplex $\setS_{\pi}$ according to 
\begin{align}
\setS_{\pi}=\{ x\in [0,1]^m:  x_{\pi(1)}\geq   x_{\pi(2)} \geq \dots \geq x_{\pi(m)}\}.   
\end{align}
Moreover, for every $\pi\in\Pi$, define $\hat \pi\colon\reals^m\to\reals^m$ according to 
\begin{align}
\hat \pi\tp{(x_{1},\dots,x_{m})} = \tp{(x_{\pi(1)},\dots,x_{\pi(m)})}.  
\end{align}
Since the spike function is symmetric, i.e., $\phi(x)=\phi(\hat\pi x)$ for all $\pi\in\Pi$ and $x\in\reals^m$, 
\cref{eq:phix0} implies  
\begin{align}\label{eq:phix0pi}
\phi(x-n)=0\quad\text{for all $x\in \setS_\pi$ and $ n\in\mathbb{Z}^m\setminus\hat \pi^{-1}(\setM)$.}
\end{align}
Let  $\setK=\{0,\dots,N-1\}^m$ and set, for every $k\in\setK$, 
$\setQ_k=[0,1]^m/N+\{k/N\}$. Now, if $x\in \setQ_k$, then there must exist a $\pi\in\Pi$ such that 
$Nx-k\in\setS_\pi$. 
We can therefore cover 
\begin{align}\label{eq:coverRm}
[0,1]^m=\bigcup_{k\in\setK} \bigcup_{\pi \in\Pi} \setR_{k,\pi}, 
\end{align}
where we set  $\setR_{k,\pi}=\{x\in\setQ_k:Nx-k\in\setS_\pi \}$. Now, fix $k\in\setK$ and $\pi\in\Pi$ arbitrarily. 
 Then, we  have  
\begin{align}
 f_N(x) 
&= \sum_{n\in\setJ^m} h(f(({n}/{N})) \phi(Nx-n)\label{eq:step1lip}\\
&= \sum_{n\in\setJ^m-\{k\}} h(f(((n+k)/{N})) \phi(Nx-k -n)\label{eq:step1lipa}\\
&=\sum_{n\in \hat\pi^{-1}(\setM)} h(f((n+k)/{N})) \phi(Nx-k-n)\label{eq:step2lip}\\
&= \sum_{n\in \setM} h(f((\hat\pi^{-1}(n)+k)/{N})) \phi(\hat\pi(Nx-k)-n)\quad\text{for all $x\in \setR_{k,\pi}$},\label{eq:step3lip}
\end{align}
where in \cref{eq:step2lip} follows from 
\cref{eq:phix0pi} upon noting that
\begin{align}
\hat\pi^{-1}(\setM)\subseteq \{0,1\}^m\subseteq \setJ^m-\{k\}
\end{align}
 and in \cref{eq:step3lip} we used 
that the spike function is symmetric.  Now, since $Nx-k\in\setS_\pi$, we have $\hat\pi(Nx-k)\in\setS$.  
Application of \cref{eq:phix1}--\cref{eq:phix3} to \cref{eq:step1lip}--\cref{eq:step3lip} therefore yields 
\begin{align}\label{eq:intermediateaffince}
 f_N(x) &= h(f(k/N))+\sum_{i=1}^m((Nx_{\pi(i)}-k_{\pi(i)} ))( h(f(z_i/N ))- h(f(z_{i-1}/N)))
\end{align}
for all $x\in \setR_{k,\pi}$, where we set $z_0=k$ and $z_i=k+\hat\pi^{-1}(e_1+\dots+e_i)$ for $i=1,\dots,m$. 
We therefore have 
\begin{align}\label{eq:intermediateaffince2}
\abs{f_N(x)-f_N(y)}
&\leq N\sum_{i=1}^m\abs{x_{\pi(i)}-y_{\pi(i)}}\abs{ h(f(z_i/N ))- h(f(z_{i-1}/N))}
\end{align}
for all $x,y\in \setR_{k,\pi}$. Now, 
\begin{align}
&N\abs{ h(f({z_i}/{N})) -  h(f({z_{i-1}}/{N}))}\label{eq:newlip1}\\
&\leq N\abs{ h(f({z_i}/{N})) -    f({z_{i}}/{N})}
+N\abs{   f({z_i}/{N}) -   f({z_{i-1}}/{N})}\\
&\ \ \ +N\abs{   f({z_{i-1}}/{N}) -   h(f({z_{i-1}}/{N}))}\\
&\leq \operatorname{Lip}^{(1)}( f)+2N\delta\quad\text{for $i=1,\dots,m$}\label{eq:newlip2},
\end{align}
where  in \cref{eq:newlip2} we applied the Lipschitz property of $ f$ and the assumption $\abs{h(x)-x}\leq \delta$ for all $x\in\setD_f$. Using \cref{eq:newlip1}--\cref{eq:newlip2} and the arbitrariness of $k$ and $\pi$ in \cref{eq:intermediateaffince2} therefore yields 
\begin{align}
\operatorname{Lip}^{(1)}(f_N|_{\setR_{k,\pi}})
&\leq \operatorname{Lip}^{(1)}( f)+2N\delta\quad\text{for all $k\in\setK$ and $\pi\in\Pi$.} \label{eq:newbound2} 
\end{align}
Since the sets $\setR_{k,\pi}$ are all convex and closed (every $\setR_{k,\pi}$ is the intersection of the two closed and convex sets $\setQ_k$ and $(\setS_\pi+\{k\})/N$ ) and cover $[0,1]^m$ thanks to \cref{eq:coverRm}, \Cref{lem:lipconvex} in combination with \cref{eq:newbound2}  establishes 
\begin{align}\label{eq:lipfinal}
\operatorname{Lip}^{(1)}(f_N) \leq \operatorname{Lip}^{(1)}( f)+2N\delta. 
\end{align}
 Finally, \cref{eq:phifN} follows from the fact that 
\begin{align}
\rho (\Phi_{n,N}(x))=\phi(Nx-n)\quad\text{ for all  $x\in[0,1]^m$ and $n\in\setJ^m$} 
\end{align}
and \cref{eq:NN3h}--\cref{eq:NN6h} are by construction. 
\end{proof}
Some comments are now in order. First, note that  only the last affine transformation $W$ in the ReLU neural network $\Phi$ in \cref{eq:dfnPhi} depends  on $f$.  For the particular choice  $h(x)=x$, i.e., $\delta=0$, \Cref{thm:1} recovers \cite[Theorem 1.1]{hokr24}. 
Next, note that  \cref{i3:approxf} of \Cref{thm:1} in combination with \cref{eq:phifN} only yields a local upper bound on the Lipschitz constant of $\Phi$ on the unit cube. Following the ideas in the proof of \cite[Corollary 5.1]{hokr24}, 
 $\Phi\circ \kappa$ with  $\kappa\colon \reals^m\to[0,1]^m$ defined according to 
\begin{align}
\kappa(x) 
=
\begin{pmatrix}
\rho(x_1)- \rho(x_1-1)\\
\vdots \\
\rho(x_m)- \rho(x_m-1)
\end{pmatrix}, 
\end{align}
yields a ReLU neural network satisfying $\Phi(\kappa(x)=\Phi(x)$ for all $x\in[0,1]^m$ and global Lipschitz constant 
\begin{align}
\operatorname{Lip}^{(1)}(\Phi\circ \kappa)\leq \operatorname{Lip}^{(1)}(\Phi|_{[0,1]^m}). 
\end{align}
Finally, we  emphasize that  $h$ in \Cref{thm:1} need not be Lipschitz. This insight will allow us 
to apply  \Cref{thm:1} in the case where $h$ is a quantization function $y\mapsto [Ny]/N$ with $N\in\naturals$, 
and thereby establish  that  Lipschitz functions can be approximated  by  ReLU neural networks of bounded Lipschitz constants with quantized weights. This  solves the open problem posed  in \cite[Section 6.3]{hokr24} and is formalized as follows.

\begin{cor}\label{thm:LipNN}
Let $m,N\in\naturals$ and set 
\begin{align}
\setF_N=\{k/N:k\in \mathbb{Z}\}\cap [-N, N]. 
\end{align}
There exists  a collection $\colK^{(m)}_N\subseteq \setN_{m,1}$ of ReLU neural networks with 
\begin{align}
\log\Big(\abs{\colK^{(m)}_N}\Big)\leq {(N+1)^m}\log(2N^2+1)
\end{align}
 such that, for every $\setA\subseteq [0,1]^m$  and every  Lipschitz function $f\colon \setA\to \reals$ satisfying  
\begin{align}
\lVert f \rVert_{L_\infty(\setA)} +\operatorname{Lip}(f)\leq N,\label{eq:Ncond}
\end{align}  
there is a 
$\Phi\in \colK^{(m)}_N$ satisfying 
\begin{enumerate}[itemsep=1ex,topsep=1ex]
\renewcommand{\theenumi}{(\roman{enumi})}
\renewcommand{\labelenumi}{(\roman{enumi})}
\item $\lVert f-\Phi\rVert_{L_\infty (\setA)}\leq (\operatorname{Lip}(f)+1/2)/N$; \label{i1:approxfN}
\item $ \operatorname{Lip}(\Phi|_{[0,1]^m})/m\leq  \operatorname{Lip}^{(1)}(\Phi|_{[0,1]^m})\leq \operatorname{Lip}(f)+1$. \label{i3:approxfN}
\end{enumerate}
Moreover, the ReLU neural networks  $\Phi\in \colK^{(m)}_N$ have the same architecture and satisfy: 
\begin{align}
 \setL(\Phi)&= \lceil\log(m+1)\rceil+4\label{eq:NN3}\\
 \setM(\Phi) & \leq (N+1)^m(62m  -28)\\
 \setW(\Phi) &\leq  (N+1)^m6m\label{eq:NN4}\\
 \setK( \Phi) &\subseteq \setF_N \label{eq:NN5}\\
 \setB(\Phi) &\leq N.\label{eq:NN6}
 \end{align}
\end{cor}
\begin{proof}
Arbitrarily fix a set $\setA\subseteq [0,1]^m$ and a Lipschitz function $f\colon \setA\to \reals$ satisfying  \cref{eq:Ncond} and  
set $\setJ=\{0,1,\dots,N\}$, 
Next, define  $\tilde f\colon [0,1]^m\to \reals$ according to  
\begin{align}\label{eq:explicitftilde}
\tilde f(x) = \inf_{y\in\setA}(f(y)+\operatorname{Lip}(f)\lVert x-y\rVert_\infty). 
\end{align}
Then, $\tilde f$ is Lipschitz with $\tilde f(x)=f(x)$ for all $x\in\setA$ and   $\operatorname{Lip}(\tilde f)=\operatorname{Lip}(f)$ owing to \cite[Theorem 1]{mc34}. Moreover, by construction, we have    
$\Vert \tilde f\rVert_{L_\infty([0,1]^m)} \leq \Vert f\rVert_{L_\infty(\setA)}+\operatorname{Lip}(f)$. 
Now,    consider the  quantization function 
$g\colon [-\lVert \tilde f \rVert_{L_\infty([0,1]^m)},\lVert \tilde f \rVert_{L_\infty([0,1]^m)}] \to \setF_N$, $y\mapsto [Ny]/N$  
and note that  
\begin{align}\label{eq:quantf}
\abs{g(y)-y}\leq  1/(2N)\quad\text{for  all $y\in [-\lVert \tilde f \rVert_{L_\infty([0,1]^m)},\lVert \tilde f \rVert_{L_\infty([0,1]^m)}] $.}
\end{align} 
Further, note that $g$ is well-defined since $[N \tilde f(x)]/N\in \mathbb{Z}/N$ and 
\begin{align}
\abs{[N \tilde f(x)]}/N 
&\leq [N  \lVert \tilde f \rVert_{L_\infty([0,1]^m)}]/N \\
&\leq [N  (\lVert  f \rVert_{L_\infty(\setA)} + \operatorname{Lip}(f) )]/N \label{eq:useexplicitftilde}\\
&\leq [N^2]/N\label{eq:welldefuseassamptuion}\\
&= N\quad\text{for all $x\in[0,1]^m$,} 
\end{align}
where in \cref{eq:welldefuseassamptuion} we used assumption \cref{eq:Ncond}. 
Now, let $\Phi$ be the ReLU neural network obtained by applying  \Cref{thm:1} to the Lipschitz function $\tilde f$ with $h=g$ and $\delta=1/(2N)$. Then, 
\begin{align}
\lVert f-\Phi\rVert_{L_\infty (\setA)}
&\leq \lVert \tilde f-\Phi\rVert_{L_\infty ([0,1]^m)}\\
&\leq (\operatorname{Lip}(\tilde f)+1/2)/N \label{eq:usethm1eins}\\
&=(\operatorname{Lip}( f)+1/2)/N,
\end{align}
where \cref{eq:usethm1eins} is by \cref{i1:approxf} in  \Cref{thm:1} and \cref{eq:phifN}, which establishes \cref{i1:approxfN}. By the same token,  \cref{i3:approxf} in \Cref{thm:1} yields  
\begin{align}
 \operatorname{Lip}(\Phi|_{[0,1]^m})/m\leq  \operatorname{Lip}^{(1)}(\Phi|_{[0,1]^m}) \leq \operatorname{Lip}(\tilde f)+1\label{eq:usethm1zwei}=\operatorname{Lip}( f)+1, 
\end{align}
which proves \cref{i3:approxfN}. 
Moreover, \cref{eq:NN3}--\cref{eq:NN4} follow from   \cref{eq:NN3h}--\cref{eq:NN4h} and \cref{eq:NN5}--\cref{eq:NN6} are valid by construction. 
Finally, since $\Phi$ in \cref{eq:dfnPhi} is completely determined by 
\begin{align}
F= (g(\tilde f(\sigma(1)/N)),\dots,   g(\tilde f(\sigma((N+1)^m)/N))) \in \setF_N^{\abs{\setJ^m}}, 
\end{align}
where $\sigma$ is  as defined in 
\cref{item:phifN} of \Cref{thm:1}, 
we can conclude that 
\begin{align}
\log \Big(\abs{\colK^{(m)}_N}\Big) \leq {(N+1)^m}\log(\abs{\setF_N}) = {(N+1)^m}\log (2N^2+1). 
\end{align}
\end{proof}

\section{Rectifiable Measures}\label{sec:3}


In this section, we describe how to generate measures on low-dimensional objects using ReLU neural networks, which is required to upper-bound the second  term in (4) (see second
approximation step in Figure 1).  Specifically, we are interested in the class of measures on $\reals^n$ that are supported on (countable unions of) Lipschitz images of compact sets in $\reals^m$. This class includes all possible measures that are supported (up to a set of Hausdorff measures zero) on arbitrary subsets of countable unions of differentiable manifolds \cite[Lemma 5.4.2]{krpa08}. We start with the definition of rectifiable sets.   

\begin{dfn}\label{dfn:rectset}
Let $m,n\in\naturals$ with $m\leq n$. The set  $\setE\subseteq \reals^n$ is called 
\begin{enumerate}[itemsep=1ex,topsep=1ex]
\renewcommand{\theenumi}{(\roman{enumi})}
\renewcommand{\labelenumi}{(\roman{enumi})}
\item \label{item:rectset1} $m$-rectifiable if there exist a nonempty and compact  set $\setA\subseteq \reals^m$ and a Lipschitz mapping $f\colon\setA\to \reals^n$ such that $\setE=f(\setA)$;
\item countably $m$-rectifiable if there exist $m$-rectifiable sets $\setE_i$, $i\in\naturals$, such that  
\begin{align}
\setE=\bigcup_{i=1}^\infty \setE_i.
\end{align}
\end{enumerate}
\end{dfn}
Our definition of a rectifiable set in  \Cref{dfn:rectset} is  slightly more restrictive than the one  in \cite[Definition 3.2.14]{fed69} in the sense that we assume the set $\setA$ to be compact rather than bounded. This has the advantage that rectifiable sets are always Borel measurable. For the sake of concreteness, consider the following examples. 
\begin{exa}
 For every $n\in\naturals$, let $\ell_n=\{(x,y)\in\reals^2 : x=1/n, y\in[0,1 ] \}$. Then, every $\ell_n$ is the Lipschitz image of $[0,1]$ and, therefore, $1$-rectifiable.   
The sets $\setA_m=\bigcup_{n=1}^m\ell_n$ are $1$-rectifiable for all $m\in\naturals$ owing to \Cref{eq:unicer} below.  The set $\setA_\infty=\bigcup_{n\in\naturals}\ell_n$ is countably $1$-rectifiable. 
\end{exa}

 We need the following 
 quantitative (in terms of explicit Lipschitz constants) version of  \cite[Lemma 3.1]{rikostbo23}, which states that finite unions of $m$-rectifiable sets are  again $m$-rectifiable.

\begin{lem} \label{eq:unicer}
Let $\ell, m,n\in\naturals$ with $m\leq n$. For $k=1,\dots, \ell$ let $\setA_k\subseteq\reals^m$ be nonempty and compact with $s_k=\diam(\setA_k)$ and  let $f_k\colon\setA_k\to\reals^n$  be Lipschitz.   Consider the  set $\setE=\bigcup_{k=1}^\ell f_k(\setA_k)$. 
Then, there is a compact set $\setB \subseteq [0,1]^m$ and a Lipschitz mapping $g\colon \setB \to \reals^n$ with 
$\lVert g\rVert_{L_\infty(\setB)}=\max_{k=1,\dots,\ell} \lVert f_k\rVert_{L_\infty(\setA_k)}$ and 
\begin{align}\label{eq:lipkkk}
\operatorname{Lip}(g)\leq 2 \ell \max\big\{ \diam(\setE), s_1  \operatorname{Lip}(f_1), s_2  \operatorname{Lip}(f_2),\dots,s_\ell \operatorname{Lip}(f_\ell)\big\}
\end{align}
such that  $\setE=g(\setB)$. In particular, $\setE$ is $m$-rectifiable. 
\end{lem}
\begin{proof}
For $k=1,\dots, \ell$,  fix   a closed cube\footnote{In the degenerate case where $\setA_k=\{a_k\}$, we simply set $\setQ_k=\{a_k\}$ and define $\phi_k\colon\{b_k\}\to \{a_k\}$ according to $\phi_k(b_k)= a_k$ with $b_k$ a fixed point in $[0,1]^m$.}  $\setQ_k\subseteq\reals^m$ of sidelength $s_k=\diam(\setA_k)$ such that $\setA_k\subseteq \setQ_k$,  consider the  affine bijection $\phi_k\colon [0,1]^m\to \setQ_k$ with $\operatorname{Lip}(\phi_k)=s_k$,  
 define the affine bijections $ \psi_k\colon\reals^m\to\reals^m$ according to 
\begin{align}
\psi_k(x) = \tp{(2\ell x_1-2(k-1),x_2,x_3,\dots, x_m)},  
\end{align} 
and set $\setB_k= \psi^{-1}(\phi^{-1}(\setA_k ))$.  
Since 
\begin{align}
\psi^{-1}_k(x) = \tp{(x_1/(2\ell) +(k-1)/\ell,x_2,x_3,\dots, x_m)} 
\end{align}
and  $\phi^{-1}(\setA_k ) \subseteq [0,1]^m$, we have 
\begin{align}\label{eq:Bkconst}
 \setB_k\subseteq  [(k-1)/\ell, (2k-1)/{2\ell}]\times [0,1]^{m-1}\quad\text{ for $k=1,\dots,\ell$}
 \end{align}
  by construction.  
  What is more, the sets $\setB_k$ are compact and, by  \cref{eq:Bkconst}, pairwise disjoint subsets of $[0,1]^m$. Now, set $\setB=\bigcup_{k=1}^\ell \setB_k$ and define $g\colon \setB\to \reals^n$ according to 
\begin{align}
g(x) = f_k(\phi_k(\psi_k(x)))\quad\text{for all $x\in\setB_k$ and $k=1,\dots, \ell$}
\end{align} 
so that  $\setE=g(\setB)$. If $x,y\in\setB_k$, then 
\begin{align}\label{eq:lipk}
\lVert g(x)-g(y)\rVert_\infty \leq \operatorname{Lip}(f_k) \operatorname{Lip}(\phi_k) \operatorname{Lip}(\psi_k)\lVert x-y\rVert_\infty
= 2\ell s_k \operatorname{Lip}(f_k)\lVert x-y\rVert_\infty.  
\end{align}
If $x\in\setB_{k_1}$ and $y\in\setB_{k_2}$ with $k_1\neq k_2$, then $\lVert x-y\rVert_\infty \geq 1/(2\ell)$ thanks to \cref{eq:Bkconst}, which implies 
\begin{align}\label{eq:lipkk}
\frac{\lVert g(x)-g(y)\rVert_\infty}{\lVert x-y\rVert_\infty}\leq 2\ell \diam(\setE). 
\end{align}  
Combining \cref{eq:lipk} and \cref{eq:lipkk} yields \cref{eq:lipkkk}. 
\end{proof}



We next introduce measures on (countably) $m$-rectifiable sets. 

\begin{dfn}\label{dfn:recmu}
Let $m,n\in\naturals$ with $m\leq n$ and let $\nu$ be a finite measure on $\reals^n$. Then, we call $\nu$ 
\begin{enumerate}[itemsep=1ex,topsep=1ex]
\renewcommand{\theenumi}{(\roman{enumi})}
\renewcommand{\labelenumi}{(\roman{enumi})}
\item $m$-rectifiable subordinary to $\setE$ if it is Borel regular and supported on the $m$-rectifiable set $\setE\subseteq \reals^n$;\label{item:recmu1}
\item $\colH^m$-rectifiable subordinary to $(\setE,\phi)$ if $\nu =\phi (\colH^m|_\setE)$, where 
$\setE\subseteq \reals^n$ is  $m$-rectifiable  and $\phi\colon \reals^n\to [0,\infty)$ is a $\colH^m$-measurable mapping;\label{item:recmu2}
\item countably $m$-rectifiable subordinary to $\setE$  if it is Borel regular and supported on the   
countably  $m$-rectifiable set $\setE\subseteq \reals^n$;\label{item:recmu3}
\item countably $\colH^m$-rectifiable subordinary to $(\setE,\phi)$ if $\nu =\phi (\colH^m|_\setE)$, where 
$\setE\subseteq \reals^n$ is countably $m$-rectifiable  and $\phi\colon \reals^n\to [0,\infty)$ is a Borel measurable mapping.\label{item:recmu4}
\end{enumerate}
\end{dfn}
Note that in \cref{item:recmu2} and \cref{item:recmu4} of \Cref{dfn:recmu}, we do not assume Borel regularity  for $\nu$, rather it follows directly as shown next.

\begin{lem}\label{lem:recborel}
 (Countably) $\colH^m$-rectifiable measures are Borel regular. In particular, the  measures in
  \cref{item:recmu1}--\cref{item:recmu4} of \Cref{dfn:recmu} are all Radon measures, and  we have the implications 
  \cref{item:recmu2} $\Rightarrow$ \cref{item:recmu1} and  \cref{item:recmu4} $\Rightarrow$ \cref{item:recmu3} in \Cref{dfn:recmu}. 
\end{lem}
\begin{proof}
Let $\setE\subseteq \reals^n$ be (countably)  $m$-rectifiable,   let $\phi\colon \reals^n\to [0,\infty)$ be  $\colH^m$-measurable, and suppose that $\nu$ is (countably) $\colH^m$-rectifiable subordinary to $(\setE,\phi)$. 
Since  $\colH^m|_\setE$ is Borel regular owing to  \cref{item:murest2} in \Cref{thm:murest} and $\phi$ is $\colH^m$-measurable, by \Cref{lem:fmu}, $\nu=\phi(\colH^m|_\setE)$ must be a Borel regular measure. The measures in  \cref{item:recmu1}--\cref{item:recmu4} of \Cref{dfn:recmu} are therefore all Borel regular and, in turn,  Radon measures as they are finite by assumption.    
 The implications 
  \cref{item:recmu2} $\Rightarrow$ \cref{item:recmu1} and  \cref{item:recmu4} $\Rightarrow$ \cref{item:recmu3} in \Cref{dfn:recmu} follow immediately from \Cref{lem:fmu}. 
\end{proof}

The following result states that $m$-rectifiable and $\colH^m$-rectifiable measures in $\reals^n$ can be generated as push-forwards under Lipschitz mappings of Radon measures supported on compact sets in $\reals^m$, which is the main idea behind the second step of the  proof technique  depicted in  \Cref{fig:illustration}.

\begin{thm}\label{thm:recmunu}
Let $m,n\in\naturals$ with $m\leq n$, let $\setA\subseteq\reals^m$ be nonempty and compact, suppose that   $f\colon \setA\to\reals^n$ is a Lipschitz mapping, and consider the $m$-rectifiable set $\setE=f(\setA)$. Then, the following statements hold. 
\begin{enumerate}
\renewcommand{\theenumi}{(\roman{enumi})}
\renewcommand{\labelenumi}{(\roman{enumi})}
\item \label{item:recpush1}If $\nu$ is $m$-rectifiable subordinary to $\setE$,  then, there exists a   
Radon measure $\mu$ on $\setA$  such that $\nu=f\#\mu$. 
\item  \label{item:recpush2}
 Suppose that $\phi\colon \reals^n\to [0,\infty)$ is  $\colH^m$-measurable  and $\nu$ is $\colH^m$-rectifiable subordinary to $(\setE,\phi)$.  Then, there exists a   
Radon measure $\mu$ on $\setA$   such that $\nu=f\#\mu$ with $\mu\ll\colL^m|_{\setA}$. 
Moreover, we have  $\dim_{\mathrm H}(\mu)=m$  provided that $\nu$ is not identically zero. 
 \end{enumerate}
\end{thm}
\begin{proof}
 In  \cref{item:recpush1} and \cref{item:recpush2}, $\nu$ is supported on the 
 $m$-rectifiable set $\setE\subseteq \reals^n$ and a Radon measure (see \Cref{lem:recborel}). 
 The existence of a Radon measure $\mu$ on $\setA$  satisfying $\nu=f\#\mu$ then follows from  
\Cref{lem:munu}.      
 We next prove that $\mu\ll\colL^m|_{\setA}$ if $\nu$ is $\colH^m$-rectifiable subordinary to $(\setE,\phi)$. 
 Toward a contradiction, suppose that there exists a set $\setC\subseteq \setA$ such that $\mu(\setC)>0$ and  $\colL^m(\setC)=0$. Then, we  have 
\begin{align}\label{eq:contradict1}
\nu(f(\setC)) = \mu(f^{-1}(f(\setC)))\geq\mu(\setC)>0. 
\end{align}
Now, Borel regularity of $\colL^m$ implies the existence of a  Borel set $\setB\subseteq \reals^m$ with $\setC\subseteq \setB$ and $\colL^m(\setB)=\colL^m(\setC)=0$. 
Using the properties of Hausdorff measures in \Cref{lem:Hmeasure}, we obtain
\begin{align}\label{eq:measfazero}
\colH^m(f(\setC))\leq \operatorname{Lip}^m(f) \colH^m(\setC) \leq \operatorname{Lip}^m(f) \colH^m(\setB) =  \operatorname{Lip}^m(f) \colL^m(\setB) =0,
\end{align}
which stands in contradiction to \cref{eq:contradict1} since $\nu\ll\colH^m$ by   \Cref{lem:fmu}.  Hence,  $\mu(\setC)>0$ implies  $\colL^m(\setC)>0$ so that  $\mu\ll\colL^m|_{\setA}$. 

It remains to show that $\dim_{\mathrm H}(\mu)=m$ if  $\nu$ is $\colH^m$-rectifiable subordinary to $(\setE,\phi)$ and nontrivial. 
The elementary properties of Hausdorff dimension \cite[Section 3.2]{fa14} imply that  Hausdorff dimension cannot exceed the ambient dimension  
 so that 
 $\dim_{\mathrm H}(\mu)\leq m$. 
Next, note that $\nu(f(\operatorname{spt}(\mu)))>0$ as $\nu=f\#\mu$ and  $\nu$ was assumed to be not identically zero. 
This implies $\colH^m(f(\operatorname{spt}(\mu)))>0$ as $\nu\ll\colH^m$. 
We thus have $\colH^m(\operatorname{spt}(\mu))> 0$ by \cref{eq:itemHLL} in \Cref{lem:Hmeasure}. 
Therefore,  $\dim_{\mathrm H}(\mu)\geq m$ owing to the definition of Hausdorff dimension  \Cref{dfn:HD}. 
\end{proof} 
 
Note that, in \Cref{thm:recmunu},  $\mu$ in \cref{item:recpush1}  can have 
arbitrarily  small $\dim_{\mathrm H}(\mu)$ as we did not impose any regularity conditions on $\nu$, whereas $\mu$ in \cref{item:recpush2}  satisfies $\dim_{\mathrm H}(\mu)=m$, i.e., $m$ is the intrinsic dimension of $\mu$ and no dimension reduction can be achieved by adjusting the support set,  which is a consequence of $\nu\ll\colH^m$.  
Further, in  \cref{item:recpush2} of \Cref{thm:recmunu}, if the function $\phi$  is Borel and $f$ is injective, we have an explicit form of  $\mu$, which follows essentially from the area formula   \Cref{thm:area} and is proved in the next result. In this situation, all ReLU neural networks in the approximation depicted in \Cref{fig:illustration} are explicit, and their expressions can be read from the respective proofs of the results.
  
\begin{lem}\label{lem:density}
Let $\setA\subseteq\reals^m$ be nonempty,  compact, with boundary of Lebesgue measure zero,  suppose that   $f\colon \setA\to\reals^n$ is an injective Lipschitz mapping, and consider the $m$-rectifiable set $\setE=f(\setA)$. Suppose that $\phi\colon \reals^n\to [0,\infty)$ is a Borel measurable mapping and $\nu$ is 
 $\colH^m$-rectifiable subordinary to $(\setE,\phi)$. Then, 
$\nu=f\#\mu$ with $\mu=\psi \colL^m_\setA$  and $\psi =  (J f) (\phi \circ f)$.
\end{lem}
\begin{proof}
We first establish that $\psi$ is Lebesgue measurable by inspecting the individual constituents:  
The function $\phi\circ f\colon\setA\to\reals$ as the composition of  Borel measurable mappings is  Borel measurable. The function $Jf:\setA\to\reals$ is Lebesgue measurable  thanks to  Rademacher's theorem \cite[Theorem 5.1.11]{krpa08}. 
Next,  let $\hat f\colon \reals^m\to \reals^n$ be a Lipschitz mapping satisfying $\hat f|_\setA=f$  \cite[Theorem 3.1]{evga14} and arbitrarily fix a 
Borel set $\setC\subseteq \reals^n$. Then, we have 
\begin{align}
(f\#\mu)(\setC)
&= \int_{f^{-1}(\setC)} \psi \,\mathrm d \colL^m_\setA\label{eq:mufC1}\\
&= \int_{f^{-1}(\setC)} (J \hat f) (\phi \circ f) \,\mathrm d \colL^m\label{eq:usearealeb0}\\
&= \int_{\setC\cap \setE} \phi \,\mathrm d \colH^m\label{eq:usearealeb}\\
&= (\phi \,  \colH^m|_\setE)(\setC)  \\
&= \nu(\setC),\label{eq:mufC2} 
\end{align} 
where  \cref{eq:usearealeb0} follows from $f^{-1}(\setC)\subseteq\setA$ 
and the fact that $J \hat f(x)=J f(x)$ for Lebesgue almost all $x\in \setA$, which follows from  $f(x)=\hat f(x)$ for all $x\in\setA$ and  the fact that  the boundary of $\setA$ has Lebesgue measure zero by assumption, and in 
 \cref{eq:usearealeb} we applied  \Cref{thm:area} upon noting that 
  \begin{align}
\sum_{x\in f^{-1}(\setC)\cap \hat f^{-1}(y)} \phi(f(x)) = \phi(y) \ind{\setE \cap\setC}(y)\quad\text{for all $y\in \reals^n$}
\end{align}
and 
$ f^{-1}(\setC)$ is Borel measurable since $f$ is continuous and  $\setC$ is Borel.  Now, $\nu$ is a Radon measure owing to \Cref{lem:recborel}. If we can show that  $f\#\mu$ is also a Radon measure, then \cref{eq:mufC1}--\cref{eq:mufC2} yields  $\nu=f\#\mu$ by \Cref{lem:borequalmu} and we are done.  Now, $\mu$ is a finite measure since \cref{eq:mufC1}--\cref{eq:mufC2} for the particular choice $\setC=\setE$ yields 
\begin{align}
\mu(\setA)=\mu(f^{-1}(\setE))= (f\#\mu)(\setE) =\nu(\setE) <\infty. 
\end{align}
Moreover, by \Cref{lem:fmu}, $\mu$ is  Borel regular since $\psi$ is Lebesgue measurable and $\colL^m_\setA$ is Borel regular thanks to \cref{item:murest2t} in 
\Cref{cor:tracemeasure}. Therefore, by \Cref{lem:Radon}, 
 $\mu$ is  a Radon measure, so  $f\#\mu$ is also a Radon measure  thanks to \Cref{thm:radonpush}. \end{proof}

We next show that the classes of $m$-rectifiable and $\colH^m$-rectifiable measures is closed under addition. 

\begin{lem}
Let $\setE_1,\setE_2\subseteq \reals^n$ be $m$-rectifiable. Then, the following properties hold:
\begin{enumerate}
\renewcommand{\theenumi}{(\roman{enumi})}
\renewcommand{\labelenumi}{(\roman{enumi})}
\item \label{item:sumrec1}If $\nu_k$ is  $m$-rectifiable subordinary to $\setE_k$ for $k=1,2$, then $\nu_1+\nu_2$ is $m$-rectifiable subordinary to $\setE_1\cup\setE_2$. 
\item\label{item:sumrec2} Suppose that $\phi_k\colon \reals^n\to [0,\infty)$ is  Borel  and $\nu_k$ is  $\colH^m$-rectifiable subordinary to $(\setE_k,\phi_k)$ for $k=1,2$.  Then, $\nu_1+\nu_2$ is  $\colH^m$-rectifiable subordinary to $(\setE_1\cup\setE_2,\ind{\setE_1}\,\phi_1 +\ind{\setE_2}\,\phi_2 )$. 
\end{enumerate}
\end{lem}
\begin{proof}
We first prove  \cref{item:sumrec1}. The set  $\setE_1\cup\setE_2$  is $m$-rectifiable owing to  
\Cref{eq:unicer}. Since $\nu_1+\nu_2$ is supported on $\setE_1\cup\setE_2$ and sums of Borel regular measures are again Borel regular thanks to \Cref{lem:sum}, we can conclude that $\nu_1+\nu_2$ is $m$-rectifiable subordinary to $\setE_1\cup\setE_2$. 
Next, we establish \cref{item:sumrec2}. 
Set 
$\phi =\ind{\setE_1}\,\phi_1 +\ind{\setE_2}\,\phi_2 $  and consider the measure $\nu=\phi\, (\colH^m|_{\setE_1\cup\setE_2})$. Since $\setE_1\cup\setE_2$  is $m$-rectifiable owing to  
\Cref{eq:unicer}, $\nu$ is 
 $\colH^m$-rectifiable subordinary to $(\setE_1\cup\setE_2,\phi)$.  The proof is concluded if we can show that $\nu=\nu_1+\nu_2$. 
Since $\nu$ and $\nu_1+\nu_2$ are Borel regular measures thanks to  \Cref{lem:recborel} and \Cref{lem:sum}, respectively, 
and 
\begin{align}
\nu(\setB)&= \int_\setB  \phi \,\colH^m|_{\setE_1\cup\setE_2} \\
&=\int_\setB  \phi_1 \,\mathrm d \colH^m|_{\setE_1}+\int_\setB  \phi_2 \,\mathrm d \colH^m|_{\setE_2}\\
&=\nu_1(\setB)+\nu_2(\setB)\quad\text{for all Borel sets $\setB\subseteq\reals^n$,}
\end{align}
we have $\nu=\nu_1+\nu_2$ by  \Cref{lem:borequalmu}. 
\end{proof}

The following result states that countably $m$-rectifiable measures can be approximated arbitrarily well by  $m$-rectifiable measures. 
The same holds for countably $\colH^m$-rectifiable measures. 
\begin{lem}\label{lem:approxcountrec}
Suppose that $\setA_k\subseteq \reals^m$ is nonempty and  compact and $f_k\colon \setA_k\to\reals^n$ is Lipschitz for all $k\in\naturals$ and consider the $m$-rectifiable set 
$\setE=\bigcup_{k=1}^\infty f_k(\setA_k)$. Assume that $\setE$ is bounded and set $\setE_\ell=\bigcup_{k=1}^\ell f_k(\setA_k)$  for every  $\ell\in\naturals$. Then, the following statements hold: 
\begin{enumerate}
\renewcommand{\theenumi}{(\roman{enumi})}
\renewcommand{\labelenumi}{(\roman{enumi})}
\item \label{item:approxrect1}Suppose that $\nu$ is countably $m$-rectifiable subordinary to $\setE$ with $\nu(\setE)=1$. Then,    the measure $\nu|_{\setE_\ell}/\nu(\setE_\ell)$ is $m$-rectifiable subordinary to $\setE_\ell$ and satisfies 
\begin{align}\label{eq:W1countrect}
W_1\big(\nu|_{\setE_\ell}/\nu(\setE_\ell),\nu\big) \leq (1-\nu(\setE_\ell)) \diam(\setE)\quad\text{for all $\ell\in\naturals$.}
\end{align}
\item \label{item:approxrect2}Suppose that  $\phi\colon \reals^n\to [0,\infty)$ is  Borel and $\nu$ is countably $\colH^m$-rectifiable subordinary to $(\setE,\phi)$ with $\nu(\setE)=1$.  
Then,  the measure $\nu|_{\setE_\ell}/\nu(\setE_\ell)$ is $\colH^m$-rectifiable subordinary to $(\setE_\ell,\phi/\nu(\setE_\ell))$ and satisfies \cref{eq:W1countrect} for all $\ell\in\naturals$. 
\end{enumerate} 
\end{lem}
\begin{proof}
Suppose first that $\nu$ is countably $m$-rectifiable subordinary to $\setE$ 
and fix $\ell\in\naturals$ arbitrarily. By definition of countable $m$-rectifiability, $\nu$ is   Borel regular and finite. Moreover, $\setE_\ell$ is $m$-rectifiable owing to \Cref{eq:unicer}. Since $\nu|_{\setE_\ell}/\nu(\setE_\ell)$ is supported on $\setE_\ell$ and Borel regular thanks to 
\cref{item:murest2} in  \Cref{thm:murest}, we can conclude that 
 $\nu|_{\setE_\ell}/\nu(\setE_\ell)$ is $m$-rectifiable subordinary to $\setE_\ell$. 
Next, suppose that  $\nu$ is countably $\colH^m$-rectifiable subordinary to $(\setE,\phi)$. Then, 
we have 
\begin{align}
\nu|_{\setE_\ell}/\nu(\setE_\ell)
&=(\phi (\colH^m|_\setE))|_{\setE_\ell}/\nu(\setE_\ell)\\
&= \phi/\nu(\setE_\ell) (\colH^m|_{\setE_\ell}), \label{eq:uselemfmu}
\end{align}
where \cref{eq:uselemfmu} follows from \Cref{lem:fmu}. Hence, since $\setE_\ell$ as the finite union of $m$-rectifiable sets is $m$-rectifiable owing to  \Cref{eq:unicer},  we conclude that $\nu|_{\setE_\ell}/\nu(\setE_\ell)$ is $\colH^m$-rectifiable subordinary to $(\setE_\ell, \phi/\nu(\setE_\ell))$. 
It remains to  prove \cref{eq:W1countrect}   for \cref{item:approxrect1} and \cref{item:approxrect2}. Set $\delta_\ell=1-\nu(\setE_\ell)$. 
We can write 
\begin{align}
W_1\big(\nu|_{\setE_\ell}/(1-\delta_\ell),\nu\big)
&=W_1\Big(  \nu|_{\setE_\ell}+ \delta_\ell(\nu|_{\setE_\ell}/(1-\delta_\ell)),
 \nu|_{\setE_\ell}+ \delta_\ell(\nu|_{\setE\setminus\setE_\ell}/\delta_\ell)  \Big) \\
&\leq W_1(\nu|_{\setE_\ell},\nu|_{\setE_\ell}) + \delta_\ell W_1\big(\nu|_{\setE_\ell}/(1-\delta_\ell),\nu|_{\setE\setminus\setE_\ell}/\delta_\ell\big) \label{eq:boundWrec1}\\
&\leq \delta_\ell \diam(\setE),\label{eq:boundWrec2}
\end{align}
where \cref{eq:boundWrec1} follows from \Cref{eq:lemWrest} and in \cref{eq:boundWrec2} we applied \Cref{lem:boundWboundedS}.
\end{proof}

\subsection{Generating  Rectifiable Measures}
We are now ready to present our results on generating  rectifiable measures through ReLU neural networks. The exposition starts with $m$-rectifiable measures. 
\begin{thm}\label{thm:NNmain1}
Let $m,n,N\in\naturals$ with $m\leq n$ and set 
\begin{align}
\setF_N=\{k/N: k\in \mathbb{Z}\} \cap[-N,N].   
\end{align}
There exists  a collection $\colK^{(m,n)}_N\subseteq \setN_{m,n}$ of ReLU neural networks with 
\begin{align}
\log\Big(\abs{\colK^{(m,n)}_N}\Big)= {n(N+1)^m}\log(2N^2+1)
\end{align} 
such that, for every nonempty and  compact set $\setA\subseteq \reals^m$, every  Lipschitz mapping $f\colon \setA\to \reals^n$ satisfying  
\begin{align}\label{eq:condfarec}
\lVert f \rVert_{L_\infty(\setA)} +\diam(\setA)\operatorname{Lip}(f)\leq N,
\end{align}
and every measure $\nu$ that is $m$-rectifiable subordinary to $\setE:=f(\setA)$,   
there are a 
$\Phi\in \colK^{(m,n)}_N$, a compact set $\setB\subseteq [0,1]^m$,   and a Radon measure $\mu$ on $\setB$ so that $\nu(\setE)=\mu(\setB)$, 
$\Phi|_\setB\#\mu$ is $m$-rectifiable subordinary to $\Phi(\setB)$,  
\begin{align}\label{eq:NN7x}
\operatorname{Lip}(\Phi|_{[0,1]^m})\leq m(\diam(\setA)\operatorname{Lip}(f)+1), 
\end{align}
and 
\begin{align}\label{eq:NN8x}
W_1(\nu, \Phi|_\setB\#\mu)\leq \frac{\nu(\setE)(\diam(\setA)\operatorname{Lip}(f)+1/2)}{N}. 
\end{align}
Moreover, the ReLU neural networks  $\Phi\in\colK^{(m,n)}_N$ have the same architecture and satisfy 
\begin{align}
 \setL(\Phi)&=\lceil\log(m+1)\rceil+4\label{eq:NN3x}\\
 \setM(\Phi) & \leq n(N+1)^m(62m  -28)\\
 \setW(\Phi) &\leq  n(N+1)^m6m\label{eq:NN4x}\\
 \setK( \Phi) &\subseteq \setF_N  \label{eq:NN5x}\\
 \setB(\Phi) &\leq N.\label{eq:NN6x}
 \end{align}
\end{thm}
\begin{proof}
Arbitrarily fix a compact set $\setA\subseteq \reals^m$  and a Lipschitz function $f\colon \setA\to \reals^n$ satisfying  
\cref{eq:condfarec}  and let\footnote{In the degenerate case where $\setA=\{a\}$, we simply set $\setQ=\{a\}$ and $\setB=\{0\}$.}  $\setQ$ be a closed cube of side-length $\diam(\setA)$ such that $\setA\subseteq \setQ$. Denote by $\phi\colon [0,1]^m\to\setQ$ the affine mapping satisfying $\setQ=\phi([0,1]^m)$, and set $\setB=\phi^{-1}(\setA)$. 
The function $g=\tp{(g_1,g_2,\dots,
g_n)}\colon \setB\to \reals^n$ defined according to  $g=f\circ \phi$ is Lipschitz with  $\lVert g\rVert_{L^\infty(\setB)}=\lVert f\rVert_{L^\infty(\setA)}$  and $\operatorname{Lip}(g) \leq \diam(\setA)\operatorname{Lip}(f)$. By construction, we have   
 $\setE=f(\setA)=g(\setB)$. Moreover, by \cref{item:recpush1} in \Cref{thm:recmunu}, there exists a 
 Radon measure $\mu$ on $\setB$  such that $\nu=g\#\mu$, which implies 
 \begin{align}\label{eq:detailsg}
 \nu(\setE)=\mu(g^{-1}(\setE))=\mu(g^{-1}(g(\setB)))=\mu(\setB).
 \end{align} 
 Now, let $\Phi_{j}$ denote the ReLU neural network resulting from  the application of  \Cref{thm:LipNN}   to the Lipschitz mapping $g_j$  satisfying 
\begin{align}\label{eq:Lphi}
\lVert g_j-\Phi_{j}\rVert_{L^\infty(\setB)} \leq \frac{\diam(\setA)\operatorname{Lip(f)}+1/2}{N} 
\end{align}
and
\begin{align}\label{eq:LLphi}
\operatorname{Lip}(\Phi_{j}|_{[0,1]^m})  \leq m(\diam(\setA)\operatorname{Lip}(f) + 1)  \quad\text{for $j=1,\dots,n$}
\end{align}
and set $\Phi=P(\Phi_{1},\dots,\Phi_{n} )$. 
Then,  \cref{eq:LLphi}  implies  that $\Phi$ satisfies \cref{eq:NN7x}. 
Moreover,   
\cref{eq:NN3}--\cref{eq:NN6} for each  $\Phi_{j}$  imply that   
$\Phi$ satisfies \cref{eq:NN3x}--\cref{eq:NN6x}. Next, note that 
\begin{align}
W_1(\nu, \Phi|_\setB\#\mu)
&=W_1(g\#\mu, \Phi|_\setB\#\mu)\\
&\leq \int \lVert g-\Phi|_\setB \rVert_{L_\infty(\setB)} \mathrm d \mu \label{eq:approxW1}\\
&=\int \max_{j=1,\dots,n}\lVert g_j-\Phi_{j}\rVert_{L^\infty(\setB)} \mathrm d \mu \\
&\leq \frac{\mu(\setB)(\diam(\setA)\operatorname{Lip}(f)+1/2)}{N}\label{eq:approxW2}\\
&= \frac{\nu(\setE)(\diam(\setA)\operatorname{Lip}(f)+1/2)}{N},
\end{align}
where \cref{eq:approxW1} follows from \Cref{lem:pushW1}, \cref{eq:approxW2} is by \cref{eq:Lphi}, and 
in \cref{eq:approxW2} we applied \cref{eq:detailsg}, 
which proves \cref{eq:NN8x}. Now, $\Phi|_\setB\#\mu$ is a Radon measure by \Cref{thm:radonpush} and 
supported on the $m$-rectifiable set $\Phi(\setB)$. Therefore,  $\Phi|_\setB\#\mu$ is $m$-rectifiable subordinary to $\Phi(\setB)$.
\end{proof}

We can combine \Cref{thm:NNmain1} with  \Cref{cor:mupushforward} to obtain the following space-filling result 
for $m$-rectifiable probability measures (see \Cref{fig:illustration} for an illustration).

\begin{cor}\label{cor:rec1}
Let $n,N\in\naturals$ and set  
\begin{align}\label{eq:DK}
\setD_{3N} =\{ a/b: a\in\mathbb{Z},b\in\naturals, \abs{a}\leq 4(3N)^{m+1} \ \text{and}\ b\leq 4(3N)^{m+2}\}.    
\end{align} 
There exists a collection $\colJ^{(m,n)}_N\subseteq \setN_{1,n}$ of ReLU neural networks with 
\begin{align}\label{eq:cardJn}
\log\Big(\abs{\colJ^{(m,n)}_N}\Big)\leq {3n(3N)^m}  \log (6N)
\end{align} 
 such  that, for every nonempty and  compact set $\setA\subseteq \reals^m$, every  Lipschitz mapping $f\colon \setA\to \reals^n$ satisfying  
\begin{align}\label{eq:condfarec2}
\lVert f \rVert_{L_\infty(\setA)} +\diam(\setA)\operatorname{Lip}(f)\leq  N,
\end{align}
and every probability measure $\nu$ that is $m$-rectifiable subordinary to $f(\setA)$,   
there is a 
$\Psi\in \colJ^{(m,n)}_N$  satisfying 
\begin{align}\label{eq:finallippsi}
\operatorname{Lip}(\Psi|_{[0,1]})\leq m(6(N+1))^{m+1}
\end{align}
and 
\begin{align}\label{eq:finalWipsi}
W_1\Big( \nu,\Psi\#(\mathcal{L}^{(1)}|_{[0,1]})\Big)\leq \frac{(m+1)(\diam(\setA)\operatorname{Lip}(f)+1)}{N}. 
\end{align}
Moreover, the ReLU neural networks $\Psi\in \colJ^{(m,n)}_N $ have the same architecture and satisfy 
\begin{align}
 \setL(\Psi)&=\lceil\log(m+1)\rceil+6(m-1)+7\label{eq:NN3xyA}\\
 \setM(\Psi) & \leq 78mn(3N)^m\label{eq:magrec}\\
 \setW(\Psi) &\leq  6mn(3N)^m\\
 \setK( \Psi) &\subseteq \setD_{3N}\\
 \setB(\Psi) &\leq  4(3N)^{m+1}.\label{eq:NN6xyA}
 \end{align}
\end{cor}
\begin{proof}
Consider the collections  $\colK^{(m,n)}_N$ and $\colG^{(m)}_{3N}$ from  \Cref{thm:NNmain1} and \Cref{cor:mupushforward} (with $d=m$ and $K=3N$), respectively and set 
\begin{align}
\colJ^{(m,n)}_N =\{ \Phi \circ \rho\circ \Sigma : \Phi\in\colK^{m,n}_N\  \text{and}\ \Sigma \in \colG^{(m)}_{3N} \}. 
\end{align}
Then, we have 
\begin{align}
\abs{\colJ^{(m,n)}_N}
&\leq \abs{\colK^{(m,n)}_N}\abs{\colG^{(m)}_{3N}}\label{eq:needlater1}\\
&\leq (2N^2+1)^{n(N+1)^m} (36N)^{(3N)^m}\\
&\leq (6N)^{3n(3N)^m}, 
\end{align}
which establishes \cref{eq:cardJn}. Since every $\Phi\in\colK^{(m,n)}_N$ satisfies \cref{eq:NN3x}--\cref{eq:NN6x} and every $\Sigma\in \colG_{3N}$ obeys \cref{eq:corK1}--\cref{eq:corK2} for $d=m$ and  $K=3N$, we can conclude that every $\Psi\in \colJ^{(m,n)}_N$ satisfies  \cref{eq:NN3xyA}--\cref{eq:NN6xyA} upon noting that 
\begin{align}\label{eq:FN}
\setF_N=\{k/N: k\in \mathbb{Z}\} \cap[-N,N]\subseteq \setD_{3N}.   
\end{align} 
  Now, arbitrarily fix a compact set $\setA\subseteq \reals^m$, a Lipschitz function $f\colon \setA\to \reals$ satisfying  
\cref{eq:condfarec2}, and a probability measure $\nu$ that is $m$-rectifiable subordinary to $f(\setA)$. Then, 
\Cref{thm:NNmain1} implies the existence of  a compact set $\setB\subseteq [0,1]^m$, a 
Radon probability measure $\mu$ on $\setB$, and a $\Phi\in \colK^{m,n}_N$ satisfying 
\begin{align}\label{eq:NN8xA}
W_1(\nu, \Phi|_\setB\#\mu)\leq \frac{\diam(\setA)\operatorname{Lip}(f)+1/2}{N} 
\end{align}
and 
\begin{align}\label{eq:NN8xA1}
\operatorname{Lip}(\Phi|_{[0,1]^m})\leq m(\diam(\setA)\operatorname{Lip}(f)+1). 
\end{align}
Next, consider the inclusions $i_1\colon\setB\to[0,1]^m$, $i_1(x)=x$ and  $i_2\colon [0,1]^m\to \reals^m$, $i_2(x)=x$ and define 
$\mu_1= i_1\#\mu$ and $\mu_2= i_2\#\mu_1$. 
By construction,  $\mu_2$ is a Borel probability measure on $\reals^d$ with $\operatorname{spt}(\mu_2)\subseteq [0,1]^m$. Therefore,  by \Cref{cor:mupushforward}, there is a $\Sigma\in \colG^{(m)}_{3N}$ satisfying 
\begin{align}\label{eq:rangesigma}
\Sigma([0,1])\subseteq [0,1]^m,
\end{align} 
\begin{align}\label{eq:Lipsigma}
\operatorname{Lip}(\Sigma) \leq 2^{m+1}(3N)^m,    
\end{align}
and 
\begin{align}\label{eq:NN8xB}
W_1\Big( \mu_2,\Sigma\#(\mathcal{L}^{(1)}|_{[0,1]})\Big) \leq 1/N. 
\end{align}
We now set $\Psi=\Phi \circ \rho\circ \Sigma$. 
Next, note that 
\begin{align}
\Psi(t)= \Phi|_{[0,1]^m}(\hat \Sigma(t))\quad\text{for all $t\in [0,1]$}
\end{align}
with $\hat \Sigma:\reals\to [0,1]^m$, $\hat \Sigma (t)= \Sigma((\rho(t)-\rho(1-t)))$ thanks to \cref{eq:rangesigma}. 
We therefore have 
\begin{align}
\operatorname{Lip}(\Psi|_{[0,1]})
&\leq \operatorname{Lip}(\Phi|_{[0,1]^m})\operatorname{Lip}(\hat \Sigma|_{[0,1]})\\
&\leq \operatorname{Lip}(\Phi|_{[0,1]^m})\operatorname{Lip}(\Sigma)\\
&\leq m(6(N+1))^{m+1},\label{eq:combinealllip}
\end{align}
where \cref{eq:combinealllip} follows from \cref{eq:condfarec2}, \cref{eq:NN8xA1}, and \cref{eq:Lipsigma}, which proves \cref{eq:finallippsi}. 
Moreover, we have 
\begin{align}
&W_1\Big( \nu,\Psi\#(\mathcal{L}^{(1)}|_{[0,1]})\Big)\\
&\leq 
W_1\Big( \nu,\Phi|_\setB\#\mu\Big)+ W_1\Big( \Phi|_\setB\#\mu,\Psi\#(\mathcal{L}^{(1)}|_{[0,1]})\Big)\label{eq:usecor0}\\
&\leq \frac{\diam(\setA)\operatorname{Lip}(f)+1/2}{N}+ W_1\Big( \Phi|_\setB\#\mu,\Psi\#(\mathcal{L}^{(1)}|_{[0,1]})\Big),\label{eq:usecor1a}
\end{align}
where \cref{eq:usecor1a} follows from \cref{eq:NN8xA}. To upper-bound the second term in \cref{eq:usecor1a}, we note that 
\begin{align}\label{eq:mumu1}
\Phi|_\setB\#\mu=\Phi|_{[0,1]^m}\#\mu_1
\end{align} and 
\begin{align}
\Psi\#(\mathcal{L}^{(1)}|_{[0,1]})
&=(\Phi|_{[0,1]^m} \circ \hat \Sigma)\#(\mathcal{L}^{(1)}|_{[0,1]}).\label{eq:usekappe0}
\end{align}
Hence, we obtain 
\begin{align}
&W_1\Big( \Phi|_\setB\#\mu,\Psi\#(\mathcal{L}^{(1)}|_{[0,1]})\Big)\label{eq:usecor1pre}\\
&\leq  \operatorname{Lip}(\Phi|_{[0,1]^m}) W_1\Big(\mu_1,\hat \Sigma\#(\mathcal{L}^{(1)}|_{[0,1]})\Big)\label{eq:usecor1}\\
&\leq m(\diam(\setA)\operatorname{Lip}(f) + 1) W_1\Big(\mu_2, \Sigma\#(\mathcal{L}^{(1)}|_{[0,1]})\Big)\label{eq:usecor2}\\
&\leq \frac{m(\diam(\setA)\operatorname{Lip}(f) + 1)}{N}, \label{eq:usecor3}
\end{align}
where \cref{eq:usecor1} follows from \cref{eq:mumu1}, \cref{eq:usekappe0} and \Cref{lem:pushW2}, in \cref{eq:usecor2} we applied 
\cref{eq:NN8xA1}, and \cref{eq:usecor3} is by \cref{eq:NN8xB}. Using \cref{eq:usecor1pre}--\cref{eq:usecor3}
in \cref{eq:usecor1a} establishes \cref{eq:finalWipsi}. 
\end{proof}

We next use \Cref{cor:rec1} to obtain results on the generation of classes of $m$-rectifiable measures, defined next. 

\begin{dfn}\label{dfn:classG}
Let $m,n\in\naturals$ with $m\leq n$ and $C\in(0,\infty)$. The class  $\colG(n,m,C)$ consists of all $m$-rectifiable  probability measures $\nu$ on $\reals^n$ with the property  that there exist a nonempty and compact   $\setA\subseteq \reals^m$ and a  Lipschitz mapping $f\colon \setA\to\reals^n$ 
satisfying 
\begin{align}\label{eq:condC}
\lVert f \rVert_{L_\infty(\setA)}+\diam(\setA)\operatorname{Lip}(f) \leq C 
\end{align}
such that   $\nu$ is   $m$-rectifiable subordinary to $f(\setA)$. 
\end{dfn}

Our approximation result for classes of $m$-rectifiable measures is now as follows. 

\begin{cor}\label{cor:rec2}
Let $m,n\in\naturals$ with $m\leq n$ and $C\in(0,\infty)$ and consider the class $\colG(n,m,C)$ as per \Cref{dfn:classG}. 
Then, for every $\varepsilon \in (0,1]$, there exists a collection $\colR(n,m,C,\varepsilon)$ of ReLU neural networks  
of cardinality 
\begin{align}\label{eq:finalcor3bits}
\abs{\colR(n,m,C,\varepsilon)}\leq 2^{b(\varepsilon)}, 
\end{align}
where 
\begin{align}
b(\varepsilon) = 3n(3\lceil (C+1)(m+1)/\varepsilon\rceil)^m \log(6\lceil (C+1)(m+1)/\varepsilon\rceil), 
\end{align}
such that for every $\nu\in \colG(n,m,C)$, there is  a $\Psi\in \colR(n,m,C,\varepsilon)$ satisfying 
\begin{align}
\operatorname{Lip}(\Psi|_{[0,1]})\leq m(6(\lceil (m+1)(C+1)/{\varepsilon}\rceil +1))^{m+1} 
\end{align}
and 
\begin{align}\label{eq:bounderrornu}
W_1\Big( \nu,\Psi\#(\mathcal{L}^{(1)}|_{[0,1]})\Big)\leq \varepsilon. 
\end{align}
Moreover, the ReLU neural networks $\Psi\in \colR(n,m,C,\varepsilon)$ have the same architecture and satisfy \cref{eq:NN3xyA}--\cref{eq:NN6xyA} with $N=\lceil (m+1)(C+1)/{\varepsilon}\rceil$ . 
\end{cor}
\begin{proof} 
Follows from \Cref{cor:rec1} with  $N=\lceil (m+1)(C+1)/{\varepsilon}\rceil$ and 
$\colR(n,m,C,\varepsilon)=\colJ^{(m,n)}_{N}$ upon noting that \cref{eq:condfarec2} is satisfied as  \cref{eq:condC} implies 
\begin{align}
\lVert f \rVert_{L_\infty(\setA)}+\diam(\setA)\operatorname{Lip}(f)\leq   C \leq N 
\end{align}
and \cref{eq:finalWipsi} in combination with \cref{eq:condC} yields  
\begin{align}
W_1\Big( \nu,\Psi\#(\mathcal{L}^{(1)}|_{[0,1]})\Big)\leq \frac{(m+1)(\diam(\setA)\operatorname{Lip}(f)+1)}{N}  \leq  \varepsilon. 
\end{align}
\end{proof}
Qualitatively, \Cref{cor:rec2} states that the rate at which $b(\varepsilon)$ in \cref{eq:finalcor3bits} tends to $\infty$ as $\varepsilon\to 0$  equals the rectifiability parameter, i.e., we have 
\begin{align}\label{eq:qualirec}
\inf\{s\in (0,\infty): \lim_{\varepsilon\to 0}\varepsilon^s b(\varepsilon)<\infty\}  =m.
\end{align}
 We next establish that this result  is metric entropy optimal by constructing a matching lower bound. 
 \begin{prp}\label{prp:converse}
 Let $m,n\in\naturals$ with $m\leq n$ and $C\in(0,\infty)$ and consider the metric space $(\colG(n,m,C), W_1)$
 with $\colG(n,m,C)$ as per \Cref{dfn:classG}.  Then, we have 
 \begin{align}\label{eq:converse}
 \inf\{s\in(0,\infty) :\liminf_{\varepsilon\to 0 }H_\varepsilon(\colG(n,m,C),W_1)\varepsilon^s<\infty\} \geq m. 
 \end{align}
 \end{prp}
 \begin{proof} 
 Let $\setQ$ denote the cube of side-length $\delta=C/2$ centered at the origin in $\reals^m$ and let $f\colon\setQ\to \reals^n$ be defined according to $f(x_1,\dots, x_m)=\tp{(x_1,\dots,x_m,0,\dots,0)}$. Then, we have 
 \begin{align}
\lVert f \rVert_{L_\infty(\setQ)}+\diam(\setQ)\operatorname{Lip}(f) \leq  2\delta \leq C. 
 \end{align}
Therefore, if $\mu\in\setP(\setQ)$, then $f\#\mu\in \colG(n,m,C)$. Next, let 
$\setA=\{\nu\in\setP(\reals^n):\nu(\reals^n\setminus f(\setQ))=0 \}$. 
By monotonicity of the Kolmogorov $\varepsilon$-entropy, it is sufficient to establish 
\begin{align}\label{eq:converse1}
 \inf\{s\in(0,\infty) :\liminf_{\varepsilon\to 0 }H_\varepsilon(\setA,W_1)\varepsilon^s<\infty\} \geq m. 
 \end{align}
Now, consider the projection $g\colon f(\setQ)\to \setQ$, $\tp{(x_1,\dots,x_m,0,\dots,0)}\to \tp{(x_1,\dots,x_m)}$. 
If $\mu_1,\mu_2\in \setP(\setQ)$, then 
\begin{align}
W_1(\mu_1,\mu_2)
&= W_1(g\#(f\#\mu_1),g\#(f\#\mu_2))\\
&\leq \operatorname{Lip}(g) W_1(f\#\mu_1,f\#\mu_2)\label{eq:uselipW1}
\end{align}
with $f\#\mu_1,f\#\mu_2\in\setA$, where \cref{eq:uselipW1} follows from \Cref{lem:pushW2}. It is therefore sufficient to establish 
 \begin{align}\label{eq:converse2}
 \inf\{s\in(0,\infty) :\liminf_{\varepsilon\to 0 }H_\varepsilon(\setP(Q),W_1)\varepsilon^s<\infty\} \geq m. 
 \end{align} 
We next establish a homothetic  embedding from Hilbert cubes into $(\setP(Q),W_2)$. 
We will closely follow  the ideas of the proof of  \cite[Theorem 1.5]{kl12}.
Let $(c_n)_{n\in\naturals}$ be an arbitrary  sequence of positive real numbers satisfying 
 \begin{align}\label{eq:c_n}
 \sum_{n\in\naturals}c_n^m =  (\delta/3)^m, 
 \end{align}
 which will be specified later in the proof. Then, the Auerbach-Banach-Mazur-Ulam theorem particularized to parallelepipeds \cite[Lemma 2]{ko57} with $D=V^{1/m}=\delta/3$ implies that there exists a family of homotheties $h_n\colon \setQ\to \setQ$ of pairwise disjoint images, where every $\setQ_n:=h_n(\setQ)$ is a cube of side-length $ c_n$ for all $n\in\naturals$. Now, every 
 cube $\setQ_n$ contains a subcube $\setQ_n^\prime$ of side-length $ c_n/4$ such that (note that the diameter equals the side-length for $\lVert\,\cdot\,\rVert_\infty$) $\diam(\setQ_n^\prime)=c_n/4$ is strictly smaller than the distance of $\setQ_n^\prime$ to the boundary of $\setQ_n$. We can therefore conclude that there exists a family of homotheties $h_n^\prime\colon \setQ\to \setQ$ with pairwise disjoint images, where every $\setQ_n^\prime :=h_n^\prime(\setQ)$ is a cube of side-length $ c_n/4$ for all $n\in\naturals$, satisfying the following property:  For every $n,k\in\naturals$ with $n\neq k$, it holds that 
 \begin{align}\label{eq:wassdiagonal}
 \lVert x-y\rVert_\infty <  \lVert x-z\rVert_\infty\quad\text{for all $x,y\in \setQ_n^\prime$ and $z\in \setQ_k^\prime$.}
 \end{align}
Now, let   $(a_n)_{n\in\naturals}$ be an arbitrary  sequence of positive real numbers satisfying $\sum_{n\in\naturals}a_n^{2m/(m+2)} < \infty$ and set $b_n=a_n^{2m/(m+2)}/\alpha$ and $c_n=a_n^{2/(m+2)}/\beta$ for all $n\in\naturals$, where we set $\alpha=\sum_{n\in\naturals}a_n^{2m/(m+2)}$ and $\beta = 3\alpha^{1/m}/\delta$.  These choices guarantee $\sum_{n\in\naturals}b_n=1$ and  (see requirement \cref{eq:c_n}) $\sum_{n\in\naturals}c_n^m = (\delta/3)^m$. 
  Next, consider the Hilbert cube $\mathrm{HC}:=\setQ^\naturals$ equipped with the metric 
 \begin{align}\label{eq:metricHC}
 \rho_{(a_n)_{n\in\naturals}}(x,y)=\sqrt{\sum_{n\in\naturals} a_n^2 \lVert x_n-y_n\rVert_\infty}
 \end{align}
 for all $x=(x_n)_{n\in\naturals}, y=(y_n)_{n\in\naturals}$ with $x,y\in \mathrm{HC}$ and the mapping 
 \begin{align}
 f\colon (\mathrm{HC}, \rho_{(a_n)_{n\in\naturals}})&\to (\setP(\setQ),W_2)\\
 (x_n)_{n\in\naturals}&\mapsto \sum_{n\in\naturals}b_n\delta_{h_n^\prime(x_n)}.
 \end{align}
 We now show that $f$ is a homothetic  embedding.   
 To this end, pick $x=(x_n)_{n\in\naturals}, y=(y_n)_{n\in\naturals}$ with $x,y\in \mathrm{HC}$ arbitrarily. 
 Let $\pi$ be a coupling with marginals $\mu=\sum_{n\in\naturals}b_n\delta_{h_n^\prime(x_n)}$ and  $\nu=\sum_{n\in\naturals}b_n\delta_{h_n^\prime(y_n)}$ and set $\setA=\{h_n^\prime(x_n):n\in\naturals\}$ and $\setB=\{h_n^\prime(y_n):n\in\naturals\}$. 
Then, we have  
\begin{align}
\pi((\setQ\times\setQ)\setminus\{\setA\times\setB\})
&\leq \pi(((\setQ\setminus \setA)\times \setQ)\cup(\setQ\times   (\setQ\setminus \setB)))\\
&\leq \pi((\setQ\setminus \setA)\times \setQ) + \pi(\setQ\times   (\setQ\setminus \setB))\\
&= \mu(\setQ\setminus \setA) +\nu(\setQ\setminus \setB)\label{eq:usecoupling}\\
&=0,
\end{align}
where \cref{eq:usecoupling} follows from the fact that  $\pi$ is a coupling with marginals $\mu$ and $\nu$. This implies $\pi =\sum_{i,j\in\naturals} d_{i,j} \delta_{h_i^\prime(x_i)}\times\delta_{h_j^\prime(y_j)}$ with $d_{i,j}\in (0,\infty)$  for all $i,j\in\naturals$ satisfying $\sum_{i,j\in\naturals} d_{i,j}=1$. Since  cross-transport is more expensive than matching corresponding atoms owing to \cref{eq:wassdiagonal},  
the coupling minimizing the $2$-Wasserstein distance is 
$\pi= \sum_{n\in\naturals} b_{n} \delta_{h_n^\prime(x_n)}\times\delta_{h_n^\prime(y_n)}$ so that 
\begin{align}
\W^2_2(\mu,\nu) 
&= \sum_{n\in\naturals} b_n \lVert h_n^\prime (x_n)-  h_n^\prime (y_n) \rVert_\infty^2\\
&= \frac{1}{16}\sum_{n\in\naturals} b_n c_n^2 \lVert x_n-  y_n \rVert_\infty^2\\
&= \frac{1}{16\alpha \beta^2}\sum_{n\in\naturals} a_n^2  \lVert x_n- y_n \rVert_\infty^2\\
&=\frac{1}{16\alpha \beta^2} \rho^2_{(a_n)_{n\in\naturals}}(x,y),
\end{align}
which establishes that $f$ is indeed a homothetic  embedding. 
 
Finally, for every $\gamma\in ((m+2)/(2m),\infty)$, consider the Hilbert cube $\mathrm{HC}=\setQ^\naturals$ equipped with the metric  $d_\gamma:=d_{(n^{-\gamma})_{n\in\naturals}}$ as per \cref{eq:metricHC}. Then, 
we have 
\begin{align}
& \inf\{s\in(0,\infty) :\liminf_{\varepsilon\to 0 }H_\varepsilon(\setP(Q),W_1)\varepsilon^s<\infty\}\label{eq:usekloeckner0}\\
&=\inf\{s\in(0,\infty) :\liminf_{\varepsilon\to 0 }H_\varepsilon(\setP(Q),W_2)\varepsilon^s<\infty\}\label{eq:usekloeckner0a}\\
&\geq  \inf\{s\in(0,\infty) :\liminf_{\varepsilon\to 0 }H_\varepsilon(\mathrm{HC},d_\gamma)\varepsilon^s<\infty\}\label{eq:usekloeckner0b}\\
&\geq \frac{2}{2\gamma -1}\quad\text{for all $\gamma \in ((m+2)/(2m),\infty)$,}\label{eq:usekloeckner}
\end{align}
where \cref{eq:usekloeckner0a} is by  \Cref{lem:equiv}, in \cref{eq:usekloeckner0b} we used the fact that $f$ is a homothetic embedding, and 
 \cref{eq:usekloeckner} follows from \cite[Propositions 3.2  and  4.3]{kl12} upon noting that 
\begin{align}
&\inf\{s\in(0,\infty): \liminf_{\varepsilon\to 0 }H_\varepsilon(\mathrm{HC},d_\gamma)\varepsilon^s<\infty\}\\
&= \inf\{s\in(0,\infty): \liminf_{\varepsilon\to 0 }e^{H_\varepsilon(\mathrm{HC},d_\gamma)}e^{-\varepsilon^{-s}}<\infty\}
\end{align}
owing to \Cref{lem:equivMdim}.
 Taking the limit $\gamma\to (m+2)/(2m)$ in \cref{eq:usekloeckner0}--\cref{eq:usekloeckner} establishes \cref{eq:converse2}. 
 \end{proof}
\subsection{Generating Countably $m$-rectifiable Measures}

The result for countably $m$-rectifiable measures, stated next,  follows from combining \Cref{eq:unicer,lem:approxcountrec} and  \Cref{thm:NNmain1}. The main idea is to approximate countably $m$-rectifiable measures by $m$-rectifiable measures that are supported on a finite number $\ell$ of  components of the underlying countably $m$-rectifiable sets, where the approximation error tends to zero as $\ell\to\infty$.

\begin{thm} \label{thm:NNmain2}
Let $m,n,\ell,N\in\naturals$ with $m\leq n$ and set 
\begin{align}
\setF_N=\{k/N: k\in \mathbb{Z}\} \cap[-N,N]. 
\end{align}
There exists a collection $\colK^{(m,n)}_N\subseteq \setN_{m,n}$ of ReLU neural networks with 
\begin{align}
\abs{\colK^{(m,n)}_N}=(2N^2+1)^{n(N+1)^m},
\end{align} 
 such that for every sequence $\{\setA_i\}_{i\in\naturals}$ of nonempty and compact sets $\setA_i\subseteq \reals^m$, every sequence $\{f_i \}_{i\in\naturals}$ of Lipschitz mappings $f_i\colon \setA_i\to \reals^n$ satisfying  
\begin{align}\label{eq:condfacountrec}
\max_{k=1,\dots,\ell} \lVert f_k\rVert_{L_\infty(\setA_k)}+\widehat \lambda \leq N
\end{align}
with  
\begin{align}
\widehat \lambda=2 \ell \max\big\{ \diam(\setE), \diam(\setA_1) \operatorname{Lip}(f_1), \dots,\diam(\setA_\ell)  \operatorname{Lip}(f_\ell)\big\}, 
\end{align}
and every probability measure $\nu$ that is countably $m$-rectifiable subordinary to $\setE=\bigcup_{i\in\naturals} f_i(\setA_i)$,   
there is a 
$\Phi\in \colK^{(m,n)}_N$, a compact set $\setB\subseteq [0,1]^m$,   and a Radon measure $\mu$ on $\setB$ so that 
$\Phi|_\setB\#\mu$ is $m$-rectifiable subordinary to $\Phi(\setB)$ and satisfies 
\begin{align}
\operatorname{Lip}(\Phi|_{[0,1]^m})\leq  m(\widehat \lambda+ 1)  
\label{eq:NN7y}
\end{align}
and 
\begin{align}\label{eq:NN8y}
W_1(\nu, \Phi|_\setB\#\mu)\leq \frac{\nu(\setE)(\widehat \lambda+ 1/2)}{N} + (1-\nu(\widehat \setE)) \diam(\setE) 
\end{align}
with 
\begin{align}\label{eq:widehatE}
\widehat \setE=\bigcup_{i=1}^\ell f_i(\setA_i). 
\end{align}
Moreover, the ReLU neural networks $\Phi \in \colK^{(m,n)}_N$ have the same architecture and satisfy
\begin{align}
 \setL(\Phi)&= \lceil\log(m+1)\rceil+4\label{eq:NN3y}\\
 \setM(\Phi) & \leq n(N+1)^m(62m  -28)\\
 \setW(\Phi) &\leq  n(N+1)^m6m\label{eq:NN4y}\\
 \setK( \Phi) &\subseteq \setF_N   \label{eq:NN5y}\\
 \setB(\Phi) &\leq N.\label{eq:NN6y}
 \end{align}
\end{thm}
\begin{proof}
Arbitrarily fix  a sequence $\{\setA_i\}_{i\in\naturals}$ of compact sets $\setA_i\subseteq \reals^m$, a  sequence $\{f_i \}_{i\in\naturals}$ of Lipschitz mappings $f_i\colon \setA_i\to \reals^n$ satisfying  \cref{eq:condfacountrec}, and 
a  probability measure $\nu$ that is  countably $m$-rectifiable subordinary to $\bigcup_{i\in\naturals} f_i(\setA_i)$ and  set  $\widehat \nu=\nu|_{\widehat \setE}/\nu(\widehat \setE)$ with $\widehat \setE$ as per \cref{eq:widehatE}. Thanks to \cref{item:approxrect1} in \Cref{lem:approxcountrec}, $\widehat \nu$ is  $m$-rectifiable subordinary to $\widehat \setE$ and satisfies 
\begin{align}\label{eq:W1temp}
W_1(\nu,\widehat \nu)\leq (1-\nu(\widehat \setE))\diam(\setE). 
\end{align} 
By \Cref{eq:unicer}, there exist a compact set $\setA\subseteq [0,1]^m$ and a Lipschitz mapping $g\colon \setA\to \reals^n$ with  
\begin{align}
\lVert g \rVert_{L_\infty(\setA)}\leq \max_{k=1,\dots,\ell} \lVert f_k\rVert_{L_\infty(\setA_k)}
\end{align}
 and $\operatorname{Lip}(g)\leq  \widehat \lambda$ satisfying $\widehat \setE =g(\setA)$. Now, 
 \Cref{thm:NNmain1}  implies the existence of  a set $\setB\subseteq [0,1]^m$, a ReLU neural network $\Phi\in  \colK^{(m,n)}_N$, and a Radon probability measure $\mu$ on $\setB$ satisfying   
 \begin{align}\label{eq:NN7xn}
\operatorname{Lip}(\Phi|_{[0,1]^m})\leq m(\operatorname{Lip}(g)+1)\leq m(\widehat \lambda+1)
\end{align}
and 
\begin{align}\label{eq:NN8xn}
W_1(\widehat \nu, \Phi|_\setB\#\mu)\leq \frac{\nu(\setE)(\operatorname{Lip}(g)+1/2)}{N}\leq  \frac{\nu(\setE)(\widehat \lambda+1/2)}{N}. 
\end{align} 
Combining \cref{eq:W1temp} with \cref{eq:NN8xn} yields \cref{eq:NN8y}. 
Finally, \Cref{thm:NNmain1}  implies that  every $\Phi\in\colK^{(m,n)}_N$ satisfies   \cref{eq:NN3y}--\cref{eq:NN6y}.  
 \end{proof}

By the same token, we can combine \Cref{eq:unicer,lem:approxcountrec} with \Cref{cor:rec1} to obtain 
 the following space-filling result 
for countably $m$-rectifiable probability measures. 


\begin{cor}\label{cor:rec3}
Let $m,n,\ell,N\in\naturals$ with $m\leq n$ and set 
\begin{align}\label{eq:DKC}
\setD_{3N} =\{ a/b: a\in\mathbb{Z},b\in\naturals, \abs{a}\leq 4(3N)^{m+1}, \ \text{and}\ b\leq 4(3N)^{m+2}\}.   
\end{align} 
There exists a collection $\colJ^{(m,n)}_N\subseteq \setN_{1,n}$ of ReLU neural networks with 
\begin{align}\label{eq:cardJn2}
\log\Big(\abs{\colJ^{(m,n)}_N}\Big)\leq {3n(3N)^m}\log (6N) 
\end{align} 
such that, for every sequence $\{\setA_i\}_{i\in\naturals}$ of nonempty and  compact sets $\setA_i\subseteq \reals^m$, every sequence $\{f_i \}_{i\in\naturals}$ of Lipschitz mappings $f_i\colon \setA_i\to \reals^n$ satisfying  
\begin{align}\label{eq:condfacountrec2}
\max_{k=1,\dots,\ell} \lVert f_k\rVert_{L_\infty(\setA_k)}+\widehat \lambda \leq N, 
\end{align}
where we set 
\begin{align}\label{eq:lambdaelldefa}
\widehat \lambda=2 \ell \max\big\{ \diam(\setE), \diam(\setA_1) \operatorname{Lip}(f_1), \dots,\diam(\setA_\ell)  \operatorname{Lip}(f_\ell)\big\}, 
\end{align}
and every probability measure $\nu$ that is countably $m$-rectifiable subordinary to $\setE=\bigcup_{i\in\naturals} f_i(\setA_i)$,   
there is a $\Psi\in\colJ^{(m,n)}_N$ satisfying 
\begin{align}
\operatorname{Lip}(\Psi|_{[0,1]})\leq m(6(N+1))^{m+1}
\end{align}
and
\begin{align}\label{eq:bounderrornuC}
W_1\Big( \nu,\Psi\#(\mathcal{L}^{(1)}|_{[0,1]})\Big)\leq \frac{(m+1)(\widehat \lambda+ 1)}{N} + (1-\nu(\widehat \setE)) \diam(\setE). 
\end{align}
with 
\begin{align}\label{eq:widehatE2a}
\widehat \setE=\bigcup_{i=1}^\ell f_i(\setA_i). 
\end{align}
Moreover, the ReLU neural networks $\Psi\in \colJ^{(m,n)}_N$ have the  same architecture and satisfy
\begin{align}
 \setL(\Psi)&= \lceil\log(m+1)\rceil+6(m-1)+7\label{eq:NN3xyz}\\
 \setM(\Psi) & \leq 78mn(3N)^m\\
 \setW(\Psi) &\leq 6mn(3N)^m\label{eq:NN4xyz}\\
 \setK( \Psi) &\subseteq \setD_{3N}\\
 \setB(\Psi) &\leq 4(3N)^{m+1}.\label{eq:NN6xyz}
 \end{align}
\end{cor}
\begin{proof}
Arbitrarily fix a sequence $\{\setA_i\}_{i\in\naturals}$ of compact sets $\setA_i\subseteq \reals^m$, a  sequence $\{f_i \}_{i\in\naturals}$ of Lipschitz mappings $f_i\colon \setA_i\to \reals^n$ satisfying  \cref{eq:condfacountrec2}, and 
a  measure $\nu$ that is  countably $m$-rectifiable subordinary to $\bigcup_{i\in\naturals} f_i(\setA_i)$  and  set  $\widehat \nu=\nu|_{\widehat \setE}/\nu(\widehat \setE)$ with $\widehat \setE$ as per \cref{eq:widehatE2a}.  Thanks to \cref{item:approxrect1} in \Cref{lem:approxcountrec}, $\widehat \nu$ is  $m$-rectifiable subordinary to $\widehat \setE$ and satisfies 
\begin{align}\label{eq:W1tempHH}
W_1(\nu,\widehat \nu)\leq (1-\nu(\widehat \setE))\diam(\setE). 
\end{align} 
By \Cref{eq:unicer}, there exist a compact set $\setA\subseteq [0,1]^m$ and a Lipschitz mapping $g\colon \setA\to \reals^n$ with  
\begin{align}
\lVert g \rVert_{L_\infty(\setA)}\leq \max_{k=1,\dots,\ell} \lVert f_k\rVert_{L_\infty(\setA_k)}
\end{align}
 and $\operatorname{Lip}(g)\leq  \widehat \lambda$ satisfying $\widehat \setE =g(\setA)$. 
 Application of \Cref{cor:rec1} to $\widehat \nu$  with  $f=g$ in combination with \cref{eq:W1tempHH} establishes   the result. 
\end{proof}

We next use \Cref{cor:rec3} to obtain results on the generation of classes of countably $m$-rectifiable measures, defined next. 


\begin{dfn}\label{dfn:classGC}
Let $m,n\in\naturals$ with $m\leq n$  and $C\in(0,\infty)$. Further, suppose that $\kappa\colon (0,1]\to\naturals $  is monotonically decreasing. The class  $\colU(n,m,C,\kappa)$ consists of all   probability measures $\nu$ on $\reals^n$ that have the following  properties:   There  exist  
a sequence $\{\setA_i\}_{i\in\naturals}$ of nonempty and compact sets $\setA_i\subseteq \reals^m$ and a sequence $\{f_i \}_{i\in\naturals}$ of Lipschitz mappings $f_i\colon \setA_i\to \reals^n$ such that 
\begin{enumerate}
\renewcommand{\theenumi}{(\roman{enumi})}
\renewcommand{\labelenumi}{(\roman{enumi})}
\item\label{item:nucount5} $2\diam(\setE)\leq C$;
\item \label{item:nucount4} $\nu$ is countably $m$-rectifiable subordinary to $\setE=\bigcup_{k\in\naturals}f_k(\setA_k)$;
\item \label{item:nucount1}$\nu$  satisfies the decay condition   
\begin{align}\label{eq:decaynu}
C\Bigg(1-\nu\Big(\bigcup_{k=1}^{\kappa(\varepsilon)} f_k(\setA_k)\Big)\Bigg)\leq \varepsilon\quad\text{for all $\varepsilon\in (0,1]$;}
\end{align}
\item \label{item:nucount2}the mappings are uniformly upper-bounded by 
\begin{align}
 \max_{k=1,\dots,\kappa(\varepsilon)}\big\{\lVert f_k \rVert_{L_\infty(\setA_k)}\big\} \leq 2C\kappa(\varepsilon)/\varepsilon\quad\text{for all $\varepsilon\in (0,1]$;}
\end{align} 
\item  \label{item:nucount3}the Lipschitz constants and the diameters of the sets are upper-bounded as  
\begin{align}
2\sup_{k\in\naturals} \diam(\setA_k) \operatorname{Lip}(f_k) \leq C.
\end{align}
\end{enumerate}
\end{dfn}

The individual properties  in \Cref{dfn:classGC} can be interpreted as follows. \cref{item:nucount5} implies that $\setE$ is bounded.   
\cref{item:nucount4} is the standard property  of  countably $m$-rectifiability. 
\cref{item:nucount1} amounts in a decay condition on the measure. It tells us how well we can approximate $\nu$ by the  $m$-rectifiable measure $\nu|_{\bigcup_{k=1}^{\kappa(\varepsilon)} f_k(\setA_k)}$. 
\cref{item:nucount2} and  \cref{item:nucount3} impose regularity conditions on the Lipschitz functions. 

For the sake of concreteness, we next present an example of a measure falling into the class $\colU(2,1,2,\kappa)$. 
\begin{exa}
 For every $n\in\naturals$, let $\ell_n=\{(x,y)\in\reals^2 : x=1/n, y\in[0,1 ] \}$. 
 Consider the countably $1$-rectifiable set 
 $\setA_\infty=\bigcup_{n\in\naturals}\ell_n$. Then, $\diam(\{\setA_\infty\})=\diam([0,1]^2)=1$. For every $n\in\naturals$, consider the mapping $f_n\colon [0,1]\to \reals^2$, $f(y)=\tp{(1/n,y)}$ with $\operatorname{Lip}(f_n)=1$.  Then, \cref{item:nucount3} is satisfied for $C=2$. Take  a sequence $(a_n)_{n\in\naturals}$ in $[0,1]^\naturals$ satisfying $\sum_{n\in\naturals} a_n=1$. Further, suppose that $\nu_n$ is the uniform Borel probability measure on $\ell_n$ for all $n\in\naturals$ and set $\nu_\infty=\sum_{n\in\naturals} a_n \nu_n$, which is a Borel probability measure thanks to  \Cref{lem:infsummeasure}.  
 Now, suppose that $a_n=(e-1)e^{-n}$, i.e, we have exponential decay of $\nu$ on the $\ell_n$'s. 
 Then, we have 
 \begin{align}
 \sum_{n=1}^k a_n = (e-1)\sum_{n=1}^k e^{-n} = 1-e^{-k}\quad\text{for all $n\in\naturals$,}
 \end{align} 
 so that \cref{item:nucount1}  is satisfied for  $\kappa(\varepsilon)=\lceil \log(2/\varepsilon)/\log(e)\rceil$. 
 Next, assume that $a_n=n^{-\alpha}/\zeta(\alpha)$ with $\alpha\in (1,\infty)$ and $\zeta(\alpha)=\sum_{n\in\naturals} n^{-\alpha}$, i.e.,  we have polynomial  decay of $\nu$ on the $\ell_n$'s. 
Euler's summation formula  \cite[Theorem 3.1]{ap76} implies 
  \begin{align}
   \sum_{n=1}^k n^{-\alpha} 
   &=\int_1^kt^{-\alpha}\, \mathrm d t-\alpha  \int_1^k\frac{t-\lfloor t\rfloor}{t^{\alpha-1}}\, \mathrm d t\label{eq:zeta1}\\  
   &= -\frac{k^{1-\alpha}}{\alpha-1} +f(\alpha) +g(\alpha,k)\label{eq:zeta2}
 \end{align} 
with
\begin{align}
f(\alpha)= 1+\frac{1}{\alpha-1} - \alpha  \int_1^\infty\frac{t-\lfloor t\rfloor}{t^{\alpha-1}}\, \mathrm d t 
\end{align}
and
\begin{align}
g(\alpha,k)= \alpha \int_k^\infty\frac{t-\lfloor t\rfloor}{t^{\alpha-1}}\, \mathrm d t\quad\text{for all $k\in\naturals$.}
\end{align}
 Now, $g(\alpha,k)\geq 0$ with $g(\alpha,k)\leq k^{-\alpha}$ for all $k\in\naturals$.  Taking the limit $k\to \infty$ in \cref{eq:zeta1}--\cref{eq:zeta2} therefore yields $f(\alpha)=\zeta(\alpha)$ so that 
 \begin{align}
  \sum_{n=1}^k n^{-\alpha} = -\frac{k^{1-\alpha}}{\alpha-1} + \zeta(\alpha) + g(\alpha,k)\quad\text{for all $k\in\naturals$.}
 \end{align}
  We can therefore  set 
  \begin{align}
  \kappa(\varepsilon)=\Big(\frac{2}{\varepsilon (\alpha-1)\zeta(\alpha)}\Big)^\frac{1}{\alpha -1}
  \end{align}  
  to satisfy \cref{item:nucount1}. 
  Finally, \cref{item:nucount2} holds since 
 \begin{align}
\lVert f_n \rVert_{L_\infty([0,1])} \leq 1 \quad\text{for all $n\in\naturals$ and $\varepsilon\in (0,1]$.}
 \end{align}
\end{exa}

Our approximation result for classes of countably $m$-rectifiable measures is now as follows. 
\begin{cor}\label{cor:rec2C}
Let $m,n\in\naturals$ with $m\leq n$ and $C\in(0,\infty)$ and let  $\kappa\colon (0,1]\to\naturals $  be  monotonically decreasing. 
 Consider the class $\colU(n,m,C,\kappa)$ as per \Cref{dfn:classGC}. 
Then, for every $\varepsilon \in (0,1]$, there exists a collection $\colR(n,m,C,\kappa,\varepsilon)$ of ReLU neural networks  
of cardinality 
\begin{align}\label{eq:finalcor3bitsC}
\abs{\colR(n,m,C,\kappa,\varepsilon)}\leq 2^{b(\varepsilon)}
\end{align}
with 
\begin{align}\label{eq:upperbvarepsilon}
b(\varepsilon) ={3n(3N(\varepsilon))^m}\log(6N(\varepsilon))
\end{align}
and 
\begin{align}
N(\varepsilon)=\lceil {2(m+1)(C\kappa(\varepsilon)+1)}/{\varepsilon}\rceil, 
\end{align}
such that for each $\nu\in \colU(n,m,C,\kappa)$, there is  a $\Psi\in \colR(n,m,C,\kappa,\varepsilon)$ satisfying 
\begin{align}
\operatorname{Lip}(\Psi|_{[0,1]}) \leq m(6(N(\varepsilon)+1))^{m+1}
\end{align}
and
\begin{align}\label{eq:bounderrornuCA}
W_1\Big( \nu,\Psi\#(\mathcal{L}^{(1)}|_{[0,1]})\Big)\leq \varepsilon. 
\end{align}
Moreover, the ReLU neural networks  $\Psi\in \colR(n,m,C,\kappa,\varepsilon)$ have the same architecture and satisfy \cref{eq:NN3xyz}--\cref{eq:NN6xyz} with $N=N(\varepsilon)$. 
\end{cor}
\begin{proof}
Arbitrarily fix $\varepsilon \in (0,1]$ and set 
\begin{align}\label{eq:widehatE2}
\widehat \setE=\bigcup_{i=1}^{\kappa(\varepsilon)} f_i(\setA_i) 
\end{align}
and 
\begin{align}\label{eq:lambdaelldef}
\widehat \lambda=2 \kappa(\varepsilon) \max\big\{ \diam(\setE), \diam(\setA_1) \operatorname{Lip}(f_1), \dots,\diam(\setA_{\kappa(\varepsilon)})  \operatorname{Lip}(f_{\kappa(\varepsilon)})\big\}. 
\end{align}
We next apply \Cref{cor:rec3} with $N=N(\varepsilon)$. This yields 
\begin{align}
W_1\Big( \nu,\Psi\#(\mathcal{L}^{(1)}|_{[0,1]})\Big)&\leq \frac{(m+1)(\widehat \lambda+ 1)}{N(\varepsilon)} + (1-\nu(\widehat\setE)) \diam(\setE) \label{eq:bounderrornuC222}\\ 
&\leq \frac{(m+1)(C\kappa(\varepsilon)+ 1)}{N(\varepsilon)}+\frac{\varepsilon}{2}\label{eq:bounderrornuC222a}\\
&\leq \varepsilon, \label{eq:useagainconditions}
\end{align}
where  \cref{eq:bounderrornuC222} follows from  \cref{eq:bounderrornuC}  and in \cref{eq:bounderrornuC222a} we applied  \cref{item:nucount5},  \cref{item:nucount1}, and  \cref{item:nucount3} in \Cref{dfn:classGC}. 
It remains to check  that \cref{eq:condfacountrec2} is satisfied. 
\begin{align}
\max_{k=1,\dots,\kappa(\varepsilon)} \lVert f_k\rVert_{L_\infty(\setA_k)}+\widehat \lambda 
&\leq 2C\kappa(\varepsilon)(1+1/\varepsilon)\label{eq:boundagain}\\
&\leq N(\varepsilon), 
\end{align}
where \cref{eq:boundagain} follows from  \cref{item:nucount5},  \cref{item:nucount2}, and  \cref{item:nucount3} in \Cref{dfn:classGC}. 
We can therefore set $\colR(n,m,C,\kappa,\varepsilon)=\colJ^{(m,n)}_{N(\varepsilon)}$ with $\colJ^{(m,n)}_{N(\varepsilon)}$ as in \Cref{cor:rec3} for $N=N(\varepsilon)$. 
\end{proof}
The expression for $b(\varepsilon)$ in \cref{eq:upperbvarepsilon} depends on the decay property of the countably 
$m$-rectifiable measure $\gamma$ via $\kappa(\varepsilon)$. We next discuss the situations where $\gamma$  decays exponentially and polynomially. 
 Fix $\alpha,\beta\in (0,\infty)$ and set, for every $n\in\naturals$,  
\begin{align}
a_n=2\diam(\setE)\Bigg(1-\nu\Big(\bigcup_{k=1}^{n} f_k(\setA_k)\Big)\Bigg). 
\end{align}
\begin{enumerate}
\renewcommand{\theenumi}{(\roman{enumi})}
\renewcommand{\labelenumi}{(\roman{enumi})}
\item  Suppose that $a_n=\alpha e^{-\beta n}$, i.e., the measure decays exponentially. 
Then,  $\kappa(\varepsilon)=\lceil \log(\alpha/\varepsilon)/\beta\rceil$ satisfies the decay condition \cref{item:nucount1} in \Cref{dfn:classGC}, which results in 
\begin{align}\label{item:exp} 
b(\varepsilon)=\mathcal{O}(( 1/\varepsilon)^m \log^{m+1}( 1/\varepsilon)). 
\end{align}
In particular, we have 
\begin{align}\label{eq:qualirec22}
\inf\{s\in (0,\infty): \lim_{\varepsilon\to 0}\varepsilon^s b(\varepsilon)<\infty\}  =m.
\end{align}
Interestingly, we get the same  result as for $m$-rectifiable measures in \cref{eq:qualirec}. 
In particular, the result in \cref{eq:qualirec22} is metric entropy optimal: Set $\delta=C/2$ and let $\setQ$ be 
the cube of side-length $\delta$ centered at the origin in $\reals^m$, let  $f\colon\setQ\to \reals^n$ be defined according to $f(x_1,\dots, x_m)=\tp{(x_1,\dots,x_m,0,\dots,0)}$, and set $\setE=f(\setQ)$. Set $f_k=f$ and $\setA_k=\setQ$ for all $k\in\naturals$. Then,   $\setP(\setE)\subseteq \colU(n,m,C,\kappa)$. We therefore have 
\begin{align}
&\inf\{s\in(0,\infty) :\liminf_{\varepsilon\to 0 }H_\varepsilon(\colU(n,m,C,\kappa),W_1)\varepsilon^s<\infty\}\\ 
&\geq 
\inf\{s\in(0,\infty) :\liminf_{\varepsilon\to 0 }H_\varepsilon(\setP(\setE),W_1)\varepsilon^s<\infty\} \label{eq:useconverse0}\\
&\geq m,\label{eq:useconverse}
\end{align}
where \cref{eq:useconverse0} follows from monotonicity of  
Kolmogorov $\varepsilon$-entropy and in  \cref{eq:useconverse} we applied \Cref{prp:converse}. 
\item Suppose that $a_n=\alpha/n^\beta$, i.e., the measure decays polynomially  on the individual components.
Then,  $\kappa(\varepsilon)= \lceil(\alpha/\varepsilon)^{1/\beta}\rceil$ satisfies the decay condition \cref{item:nucount1} in \Cref{dfn:classGC}, which results in 
\begin{align}
b(\varepsilon)=\mathcal{O}((1/\varepsilon)^{m(1+1/\beta)} \log( 1/\varepsilon)). 
\end{align}
In particular, we have 
\begin{align}\label{eq:qualirec222}
\inf\{s\in (0,\infty): \lim_{\varepsilon\to 0}\varepsilon^s b(\varepsilon)<\infty\}  =m(1+1/\beta)
\end{align} 
\end{enumerate}
In \cref{eq:qualirec222}, the polynomial decay of the measure results in an increase of the dimension by the factor  $\beta+1$. We hasten to add that $2^{b(\varepsilon)}$ is an upper bound on the number of ReLU neural networks (see \cref{eq:finalcor3bitsC}). An interesting further research direction would be to figure out whether there exists a matching lower bound in the situation where 
the countable $m$-rectifiable measures  decays polynomially.

\section{Approximations by Uniform Mixtures and Space-Filling Results}\label{sec:improve}
This section collects approximation results  and  space-filling properties  of uniform mixtures, which improve and simplify corresponding  results  in \cite{peebbo21}. 
We start with the definition of (quantized)  uniform mixtures.

\begin{dfn}\label{dfn:histo}
Let $K\in\naturals$. For  $\ell=1,\dots,K$, set 
\begin{align}
\setI_\ell=
\begin{cases} 
[(\ell-1)/K, \ell/K)& \text{if $\ell<K$}\\
[(\ell-1)/K, \ell/K]& \text{if $\ell=K$}.
\end{cases}
\end{align}
Further, let
\begin{align}\label{eq:setAd}
\setA_d=\{k\in \naturals^d:  k_i \leq K \ \text{for $i=1,\dots, d$}\}
\end{align}
and, for every $k\in \setA_d$, set $\setJ_k = \setI_{k_1}\times\setI_{k_2}\times \dots\times  \setI_{k_d}$
and designate $\lambda_k=\colL^{(d)}|_{\setJ_k}$.  We refer to  a Borel measure $\kappa$ on $\reals^d$ with $\operatorname{spt}(\mu)\subseteq [0,1]^d$ uniform mixture of resolution $K$ if it can be written as 
\begin{align}\label{eq:unifmix}
\kappa=K^d \sum_{k\in\setA_d}  w_k \lambda_k 
\end{align}
with $w_k\in [0,1]$ for all $k\in\setA_d$ and $\sum_{k\in\setA_d}w_k=1$. For $N\in\naturals$, the uniform mixture of resolution $K$ in \cref{eq:unifmix} is called $(1/N)$-quantized if $w_k\in\delta\naturals$ for all $k\in\setA_d$. 
\end{dfn}
We first  consider the approximation of   Borel probability measures supported on $[0,1]^d$ by uniform mixtures.  

\begin{lem}\label{lem:kappa2}
Let $\reals^d$ be equipped with  $\lVert \,\cdot\,\rVert_\infty$ 
and consider a Borel probability measure $\mu$  on $\reals^d$ with $\operatorname{spt}(\mu)\subseteq[0,1]^d$. Further, let   $K\in\naturals$ and  define the uniform mixture $\tilde \mu$ of resolution $K$ according to 
\begin{align}
\tilde \mu = K^d\sum_{k\in\setA_d} w_k \lambda_k
\end{align}
with $w_k=\mu(\setJ_k)$ for all $k\in\setA_d$. 
Then, we have 
\begin{align}\label{eq:toshowhisto2}
W_1(\mu,\tilde \mu) \leq  1/K.    
\end{align}
\end{lem}
\begin{proof}
Set $\setB=\{k\in\setA_d:w_k>0\}$. Then, we can write 
\begin{align}
 \mu = \sum_{k\in\setA_d\cap\setB} w_k (\mu|_{\setJ_{k}}/w_k)
\end{align}
and 
\begin{align}
\tilde \mu = \sum_{k\in\setA_d\cap\setB} w_k (K^d\lambda_k)
\end{align}
with $\sum_{k\in\setB} w_k =1$. 
We thus have  
\begin{align}
W_1( \mu,\tilde \mu)
&\leq \sum_{k\in\setB} w_k W_1\big( \mu|_{\setJ_{k}}/w_k,K^d\lambda_k\big)\label{eq:uselimitW1}\\
&\leq \sum_{k\in\setB} w_k \int \rho \,\mathrm d ((\mu|_{\setJ_{k}}/w_k) \times (K^d\lambda_k))\label{eq:uselimitW1a}\\
&\leq  1/K\, \sum_{k\in\setB} w_k\label{eq:uselimitW1b} \\
&=1/K, \label{eq:uselimitW2}
\end{align} 
where \cref{eq:uselimitW1} we applied  \Cref{eq:lemWrest}, in  \cref{eq:uselimitW1a} we set $\rho(x,y)=\lVert x-y\rVert_\infty$, and \cref{eq:uselimitW1b} follows from $\diam(\setJ_{k})\leq 1/K$ for all $k\in\setA_d$. 
\end{proof}

We next state an approximation property of uniform mixtures in terms of $(1/N)$-quantized uniform mixtures. 
\begin{lem}\label{lem:kappa}
Let $\reals^d$ be equipped with  $\lVert \,\cdot\,\rVert_\infty$. Further,  
let $\kappa=K^d \sum_{k\in\setA_d}  w_k \lambda_k $ be a uniform mixture of resolution $K$ and let $N\in\naturals$ with $N\geq K^d$. Then,  there exists  
 a $(1/N)$-quantized uniform mixture $\hat \kappa$ of resolution $K$  satisfying 
\begin{align}\label{eq:W1hisnew}
W_1(\hat \kappa, \kappa)&\leq 4 (K^d-1)/N.
\end{align}
\end{lem}
\begin{proof}
\Cref{lem:mhatm} implies that we can find  a set $\{\hat w_k\}_{k\in \setA_d}$ of $(1/N)$-quantized weights satisfying   $\sum_{k\in\setA_d} \hat w_k=1$ and 
\begin{align}\label{eq:boundmhatm}
\sum_{k\in\setA_d}\abs{w_k-\hat w_k} \leq 4(K^d-1)/N. 
\end{align} 
Now, set 
\begin{align}
\hat \kappa=K^d \sum_{k\in\setA_d}  \hat w_k \lambda_k 
\end{align}
and  note that 
\begin{align}\label{eq:W1his}
W_1(\hat \kappa, \kappa)&= \sup_{\substack{\psi\colon[0,1]^d\to\reals\\ \operatorname{Lip}(\psi)\leq 1}}
\Bigg\{\int_{[0,1]^d} \psi \,\mathrm d \hat \kappa - \int_{[0,1]^d}  \psi \,\mathrm d \kappa\Bigg\} 
\end{align}
owing to \Cref{thm:dualW1} and the fact that $\operatorname{spt}(\kappa)\subseteq[0,1]^d$ and $\operatorname{spt}(\hat\kappa)\subseteq[0,1]^d$. Since 
\begin{align}
\int \psi \,\mathrm d \hat \kappa - \int \psi\,\mathrm d \kappa=\int (\psi -\psi(0)) \,\mathrm d \hat \kappa - \int (\psi -\psi(0)) \,\mathrm d \kappa,  
\end{align} 
we can restrict the supremum  in \cref{eq:W1his} to mappings $\psi$ satisfying $\psi(0)=0$  and $\operatorname{Lip}(\psi)\leq 1$. Now, arbitrarily fix $\psi\colon [0,1]^d \to \reals$ satisfying $\psi(0)=0$  and $\operatorname{Lip}(\psi)\leq 1$.  Then, we have 
\begin{align}
\abs{\int_{[0,1]^d} \psi \,\mathrm d \hat \kappa - \int_{[0,1]^d} \psi \,\mathrm d \kappa}
&\leq K^d  \sum_{k\in\setA_d} \abs{\hat w_k -w_k} \int_{\setJ_k} \abs{\psi} \,\mathrm d \lambda_k \label{eq:useinW1}\\
&\leq \sum_{k\in\setA_d}  \abs{\hat w_k -w_k}\label{eq:usepsi}\\
&\leq 4 (K^d-1)/N ,\label{eq:usepsi2}
\end{align}
where in \cref{eq:usepsi} we used $\abs{\psi(x)}=\abs{\psi(x)-\psi(0)}\leq \operatorname{Lip}(\psi) \lVert x \rVert_\infty\leq 1$ for all $x\in\setI$ and  \cref{eq:usepsi2} follows from \cref{eq:boundmhatm}. 
As $\psi$ was assumed to be arbitrary, \cref{eq:W1hisnew} follows. 
\end{proof}

We are now in a position to state the approximation result for Borel probability measures supported on $[0,1]^d$ by $(1/N)$-quantized  uniform mixtures, which improves the upper bound derived in \cite[Theorem VII.2]{peebbo21}.

\begin{thm}\label{thm:approxhisto}
Let $\reals^d$ be equipped with  $\lVert \,\cdot\,\rVert_\infty$ 
and suppose that $\mu$ is a Borel probability measure on $\reals^d$ with $\operatorname{spt}(\mu)\subseteq[0,1]^d$. Further, let   $K,N\in\naturals$  with $N\geq K^{d}$. Then, there exists  a $(1/N)$-quantized  uniform mixture $\tilde \mu$ of resolution $K$ satisfying 
\begin{align}\label{eq:toshowhisto}
W_1(\mu,\tilde \mu) \leq 4 (K^d-1)/N  + 1/K.    
\end{align}
\end{thm}
\begin{proof}
Follows by combining  \Cref{lem:kappa,lem:kappa2} through the triangle inequality. 
\end{proof}

A byproduct of  \Cref{thm:approxhisto} is  the following metric entropy upper bound on the set of Borel probability measures on  $\reals^d$ that are supported on $[0,1]^d$.  
\begin{cor}
Let $\reals^d$ be equipped with  $\lVert \,\cdot\,\rVert_\infty$.  Further,  
let $\colF$ denote the space of Borel probability  measures $\mu$ on $\reals^d$ with $\operatorname{spt}(\mu)\subseteq[0,1]^d$ and consider the metric space $(\colF, W_1)$. Then, we have 
\begin{align}
H_\varepsilon(\colF,W_1) \leq  ((2+\varepsilon)/\varepsilon)^d\log(12(2+\varepsilon)/\varepsilon). \label{eq:entropymu}
\end{align} 
\end{cor}
\begin{proof}
Arbitrarily fix  $\varepsilon\in (0,\infty)$ and set $K=K(\varepsilon)=\min\{\ell\in\naturals: 2/\ell\leq \varepsilon \}$ so that 
\begin{align}\label{eq:sandwich}
 2/K\leq \varepsilon <2/(K-1). 
\end{align} 
Next, set 
\begin{align}
\setA_d=\{k\in \naturals^d:  k_i \leq K \ \text{for $i=1,\dots, d$}\}.  
\end{align}
and  note that  the set of $1/(4K^{d+1})$-quantized weights $\{w_k\}_{k\in\setA_d}$ summing  to one satisfies 
\begin{align}
\abs {\{w_k\}_{k\in\setA_d}} 
&\leq \binom{4K^{d+1}-1}{K^d-1}\label{eq:usesterlingA}\\
&\leq \binom{4K^{d+1}}{K^d}\\
&\leq ( 12K)^{K^d}\label{eq:usesterlingB}\\
&<(12(2+\varepsilon)/\varepsilon)^{(2+\varepsilon)^d/\varepsilon^d}, \label{eq:usesterlingC}
\end{align}
where \cref{eq:usesterlingB} follows from $\binom{n}{k}\leq (en/k)^k$ and in \cref{eq:usesterlingC} we applied \cref{eq:sandwich}. 
Now, \cref{eq:usesterlingA}--\cref{eq:usesterlingC} implies  that the number of $1/(4K^{d+1})$-quatized uniform mixtures is upper-bounded by $(12(2+\varepsilon)/\varepsilon)^{(2+\varepsilon)^d/\varepsilon^d}$. To establish \cref{eq:entropymu}, it remains to prove that they form an $\varepsilon$-net. To this end, note that 
\Cref{thm:approxhisto} with $N=4K^{d+1}$ yields, for every Borel probability measure $\mu$ with $\operatorname{spt}(\mu)\subseteq[0,1]^d$, a 
$1/(4K^{d+1})$-quatized uniform mixture  $\tilde \mu$ of resolution $K$ satisfying  
\begin{align}\label{eq:W1entropy}
W_1(\mu,\tilde \mu)\leq 2/K\leq \varepsilon,
\end{align} 
where the second inequality is by \cref{eq:sandwich}.  
\end{proof}
The upper bound \cref{eq:entropymu} implies 
 \begin{align}\label{eq:scalingdmytro2d}
\inf\{s\in (0,\infty): \limsup_{\varepsilon\to 0}\varepsilon^s H_\varepsilon(\colF,W_1)<\infty\} =d, 
 \end{align} 
which is also obtained in the proof of \cite[Propositions 7.4]{kl12}. As noted before, this result is  metric entropy optimal  with matching lower bound 
\cite[Propositions 3.2 and 7.3]{kl12}.  


We next establish that $(1/N)$-quantized uniform mixtures  can be approximated  arbitrarily well as push-forwards of 
$\colL^{(1)}|_{[0,1]}$ using space-filling curves. What is more, these space-filling curves can be realized by ReLU neural networks. This is the main idea behind the first step of the  proof technique  depicted in  \Cref{fig:illustration}. 
 A result  in this direction was already established in  \cite[Theorem V.6]{peebbo21}, but we here present a quantitative version of this result. Concretely, we explicitly construct the   ReLU neural network realizing the push-forward and state explicit bounds on the network parameters (depth, connectivity, width,  weight magnitude, and Lipschitz constant).
The explicit form of the network can give important insights  for the design of  learning algorithms in  deep generative models 
 \cite{arbo17,gopomixuwaozcobe14,kiwe14}.  
 We also use a much more  direct and structured proof that does not ignore boundary terms (see \cite[Remark II.5]{peebbo21}) that  potentially could  contribute with positive measure to the push-forward of 
$\colL^{(1)}|_{[0,1]}$.   We first present the intuition behind our proof technique at a simple example. 

\vspace*{3truemm}

\begin{figure}[hbtp]
\begin{center}
\begin{tikzpicture}[scale=4.5]
\draw (0,0) rectangle (1,1);
\draw (0,0) rectangle (3/4,2/3);
\draw (0,2/3) rectangle (3/4,1);
\draw (3/4,0) rectangle (1,1/2);
\node at (0.4,0.3) {$\setK_{(1,1)}$};
\node at (0.4,0.8) {$\setK_{(1,2)}$};
\node at (0.88,0.3) {$\setK_{(2,1)}$};
\node at (0.88,0.8) {$\setK_{(2,2)}$};
\node at (0.75,-0.05) {\small{$3/4$}};
\node at (-0.1,0.65) {\small{$2/3$}};
\node at (1.07,0.5) {\small{$1/2$}};
\draw[darkgreen,thick] (0,0) -- (3/8,1);
\draw[darkgreen,thick]  (3/8,1)--(3/4,0);
\draw[darkgreen,thick]  (3/4,0)--(7/8,1);
\draw[darkgreen,thick]  (7/8,1)--(1,0);
\draw[dotted] (0,1/3) -- (3/4,1/3);
\draw[dotted] (0,5/6) -- (3/4,5/6);
\draw[dotted] (3/8,0) -- (3/8,1);
\draw[dotted] (3/4,1/4) -- (1,1/4);
\draw[dotted] (3/4,3/4) -- (1,3/4);
\draw[dotted] (7/8,0) -- (7/8,1);
\draw[darkgreen,dashed,thick] (0,0) -- (3/16,1);
\draw[darkgreen,dashed,thick] (3/16,1)--(3/8,0);
\draw[darkgreen,dashed,thick] (3/8,0)--(9/16,1);
\draw[darkgreen,dashed,thick] (9/16,1)--(3/4,0);
\draw[darkgreen,dashed,thick] (3/4,0) -- (13/16,1);
\draw[darkgreen,dashed,thick]  (13/16,1) -- (7/8,0);
\draw[darkgreen,dashed,thick]  (7/8,0) -- (15/16,1);
\draw[darkgreen,dashed,thick]  (15/16,1) -- (1,0);
\draw [->] (1.2,0.5) -- (1.4,0.5);
\node at (1.3,0.55) {$f$};
\draw (1.5,0) rectangle (2.5,1);
\draw (1.5,0.5) -- (2.5,0.5);
\draw (2,0) -- (2,1);
\node at (1.75,0.25) {$\setJ_{(1,1)}$};
\node at (1.75,.75) {$\setJ_{(1,2)}$};
\node at (2.25,0.25) {$\setJ_{(2,1)}$};
\node at (2.25,0.75) {$\setJ_{(2,2)}$};
\draw[blue,thick] (1.5,0) -- (1.5+1/6,1/2);
\draw[blue,thick] (1.5+1/6,1/2)--(1.5+1/4,1);
\draw[blue,thick] (1.5+1/4,1)--(1.5+1/3,1/2);
\draw[blue,thick] (1.5+1/3,1/2)--(1.5+1/2,0);
\draw[blue,thick] (1.5+1/2,0)--((1.5+3/4,1);
\draw[blue,thick] (1.5+3/4,1)--(1.5+1,0);
\draw[dotted] (1.5,0.25) -- (2.5,0.25);
\draw[dotted] (1.5,0.75) -- (2.5,0.75);
\draw[dotted] (1.75,0) -- (1.75,1);
\draw[dotted] (2.25,0) -- (2.25,1);
\draw[blue,thick,dashed] (1.5,0) -- (1.5+1/12,1/2);
\draw[blue,thick,dashed] (1.5+1/12,1/2)--(1.5+1/8,1);
\draw[blue,thick,dashed] (1.5+1/8,1)--(1.5+1/6,1/2);
\draw[blue,thick,dashed] (1.5+1/6,1/2)--(1.5+1/4,0);
\draw[blue,thick,dashed] (1.5+1/4,0) -- (1.5+1/12+1/4,1/2);
\draw[blue,thick,dashed] (1.5+1/12+1/4,1/2)--(1.5+1/8+1/4,1);
\draw[blue,thick,dashed] (1.5+1/8+1/4,1)--(1.5+1/6+1/4,1/2);
\draw[blue,thick,dashed] (1.5+1/6+1/4,1/2)--(1.5+1/4+1/4,0);
\draw[blue,thick,dashed] (1.5+1/2,0)--((1.5+5/8,1);
\draw[blue,thick,dashed] (1.5+5/8,1)--(1.5+3/4,0);
\draw[blue,thick,dashed] (1.5+1/2+1/4,0)--((1.5+5/8+1/4,1);
\draw[blue,thick,dashed] (1.5+5/8+1/4,1)--(1.5+3/4+1/4,0);
\end{tikzpicture} 
\end{center}
\vspace*{-5truemm}
\caption{Black solid lines indicate the partitioning of $\setI$ in rectangles $\setK_{(k_1,k_2)}$ and in squares $\setJ_{(k_1,k_2)}$ for $k_1,k_2=1,2$. Each rectangle 
$\setK_{(k_1,k_2)}$ is mapped to a square $\setJ_{(k_1,k_2)}$ under the piecewise affine mapping $f$. 
The unit interval is mapped to the green curve under  $k^{(1)}$ and to the blue curve under the mapping $\tilde f^{(1)} = f\circ k^{(1)} $.  Black dotted lines indicate  rectangles and corresponding squares for the resolution $K_2=2K=4$. The unit interval is mapped to the dashed green curve under  $k^{(2)}$ and to the dashed blue curve under the mapping $\tilde f^{(2)} = f\circ k^{(2)} $.   
\label{fig:rectangles}}
\end{figure}
\vspace*{-5truemm}

\begin{exa}Let $K=d=2$ and $N=8$, set $\setI=[0,1]^2$, and consider the $(1/8)$-quanatized uniform mixture $\mu=\sum_{k_1,k_2=1}^2w_{(k_1,k_2)}\lambda_{(k_1,k_2)}$ with $w_{(1,1)}=1/2$, $w_{(1,2)}=1/4$, and $w_{(2,1)}=w_{(2,2)}=1/8$. In a first step, we construct a piecewise affine bijection $f\colon \setI\to\setI$ that allows to write  $\mu=f\#(\colL^{(2)}|_{\setI})$. To this end, let $\setI_{1}=[0,1/2)$ and $\setI_{2}=[1/2,1)$ and consider the subsquares 
\begin{align}
\setJ_{(k_1,k_2)}= \setI_{k_1}\times \setI_{k_2}\quad\text{for $k_1,k_2=1,2$.}
\end{align}
Further,  set $\setK_{(1)}=[0,3/4)$,    $\setK_{(2)}=[3/4,1]$, 
$\setK_{(1|(1))}= [0,2/3)$,
$\setK_{(2|(1))}=  [2/3,1]$, 
$\setK_{(1|(2))}= [0,1/2)$, and 
$\setK_{(2|(2))}= [1/2,1]$ 
and consider the rectangles $\setK_{(k_1,k_2)}= \setK_{(k_1)}\times\setK_{(k_2|(k_1))}$ for $k_1,k_2=1,2$. 
The squares $\setJ_{(k_1,k_2)}$ and the rectangles $\setK_{(k_1,k_2)}$ are depicted with black solid lines in \Cref{fig:rectangles}. 
By construction, the rectangles $\setK_{(k_1,k_2)}$  are pairwise disjoint, cover $\setI$, and satisfy $\colL^{(2)}|_{\setI}(\setK_{(k_1,k_2)})=w_{(k_1,k_2)}$ for $k_1,k_2=1,2$.  It is therefore sufficient to construct a piecewise affine mapping $f$ that maps every $\setK_{(k_1,k_2)}$ to every $\setJ_{(k_1,k_2)}$.  
To this end, we set $f|_{\setK_{(k_1,k_2)}}=f^{(k_1,k_2)}$  with 
$f^{(k_1,k_2)}=(f_1^{(k_1)}, f_2^{(k_1,k_2)})$  for $k_1,k_2=1,2$, where we set 
\begin{align}
f_1^{(1)}(x_1)=2x_1/3,&& 
f_1^{(2)}(x_1)=2x_1-1, 
\end{align}
and 
\begin{align}
f_2^{(1,1)}(x_2)&=3x_2/4,\\
 f_2^{(1,2)}(x_2)&=(3x_2-1)/2, \\
f_2^{(2,1)}(x_2)&=f_2^{(2,2)}(x_2)=x_2.
\end{align}
In a second step, we combine the mapping $f$ with a space-filling curve. Specifically, let $g$ be as per \Cref{dfn:g} and set $k^{(1)}(x)=(x,g(4x/3)+g(4x -3))$ and $\tilde f^{(1)} = f\circ k^{(1)}$. The mappings $k^{(1)}$ and $\tilde f^{(1)}$ are depicted in \Cref{fig:rectangles} with solid green and blue curves, respectively. The measure $\mu$ is now approximated by $\tilde f^{(1)} \#\colL^{(1)}|_{[0,1]}$. Note that, by construction, we have 
\begin{align}
\tilde f^{(1)} \#(\colL^{(1)}|_{[0,1]})(\setJ_{(k_1,k_2)})
&=  k^{(1)}\#(\colL^{(1)}|_{[0,1]})(\setK_{(k_1,k_2)})\\
&= w_{(k_1,k_2)}\quad\text{for $k_1,k_2=1,2$}
\end{align}
so that \Cref{lem:kappa2} with $K=2$ implies $W(\mu, f^{(1)} \#\colL^{(1)}|_{[0,1]})\leq 1/2$. 
This approximation error can be reduced by increasing the resolution for  $\mu$. 
Concretely, let $s\in\naturals$ with $s>1$, set  $N_s=2^{2(s-1)}N$ and $K_s=2^{s}$, and let  $g_s$ be  as per \Cref{dfn:g}.  
 Now, $\mu$ is  also 
$(1/N_s)$-quantized   of resolution  $K_s$, which can be easily seen by partitioning every cube $\setJ_{(k_1,k_2)}$  into $2^{2(s-1)}$ subcubes of equal size and dividing all weights of $\mu$ by  $2^{2(s-1)}$. 
The same argumentation  as above applied to $\mu$ with $k^{(s)}(x)=(x,g_s(4x/3)+g_s(4x -3))$ and $\tilde f^{(s)} =f\circ  k^{(s)}$ therefore yields $W(\mu, \tilde f^{(s)} \#\colL^{(1)}|_{[0,1]})\leq 1/2^s$ by \Cref{lem:kappa2} with $K=K_s$. 
For $s=2$, the mappings $k^{(2)}$ and $\tilde f^{(2)}$ are depicted in \Cref{fig:rectangles} with dashed green and blue curves, respectively. 

It is worth noting that $\tilde f^{(s)} =f\circ  k^{(s)}$ is not the composition of two  ReLU neural networks.  
The reason for this stems form  the simple fact  that   $f$ has discontinuities at $x_1=3/4$  and, therefore, cannot be realized as a ReLU neural network. However, a direct calculation shows that  we can write  $\tilde f^{(s)}$ in the form 
\begin{align}
\tilde f^{(s)}=\tp{(\tilde f_1,\tilde f_2^{(s)})}\label{eq:tildefj0E}
\end{align}
with 
\begin{align}
\tilde f_1=\sum_{k_1=1}^2f_1^{(k_1)}\ind{\setK_{k_1}}\label{eq:tildefj1E}
\end{align}
and
\begin{align}
\tilde f_2^{(s)}= \sum_{k_1,k_2=1}^2 
(f_2^{(k_2,k_1)} \ind{\setK_{k_2|(k_1)}}) \circ g_s\circ L_{k_1} \circ \tilde f_1, \label{eq:tildefj2E}
\end{align}
where we set  $L_{\ell}(x)= 2x -\ell +1$  for all $x\in\reals$ and  $\ell=1,2$. It will be seen  that all mappings in \cref{eq:tildefj1E}--\cref{eq:tildefj2E} can be realized by ReLU neural networks.  
\end{exa}

We now use the intuition  in the  above example to prove  our main result on the 
approximation of  $(1/N)$-quantized uniform mixtures  in terms of  push-forwards of 
$\colL^{(1)}|_{[0,1]}$ using space-filling curves.

\begin{thm}\label{thm:quantunif}
Let $\reals^d$ be equipped with  $\lVert \,\cdot\,\rVert_\infty$. Further, let  $K\in\naturals\setminus\{1\}$, set $\setI=[0,1]^d$, and   consider the following setup: 
\begin{enumerate}
\renewcommand{\theenumi}{\arabic{enumi})}
\renewcommand{\labelenumi}{\arabic{enumi})}
\item
 For  $\ell=1,\dots,K$, set 
\begin{align}
\setI_\ell=
\begin{cases} 
[(\ell-1)/K, \ell/K)& \text{if $\ell<K$}\\
[(\ell-1)/K, \ell/K]& \text{if $\ell=K$}.
\end{cases}
\end{align}Further, for $j=1,\dots,d$, let
\begin{align}
\setA_j=\{k\in \naturals^j:  \lVert k \rVert_\infty \leq K \}
\end{align}
and, for every $k\in \setA_d$, set $\setJ_k = \setI_{k_1}\times\setI_{k_2}\times \dots\times  \setI_{k_d}$
and  $\lambda_k=\colL^{(d)}|_{\setJ_k}$.  Finally, let $N\in\naturals$ with $N\geq \min\{K^d,4\}$ and let $\mu$ be a   $(1/N)$-quantized uniform mixture of resolution $K$ according to 
\begin{align}
\mu=K^d \sum_{k\in\setA_d}  w_k \lambda_k 
\end{align}
with $w_k\in\naturals/N$ for all $k\in\setA_d$ and $\sum_{k\in\setA_d}w_k=1$. 
\item 
For  $j=1,\dots,d-1$ and 
$(k_1,\dots,k_j)\in \setA_j$, define the ``marginal weights" as 
\begin{align}
w_{(k_1,\dots,k_j)}=\sum_{k_{j+1}=1}^K \sum_{k_{j+2}=1}^K\dots\sum_{k_{d}=1}^K w_{k}.
\end{align}
For $j=2,\dots,d$ and 
$(k_1,\dots,k_j)\in \setA_j$,
define the  ``conditional weights" according to 
\begin{align}
w_{k_j| (k_1,\dots, k_{j-1})} = \frac{w_{(k_1,\dots,k_j)}}{w_{(k_1,\dots,k_{j-1})}}.   
\end{align}
\item 
Set  $b_0=0$ and, for $k_1=1,\dots,K$, let $b_{k_1}=\sum_{i=1}^{k_1}w_{(i)}$ and designate  the sets 
\begin{align}
\setK_{k_1}=
\begin{cases} 
[b_{k_1-1}, b_{k_1})& \text{if $k_1<K$}\\
[b_{k_1-1}, b_{k_1}]& \text{if $k_1=K$}.  
\end{cases}
\end{align}
For every  $j=2,\dots,d$ 
and $(k_1,\dots,k_j)\in \setA_j$,  set $b_{0|(k_{1},\dots,k_{j-1})}=0$,
\begin{align}
b_{k_j|(k_{1},\dots,k_{j-1})}=\sum_{i=1}^{k_j} w_{i| (k_1,\dots, k_{j-1})},
\end{align}
and define the sets 
\begin{align}
\setK_{k_j|(k_1,\dots, k_{j-1})}=
\begin{cases} 
[b_{k_j-1|(k_1,\dots, k_{j-1})}, b_{k_j|(k_1,\dots, k_{j-1})})& \text{if $k_j<K$}\\
[b_{k_j-1|(k_1,\dots, k_{j-1})}, b_{k_j|(k_1,\dots, k_{j-1})}]& \text{if $k_j=K$}.  
\end{cases}
\end{align}
\item \label{itemf}
Let  $\setK_k=\setK_{k_1}\times\setK_{k_2|k_1}\times \dots\times \setK_{k_d|(k_{1},\dots,k_{d-1})}$
and define the mapping $f\colon \setI \to \setI$ according to $f=\sum_{k\in\setA} f^{(k)}\ind{\setK_k}$ with 
\begin{align}\label{eq:f_kdef}
f^{(k)}(x):=
\begin{pmatrix}
f^{(k_1)}_1(x_1)\\
f^{(k_1,k_2)}_2(x_2)\\
\vdots \\
f^{(k_1,k_2,\dots , k_d)}_d(x_d)
\end{pmatrix}
:=
\begin{pmatrix}
\frac{x_1-b_{k_1}}{Kw_{(k_1)}} +\frac{k_1}{K}\\
\frac{x_2-b_{k_2|(k_1)}}{Kw_{k_2|(k_1)}} +\frac{k_2}{K}\\
\vdots\\
\frac{x_d-b_{k_d|(k_{1},\dots,k_{d-1})}}{Kw_{k_d|(k_{1},\dots,k_{d-1})}} +\frac{k_d}{K}
\end{pmatrix}
\end{align}
for all $x\in \setK_k$ and $k\in\setA_d$. 
\item  Fix $s\in\naturals$ and define $\tilde f^{(s)}\colon \reals \to \reals^d$ according to $\tilde f^{(s)}=\tp{(\tilde f_1,\tilde f_2^{(s)},\dots,\tilde f_d^{(s)})}$
with 
\begin{align}\label{eq:tildef1A}
\tilde f_1(x)=x\quad\text{ for all $x\in (-\infty,0)$,} 
\end{align}
\begin{align}\label{eq:tildef1}
\tilde  f_1(x) = \sum_{k_1\in \setA_1} f^{(k_1)}_1 \ind{\setK_{k_1}}(x) \quad \text{for all $x\in[0,1]$,}
\end{align}
and 
\begin{align}\label{eq:tildef1B}
\tilde  f_1(x) = 1/(Kw_{(K)})(x-1)+ 1\quad\text{ for all $x\in (1,\infty)$,} 
\end{align}
and, for $j=2,\dots,d$,  
\begin{align}
\tilde  f_j^{(s)} 
&= \sum_{(k_1,\dots,k_{j-1})\in\setA_{j-1}} \label{eq:tildefj1}
\Big(\hat f_{j}^{(k_1,\dots, k_{j-1})}\circ g_s\circ L_{k_{j-1}}\Big)\\ 
&\ \    \phantom{\sum_{(k_1,\dots,k_{j-1})\in\setA_{j-1}}}  
\circ  \Big(\hat f_{j-1}^{(k_1,\dots, k_{j-2})}\circ g_s\circ L_{k_{j-2}} \Big) \\
&\ \  \ \   \phantom{\sum_{(k_1,\dots,k_{j-1})\in\setA_{j-1}}} \vdots \\
&\ \ \phantom{\sum_{(k_1,\dots,k_{j-1})\in\setA_{j-1}}}  
\circ \Big (\hat f_2^{(k_1)} \circ g_s\circ L_{k_1}\Big)\circ \tilde f_1, \label{eq:tildefj2}
\end{align}
where $g_s$ is as per \Cref{dfn:g}, $L_{\ell}(x)= Kx -\ell +1$ for all $x\in\reals$ and $\ell\in\naturals$, and, for  $j=2,\dots,d$ and $(k_1,\dots, k_{j-1})\in\setA_{j-1}$, we set 
\begin{align}\label{eq:hatfjA}
\hat f_j^{(k_1,\dots, k_{j-1})}(x)=x\quad\text{ for all $x\in (-\infty,0)$,} 
\end{align}

\begin{align}\label{eq:hatfj}
\hat f_j^{(k_1,\dots, k_{j-1})}(x)= \sum_{k_j=1}^K  f_j^{(k_1,\dots, k_{j})}\ind{\setK_{k_j|(k_1,\dots. k_{j-1})}}(x) 
\quad\text{for all $x\in [0,1]$,}
\end{align}
and 
\begin{align}\label{eq:hatfjB}
\hat f_j^{(k_1,\dots, k_{j-1})}(x) = 1/(Kw_{k_j|(k_{1},\dots,k_{j-1})})(x-1)+ 1\quad\text{ for all $x\in (1,\infty)$.} 
\end{align}
\end{enumerate}
Then, the following statements hold. 
\begin{enumerate} 
\renewcommand{\theenumi}{(\roman{enumi})}
\renewcommand{\labelenumi}{(\roman{enumi})}
\item \label{item1space}
The mapping $f$  in \cref{itemf} is a bijection  with $f(\setK_k)=\setJ_k$ for all $k\in\setA_d$ 
satisfying    
$\mu=f\#\colL^{(d)}|_{\setI}$. 
\item \label{item2space}For  $j=1.\dots,d$, 
define the cubes $\setC_{r_j}^{k_j}$ for $k_j=1,\dots,K$  and  $r_j=1,\dots,2^{s-1}$  as follows. 
Set 
\begin{align}
\setC_{r_j}^{k_j} = \Big[\frac{k_j-1}{K}+\frac{r_j-1}{K2^{s-1}}, \ \frac{k_j-1}{K}+\frac{r_j}{K2^{s-1}}\Big) 
\end{align}
if $k_j<K$ or $k_j=K$ and $r_j< 2^{s-1}$ and 
\begin{align}
\setC_{2^{s-1}}^{K} = \Big[1-\frac{1}{K2^{s-1}}, \ 1\Big]. 
\end{align}
 Then, $\tilde f^{(s)}([0,1])\subseteq \setI$ and  
\begin{align}\label{eq:propB2}
(\tilde f^{(s)} \# (\colL^{(1)}|_{[0,1]}))\big(\setC_{r_1}^{k_1}\times\setC_{r_2}^{k_2}\dots \times \setC_{r_d}^{k_d} \big)= \mu\big(\setC_{r_1}^{k_1}\times\setC_{r_2}^{k_2}\dots \times \setC_{r_d}^{k_d} \big)
\end{align}
for all $(k_1,\dots, k_d)\in\setA_d$ and $r_1,r_2,\dots,r_d \in\{1,\dots,2^{s-1}\}$. 
\item \label{item3space}We have 
\begin{align}\label{eq:propB3}
W_1(\tilde f^{(s)} \# (\colL^{(1)}|_{[0,1]}),\mu)\leq \frac{1}{K2^{s-1}}\end{align} 
with
\begin{align}\label{eq:propB4}
\operatorname{Lip}( \tilde f^{(s)}) \leq \frac{2^{s(d-1)}N}{K}. 
\end{align}

\item \label{item4space} There exists   a ReLU neural network $\Phi_{N,K}^{(s)}\in\setN_{1,d}$ with  
\begin{align}
 \setL(\Phi_{N,K}^{(s)})&=3+(d-1)(5+s) \label{eq:ReLUpush1}\\
 \setM(\Phi_{N,K}^{(s)}) & \leq 4d(2K+3s+4) K^d/(K-1)\label{eq:ReLUpush2}\\
 \setW(\Phi_{N,K}^{(s)}) &\leq  4K^d  \label{eq:ReLUpush3}\\
 \setK( \Phi_{N,K}^{(s)}) &\subseteq \setD_{N,K} \label{eq:ReLUpush4}\\
 \setB(\Phi_{N,K}^{(s)}) &\leq N, \label{eq:ReLUpush5}
 \end{align}
 where we set 
\begin{align}
\setD_{N,K}= \{a/b: a\in\mathbb{Z}, b\in\naturals,  \abs{a}\leq N, \ \text{and}\ b\leq NK\},   
\end{align}
satisfying 
\begin{align}\label{eq:corpush}
\Phi_{N,K}^{(s)} = \tilde f^{(s)}
\end{align}
and
\begin{align}\label{eq:propB4psi}
\operatorname{Lip}( \Phi_{N,K}^{(s)}) \leq \frac{2^{s(d-1)}N}{K}.   
\end{align}
The architecture  of $\Phi_{N,K}^{(s)}$ only depends on $s,d,$ and $K$ but not on $\mu$ and $N$. 

\end{enumerate}

\end{thm}
\begin{proof}
We start by  proving \Cref{item1space}. 
By construction, we have 
\begin{align}
f^{(k_1)}_1(\setK_{k_1})&=\setI_{k_1}\\
f^{(k_1,k_2)}_2(\setK_{k_2|k_1})&=\setI_{k_2}\\
&\vdots\\
f^{(k_1,k_2,\dots , k_d)}_d(\setK_{k_d|(k_{d-1},\dots,k_1)})&=\setI_{k_d},
\end{align}
which implies $f(\setK_k)=\setJ_k$ for all $k\in\setA_d$. Since $\setI=\bigcup_{k\in\setA_d} \setJ_k$, this implies in turn that 
$f$ is onto.  Now, $f^{(k)}$  is injective on $\setK_k$ for all $k\in\setA_d$. Moreover, since the sets 
$\setJ_k$ are  pairwise disjoint and $f(\setK_k)=\setJ_k$ for all $k\in\setA_d$,  $f$ is injective on $\bigcup_{k\in\setA_d}\setK_k$.   
Finally, since 
\begin{align}
\bigcup_{k_1=1}^K \setK_{k_1}=[0,b_K]=[0,1]
\end{align}
and
\begin{align}
\bigcup_{k_j=1}^K \setK_{k_j|(k_{j-1},\dots,k_1)}
&=[0,b_{K|(k_1,\dots,k_{j-1})}]\\
&=[0,1]\quad\text{for $j=2,\dots,d$,}
\end{align}
we conclude that $\bigcup_{k\in\setA_d}\setK_k=\setI$ so that $f$ is injective on $\setI$. Summarizing, 
$f\colon \setI \to \setI$ is a bijection with $f(\setK_k)=\setJ_k$ for all $k\in\setA_d$. 
Now, fix $k\in\setA_d$ arbitrarily and suppose that $\setB_k\subseteq \setJ_k$. Then, we have 
\begin{align}
\colL^{(1)}(f_1^{-1}(\setB_k))=K w_{k_1}  \colL^{(1)}(\setB_k\cap\setI_{k_1})
\end{align}
and 
\begin{align}
\colL^{(1)}(f_j^{-1}(\setB_k))=K w_{k_j|(k_1,\dots,k_{j-1})}  \colL^{(1)}(\setB_k\cap\setI_{k_j})\quad\text{for $j=2,\dots,d$,}
\end{align}
which yields 
\begin{align}
\colL^{(d)}(f^{-1}(\setB_k))\label{eq:pull1}
&=K^d w_{k_1}\Big( \prod_{j=2}^d w_{k_j|(k_1,\dots,k_{j-1})}\Big) \lambda_k(\setB_k)\\
&= K^d w_k\lambda_k(\setB_k).\label{eq:pull2}
\end{align}
We thus have 
\begin{align}
\colL^{(d)}(f^{-1}(\setB))
&=\sum_{k\in\setA_d} \colL^{(d)}( (f^{-1}(\setB\cap\setJ_k))\\
&= \sum_{k\in\setA_d} K^d w_k \lambda_k(\setB)\label{eq:pull3}\\
&=\mu(\setB)\quad\text{for all $\setB\subseteq\setI$,}
\end{align}
where \cref{eq:pull3} follows from \cref{eq:pull1}--\cref{eq:pull2}, which proves  $\mu=f\#(\colL^{(d)}|_{\setI})$. 
We next  prove \Cref{item2space}. To this end, set 
\begin{align}\label{eq:defB1}
\setB_{r_1}^{k_1}= \tilde f_1^{-1}( \setC_{r_1}^{k_1}) \quad\text{for $k_1=1,\dots,K$ and  $r_1=1,\dots,2^{s-1}$}
\end{align}
and define inductively, for  $(k_1,\dots,k_j)\in\setA_j$, $r_1,\dots,r_j \in\{1,\dots,2^{s-1}\}$, and $j=2,\dots, d $, 
\begin{align}
\setB_{r_1,\dots,r_j}^{k_1,\dots, k_j}=   {\tilde f^{(s)}_j}^{-1} (\setC^{k_j}_{r_j}) \cap\, \setB_{r_1,\dots,r_{j-1}}^{k_1,\dots, k_{j-1}}.
\end{align}
Suppose first that $d=1$. 
Since $f_1^{k_1}$ has slope $1/(Kw_{(k_1)})$, we have 
\begin{align}
\colL^{(1)}(\setB_{r_1}^{k_1})= Kw_{(k_1)} \colL^{(1)}(\setC_{r_1}^{k_1})  =  w_{(k_1)}/{2^{s-1}}
\end{align}
for $k_1=1,\dots,K$ and  $r_1=1,\dots,2^{s-1}$, which establishes \cref{eq:propB2} for $d=1$. 
Next, suppose that $d>1$. We divide the proof into several steps. \\

\emph{Step 1:}\ We  prove by induction  that, for   $j=2,\dots,d$, we have   
\begin{align}
\tilde f^{(s)}_j(x) &= 
\Big(\hat f_{j}^{(k_1,\dots, k_{j-1})}\circ g_s\circ L_{k_{j-1}}\Big) 
\circ  \Big(\hat f_{j-1}^{(k_1,\dots, k_{j-2})}\circ g_s\circ L_{k_{j-2}} \Big) \label{eq:fj1}\\
& \ \ \ \circ \dots \circ  \Big (\hat f_2^{(k_1)} \circ g_s\circ L_{k_1}\Big)\circ \tilde f_1 (x)
\quad\text{for all $x\in \setB_{r_1,\dots,r_{j-1}}^{k_1,\dots, k_{j-1}}$,} \label{eq:fj2}
\end{align}
$(k_1,\dots,k_{j-1})\in\setA_{j-1}$, and $r_1,\dots,r_{j-1} \in\{1,\dots,2^{s-1}\}$. 
 Suppose first that $j=2$ and fix $k_1$, $r_1$, and $x\in\setB_{r_1}^{k_1}$ arbitrarily. Then, we have 
\begin{align}
\tilde f^{(s)}_2(x)
&=\sum_{(\ell_1)\in\setA_{1}} \hat f_{2}^{(\ell_1)}\circ g_s\circ L_{\ell_{1}}\circ \tilde f_1(x). 
\end{align}
Since  $ \setB_{r_1}^{k_1}= {\tilde f_1}^{-1}(\setC_{r_1}^{k_1})$, we have $\tilde f_1(x)\in \setC_{r_1}^{k_1}\subseteq \setI_{k_1}$ so that $(\hat f_{2}^{(\ell_1)}\circ g_s\circ L_{\ell_{1}}\circ \tilde f_1)(x)=0$ if 
$\ell_1\neq k_1$ owing to  $g_s(x)=0$ for all $x\notin (0,1)$ and $\hat f_{2}^{(\ell_1)}(0)=0$. We thus have 
\begin{align}
\tilde f^{(s)}_2(x)
&=\hat f_{2}^{(k_1)}\circ g_s\circ L_{k_{1}}\circ \tilde f_1(x) \quad\text{for all $x\in \setB_{r_1}^{k_1}$.} 
\end{align}
Now, suppose that  \cref{eq:fj1}-\cref{eq:fj2} holds for $j<d$ and fix  $(k_1,\dots,k_{j})\in\setA_j$, $r_1,\dots, r_{j}\in\{1,\dots,s^{s-1}\}$, and $x\in\setB_{r_1,\dots,r_{j}}^{k_1,\dots,k_{j}}$ arbitrarily. Suppose that 
$(\ell_1,\dots,\ell_{j-1})\in \setA_{j-1}$ with $ (\ell_1,\dots,\ell_{j-1})\neq (k_1,\dots,k_{j-1}) $. 
Then, we have 
\begin{align}
&\Big(\hat f_{j+1}^{(\ell_1,\dots, \ell_{j})}\circ g_s\circ L_{\ell_{j+1}}\Big) 
\circ  \Big(\hat f_{j}^{(\ell_1,\dots, \ell_{j-1})}\circ g_s\circ L_{\ell_{j-1}} \Big) \\
& \ \ \ \circ \dots \circ  \Big (\hat f_2^{(\ell_1)} \circ g_s\circ L_{\ell_1}\Big)\circ \tilde f^{(s)}_1 (x)\\
&=\Big(\hat f_{j+1}^{(\ell_1,\dots, \ell_{j})}\circ g_s\circ L_{\ell_{j+1}}\Big)(0)\label{eq:useindfjv}\\
&=0\quad \text{for $\ell_{j}=1,\dots, K$,}\label{eq:laststepindfjv}
\end{align}
where \cref{eq:useindfjv} follows from 
$\setB_{r_1,\dots,r_{j}}^{k_1,\dots,k_{j}}\subseteq \setB_{r_1,\dots,r_{j-1}}^{k_1,\dots,k_{j-1}}$, the fact that  \cref{eq:fj1}-\cref{eq:fj2} holds for $j$, and because all the functions in the sum in \cref{eq:tildefj1}--\cref{eq:tildefj2} are nonnegative and \cref{eq:laststepindfjv} is by $g_s(x)=0$ for all $x\notin (0,1)$ and $\hat f_{j+1}^{(\ell_1,\dots, \ell_{j})}(0)=0$. We therefore have 
\begin{align}
\tilde f^{(s)}_{j+1}(x) &= 
\sum_{\ell_j=1}^K\Big(\hat f_{j+1}^{(k_1,\dots, k_{j-1},\ell_j)}\circ g_s\circ L_{\ell_{j}}\Big) 
\circ  \Big(\hat f_{j}^{(k_1,\dots, k_{j-1})}\circ g_s\circ L_{k_{j-1}} \Big) \label{eq:laststepindfj0}\\
& \ \ \ \circ \dots \circ  \Big (\hat f_2^{(k_1)} \circ g_s\circ L_{k_1}\Big)\circ \tilde f^{(s)}_1 (x)\\
&=\sum_{\ell_j=1}^K\Big(\hat f_{j+1}^{(k_1,\dots, k_{j-1},\ell_j)}\circ g_s\circ L_{\ell_{j}}\Big) 
\circ \tilde f^{(s)}_{j}(x) \label{eq:useagainindfj}\\
&= \Big(\hat f_{j+1}^{(k_1,\dots, k_j)}\circ g_s\circ L_{k_{j}}\Big) 
\circ \tilde f^{(s)}_{j}(x),\label{eq:laststepindfj}
\end{align}
where  \cref{eq:useagainindfj} follows again from $\setB_{r_1,\dots,r_{j}}^{k_1,\dots,k_{j}}\subseteq \setB_{r_1,\dots,r_{j-1}}^{k_1,\dots,k_{j-1}}$  and the fact that \cref{eq:fj1}-\cref{eq:fj2} holds for $j$ and in   \cref{eq:laststepindfj}  we used $\tilde f^{(s)}_{j}(x)\in \setC^{k_j}_{r_j}\subseteq\setI_{k_j}$ if $x\in \setB_{r_1,\dots,r_{j}}^{k_1,\dots,k_{j}}$, $g_s(x)=0$ for all $x\notin (0,1)$, and $\hat f_{j+1}^{(k_1,\dots, k_{j-1}, \ell_j)}(0)=0$ for $\ell_j=1,\dots,K$. Plugging the expression  for $\tilde f^{(s)}_{j}$ from 
\cref{eq:fj1}-\cref{eq:fj2} into \cref{eq:laststepindfj} establishes \cref{eq:fj1}-\cref{eq:fj2}  for $j+1$. \\

\emph{Step 2:}\ We  prove by induction that, for  $j=2,\dots,d$, we have   
\begin{align}
\tilde f^{(s)}_j(x) &= 
\Big(\hat f_{j}^{(k_1,\dots, k_{j-1})}\circ H_{r_{j-1}}\circ L_{k_{j-1}}\Big) 
\circ  \Big(\hat f_{j-1}^{(k_1,\dots, k_{j-2})}\circ H_{r_{j-2}}\circ L_{k_{j-2}} \Big) \label{eq:fj1a}\\
& \ \ \ \circ \dots \circ  \Big (\hat f_2^{(k_1)} \circ H_{r_{1}}\circ L_{k_1}\Big)\circ \tilde f_1 (x)
\quad\text{for all $x\in \setB_{r_1,\dots,r_{j-1}}^{k_1,\dots, k_{j-1}}$,} \label{eq:fj2a}
\end{align}
$(k_1,\dots,k_{j-1})\in\setA_{j-1}$, and $r_1,\dots,r_{j-1} \in\{1,\dots,2^{s-1}\}$, 
where, for $r=1,\dots,2^{s-1}$,  the continuous piecewise linear function $H_r\colon \reals\to \reals$ is defined according to 
\begin{align}\label{eq:defHr}
H_r(x)=
\begin{cases}
2^s x -2(r-1)&\text{if $x\in (-\infty, (2r-1)/2^{s}]$}\\
-2^sx +2r&\text{if $x\in ( (2r-1)/2^{s},\infty]$.}\\
\end{cases}
\end{align}   
Suppose first that $j=2$ and fix $k_1$, $r_1$, and $x\in\setB_{r_1}^{k_1}$ arbitrarily. Then, we have 
\begin{align}
\tilde f^{(s)}_2(x)
&=\hat f_{2}^{(k_1)}\circ g_s\circ L_{k_{1}}\circ \tilde f_1(x) \label{eq:usestep1}\\
&= \sum_{r=1} ^{2^{s-1}} \hat f_{2}^{(k_1)}\circ  h_r   \circ   L_{k_{1}}   \circ \tilde f_1(x) \label{eq:usestep1a}\\
&= \hat f_{2}^{(k_1)}\circ  H_{r_1}   \circ   L_{k_{1}}   \circ \tilde f_1(x), \label{eq:usestep1b}
\end{align}
where  \cref{eq:usestep1} is by \cref{eq:fj1}-\cref{eq:fj2}, in  \cref{eq:usestep1a} we applied \Cref{lem:sawtooth1}, and \cref{eq:usestep1b} follows from $ \setB_{r_1}^{k_1}= {\tilde f_1}^{-1}(\setC_{r_1}^{k_1})$,    
\begin{align}\label{eq:LC1}
L_{k_1}(\setC_{r_1}^{k_1})\subseteq[(r_1-1)/2^{s-1}, r_1/2^{s-1}],
\end{align}
and 
$H_{r_1}(x)=h_{r_1}(x)$ for all $x\in [(r_1-1)/2^{s-1}, r_1/2^{s-1}]$.
Now, suppose that  \cref{eq:fj1a}-\cref{eq:fj2a} holds for $j<d$ and fix  $(k_1,\dots,k_{j})\in\setA_{j}$, $r_1,\dots, r_{j}\in\{1,\dots,2^{s-1}\}$, and $x\in\setB_{r_1,\dots,r_{j}}^{k_1,\dots,k_{j}}$ arbitrarily. Then, we have 
\begin{align}
\tilde f^{(s)}_{j+1}(x) &= \Big(\hat f_{j+1}^{(k_1,\dots, k_j)}\circ g_s\circ L_{k_{j}}\Big) 
\circ \tilde f^{(s)}_{j}(x)\label{useintstep1}\\
&= \sum_{r=1} ^{2^{s-1}} \Big(\hat f_{j+1}^{(k_1,\dots, k_j)}\circ h_r \circ L_{k_{j}}\Big) 
\circ \tilde f^{(s)}_{j}(x)\label{eq:usestep1aa}\\
&= \hat f_{j+1}^{(k_1,\dots, k_j)}\circ H_{r_j} \circ L_{k_{j}}
\circ \tilde f^{(s)}_{j}(x),\label{eq:lastincuctionstep}
\end{align}
where \cref{useintstep1} is by \cref{eq:laststepindfj0}--\cref{eq:laststepindfj}, in  \cref{eq:usestep1aa} we applied \Cref{lem:sawtooth1}, and \cref{eq:lastincuctionstep} follows from 
$ \setB_{r_1,\dots,r_{j}}^{k_1,\dots,k_{j}}\subseteq {\tilde f^{(s)}_j}^{-1}(\setC_{r_j}^{k_j})$,
\begin{align}\label{eq:LCj}
L_{k_j}(\setC_{r_j}^{k_j})\subseteq[(r_j-1)/2^{s-1}, r_j/2^{s-1}],
\end{align}
and $H_{r_j}(x)=h_{r_j}(x)$ for all $x\in [(r_j-1)/2^{s-1}, r_j/2^{s-1}]$. 
Plugging the expression for $\tilde f^{(s)}_{j}$ from \cref{eq:fj1a}--\cref{eq:fj2a} into \cref{eq:lastincuctionstep} establishes \cref{eq:fj1a}--\cref{eq:fj2a}  for $j+1$.\\

\emph{Step 3:}\ We prove by induction that, for   $j=2,\dots,d$, we have 
\begin{align}
&\Bigg(\Big(\hat f_{j}^{(k_1,\dots, k_{j-1})}\circ H_{r_{j-1}}\circ L_{k_{j-1}}\Big) 
\circ  \Big(\hat f_{j-1}^{(k_1,\dots, k_{j-2})}\circ H_{r_{j-2}}\circ L_{k_{j-2}} \Big) \label{eq:setA1}\\
& \ \ \ \circ \dots \circ  \Big (\hat f_2^{(k_1)} \circ H_{r_{1}}\circ L_{k_1}\Big)\circ \tilde f_1 \Bigg)^{-1}(\setA)
 \label{eq:setA2}\\
 &=\Bigg(\Big(f_{j}^{(k_1,\dots, k_{j})}\circ H_{r_{j-1}}\circ L_{k_{j-1}}\Big) 
\circ  \Big( f_{j-1}^{(k_1,\dots, k_{j-1})}\circ H_{r_{j-2}}\circ L_{k_{j-2}} \Big) \label{eq:setA3}\\
& \ \ \ \circ \dots \circ  \Big ( f_2^{(k_1,k_2)} \circ H_{r_{1}}\circ L_{k_1}\Big)\circ  f_1^{k_1} \Bigg)^{-1}(\setA)
 \label{eq:setA4}
\end{align}
for all 
$(k_1,\dots,k_{j})\in\setA_{j}$, $r_1,\dots,r_{j} \in\{1,\dots,2^{s-1}\}$, and $\setA\subseteq\overline{\setI}_{k_j}$, 
with $H_r$ as defined in \cref{eq:defHr} for $r=1,\dots, 2^{s-1}$.  
Suppose that $j=2$ and fix $(k_1,k_2)\in\setA_{2}$, $r_1,r_2 \in\{1,\dots,2^{s-1}\}$, and $\setA\subseteq\overline{\setI}_{k_2}$ arbitrarily. Then, 
\cref{item1space} implies 
\begin{align}
{\hat f_2^{(k_1)} }^{-1}(\setA) = {f_2^{(k_1,k_2)} }^{-1}(\setA) \subseteq [0,1]. 
\end{align}
Moreover, we have 
\begin{align}
L_{k_1}^{-1} (H_{r_1}^{-1}([0,1])) \subseteq \overline{\setC}_{r_1}^{k_1} \subseteq \overline{\setI}_{k_1}
\end{align}
so that \cref{item1space}  yields 
\begin{align}
\Big(\hat f_2^{(k_1)} \circ H_{r_{1}}\circ L_{k_1}\circ \tilde f_1 \Big)^{-1}(\setA) 
= 
 \Big ( f_2^{(k_1,k_2)} \circ H_{r_{1}}\circ L_{k_1}\circ  f_1^{k_1} \Big)^{-1}(\setA),
\end{align}
which establishes \cref{eq:setA1}--\cref{eq:setA4} for $j=2$. 
Next, suppose that  \cref{eq:setA1}--\cref{eq:setA4} holds for $j<d$ and fix 
 $(k_1,\dots,k_{j+1})\in\setA_{j+1}$, $r_1,\dots,r_{j+1} \in\{1,\dots,2^{s-1}\}$, and $\setA\subseteq\overline{\setI}_{k_{j+1}}$ arbitrarily. Then, 
\cref{item1space} implies 
\begin{align}\label{eq:comb1}
{\hat f_{j+1}^{(k_1,\dots,k_j)} }^{-1}(\setA) = {f_{j+1}^{(k_1,\dots,k_{j+1})} }^{-1}(\setA) \subseteq [0,1]. 
\end{align}
Moreover, we have  
 \begin{align}\label{eq:comb2}
L_{k_j}^{-1} (H_{r_j}^{-1}([0,1]))&=\overline{\setC}_{r_j}^{k_j} \subseteq \overline{\setI}_{k_j}.   
\end{align}
Now, \cref{eq:comb1} and \cref{eq:comb2} 
imply 
 \begin{align}\label{eq:comb3}
L_{k_j}^{-1} (H_{r_j}^{-1}({\hat f_{j+1}^{(k_1,\dots,k_j)}}^{-1}(\setA)))&\subseteq \overline{\setI}_{k_j}.   
\end{align}
Finally, combining \cref{eq:comb1}, \cref{eq:comb3}, and  \cref{eq:setA1}--\cref{eq:setA4}  establishes \cref{eq:setA1}--\cref{eq:setA4} for $j+1$.  \\

\emph{Step 4:}\ Combining the results from Step 2 and Step 3, we obtain 
\begin{align}
\tilde f^{(s)}_j(x) &= 
\Big(f_{j}^{(k_1,\dots, k_{j})}\circ H_{r_{j-1}}\circ L_{k_{j-1}}\Big) 
\circ  \Big( f_{j-1}^{(k_1,\dots, k_{j-1})}\circ H_{r_{j-2}}\circ L_{k_{j-2}} \Big) \label{eq:fj1anew}\\
& \ \ \ \circ \dots \circ  \Big (f_2^{(k_1,k_2)} \circ H_{r_{1}}\circ L_{k_1}\Big)\circ f^{(k_1)}_1 (x)
\quad\text{for all $x\in \setB_{r_1,\dots,r_{j}}^{k_1,\dots, k_{j}}$,} \label{eq:fj2anew}
\end{align}
$(k_1,\dots,k_{j})\in\setA_{j}$, and $r_1,\dots,r_{j} \in\{1,\dots,2^{s-1}\}$, 
and $j=2,\dots,d$, 
with $H_r$ as defined in \cref{eq:defHr} for $r=1,\dots, 2^{s-1}$.\\  

\emph{Step 5:}\ We establish 
\begin{align}\label{eq:measureBupper}
\colL^{(1)}\big(\setB_{r_1,\dots,r_{d}}^{k_1,\dots,k_{d}} \big)\leq \frac{w_{k_1,\dots,k_d}}{2^{d(s-1)}}
\end{align}
for all $(k_1,\dots,k_d)\in\setA_d$ and $r_1,\dots,r_d \in\{1,\dots,2^{s-1}\}$.
To this end, fix $(k_1,\dots,k_d)\in\setA_d$ and $r_1,\dots,r_d \in\{1,\dots,2^{s-1}\}$ arbitrarily and 
 note that \cref{eq:fj1anew}--\cref{eq:fj2anew} for $j=d$ implies 
\begin{align}\label{eq:needarea1}
(J\tilde f^{(s)}_d)|_{\setB_{r_1,\dots,r_{d}}^{k_1,\dots,k_{d}}} = \frac{2^{s(d-1)}}{K w_{(k_1,\dots, k_d)}}
\end{align}
and, since $\abs{H_r^{-1}(x)}\leq 2$ for all $x\in\reals$ and $r=1,\dots,2^{s-1}$,  
\begin{align}\label{eq:needarea2}
\abs{\setB_{r_1,\dots,r_{d}}^{k_1,\dots,k_{d}}\cap \Big\{{\tilde f^{(s)}_d}^{-1}(x)\Big\}} \leq 2^{d-1}\quad\text{for all $x\in [0,\infty)$}. 
\end{align}
We therefore have 
\begin{align}
\colL^{(1)}\big(\setB_{r_1,\dots,r_{d}}^{k_1,\dots,k_{d}} \big)
&=  \frac{K w_{(k_1,\dots, k_d)}}{2^{s(d-1)}}  \int_{\setB_{r_1,\dots,r_{d}}^{k_1,\dots,k_{d}}}  (J\tilde f^{(s)}_d)\,\mathrm d \colL^{(1)}\label{eq:needarea1a}\\
&=\frac{K w_{(k_1,\dots, k_d)}}{2^{s(d-1)}} \int_{\setC_{r_d}^{k_d}}  \abs{\setB_{r_1,\dots,r_{d}}^{k_1,\dots,k_{d}}\cap \Big\{{\tilde f^{(s)}_d}^{-1}(x)\Big\}} \,\mathrm d \colL^{(1)}\label{eq:needarea1b}\\
&\leq \frac{K w_{(k_1,\dots, k_d)}}{2^{(s-1)(d-1)}}\colL^{(1)}(\setC_{r_d}^{k_d})\label{eq:needarea1c}\\
&=\frac{w_{k_1,\dots,k_d}}{2^{d(s-1)}}, 
\end{align}
where \cref{eq:needarea1a} follows from \cref{eq:needarea1}, in \cref{eq:needarea1b} we applied the 
area formula \Cref{thm:area}, $\tilde f^{(s)}_d(\setB_{r_1,\dots,r_{d}}^{k_1,\dots,k_{d}})\subseteq \setC_{r_d}^{k_d}$, and  \cref{eq:itemHLeb} in \Cref{lem:Hmeasure}, and \cref{eq:needarea1c} is by \cref{eq:needarea2}, which proves  \cref{eq:measureBupper}.\\ 

\emph{Step 6:}\ We prove by induction that, for  $j=2,\dots,d$, we have    
\begin{align}\label{eq:fjwelldefined}
\tilde f^{(s)}_{j}([0,1]) \subseteq [0,1] 
\end{align}
and 
\begin{align}
\bigcup_{k_{1},\dots,k_j=1}^{K} \bigcup_{r_1,\dots,r_{j}=1}^{2^{s-1}} \setB_{r_1,\dots,r_{j}}^{k_1,\dots,k_{j}} 
=[0,1].\label{eq:setBunion}
\end{align} 
In particular, this implies $\tilde f^{(s)}([0,1])\subseteq \setI$. 
Suppose that $j=2$. Then,  we have 
\begin{align}
 \bigcup_{k_{1}=1}^{K} \bigcup_{r_{1}=1}^{2^{s-1}} \setB_{r_1}^{k_1} 
 &=\bigcup_{k_{1}=1}^{K} \bigcup_{r_{1}=1}^{2^{s-1}} {\tilde f_1}^{-1}( \setC_{r_1}^{k_1})  \label{eq:usef10}\\
 &={\tilde f_1}^{-1} \Bigg( \bigcup_{k_{1}=1}^{K} \bigcup_{r_{1}=1}^{2^{s-1}} \setC_{r_1}^{k_1} \Bigg)\\
&= {\tilde f_1}^{-1} ([0,1])\\
&= [0,1], \label{eq:usef11}
\end{align}
where in \cref{eq:usef11} we used the fact that 
\begin{align}\label{eq:f11}
\tilde f_1(x) \in [0,1]\quad\text{ for all $x\in [0,1]$,} 
\end{align}
which follows from \cref{eq:tildef1} and $f_1^{(k_1)}(x)\in [0,1]$ for all $x\in\setK_{k_1}$ and $k_1=1,\dots,K$. 
Now, pick  $x\in [0,1]$ arbitrarily. Then, \cref{eq:usef10}--\cref{eq:usef11} implies that there exists a $k_1$ and an $r_1$ such that $x\in \setB_{r_1}^{k_1}$. We therefore have 
\begin{align}
\tilde f^{(s)}_2(x) 
&= \hat f_{2}^{(k_1)}\circ  H_{r_1}   \circ   L_{k_{1}}   \circ \tilde f_1(x) \label{eq:usestep2}\\
&\in \hat f_{2}^{(k_1)}\circ  H_{r_1}   \circ   L_{k_{1}} (\setC_{r_1}^{k_1})  \label{eq:usestep2a}\\
&\subseteq \hat f_{2}^{(k_1)}([0,1])\label{eq:usestep2b}\\
&=[0,1],
\end{align}
where \cref{eq:usestep2} follows from \cref{eq:fj1a}--\cref{eq:fj2a} and 
in \cref{eq:usestep2b}  we used \cref{eq:LC1}, which proves  \cref{eq:fjwelldefined} for $j=2$. Next, note that 
\begin{align}
\bigcup_{k_1,k_{2}=1}^{K} \bigcup_{r_1,r_{2}=1}^{2^{s-1}} \setB_{r_1,r_2}^{k_1,k_2}
&= \bigcup_{k_1,k_{2}=1}^{K} \bigcup_{r_1,r_{2}=1}^{2^{s-1}} \Big(\setB_{r_1}^{k_1}  \cap {\tilde f^{(s)}_2}^{-1}( \setC_{r_2}^{k_2})\Big)\\
&= \bigcup_{k_1=1}^{K} \bigcup_{r_1=1}^{2^{s-1}} \setB_{r_1}^{k_1} \cap \Big( \bigcup_{k_{2}=1}^{K} \bigcup_{r_{2}=1}^{2^{s-1}}  {\tilde f^{(s)}_2}^{-1}( \setC_{r_2}^{k_2}) \Big)\\
&=  [0,1]\cap{\tilde f^{(s)}_2}^{-1}\Big( \bigcup_{k_{2}=1}^{K} \bigcup_{r_{2}=1}^{2^{s-1}}  \setC_{r_2}^{k_2} \Big)\label{eq:useBkiri}\\
&= [0,1]\cap {\tilde f^{(s)}_2}^{-1}( [0,1]) \\
&=  [0,1],\label{eq:usewelldef1} 
\end{align}
where \cref{eq:useBkiri}  is by \cref{eq:usef10}--\cref{eq:usef11} and \cref{eq:usewelldef1}  follows from \cref{eq:fjwelldefined} for $j=2$, which establishes   \cref{eq:setBunion} for $j=2$.
Now, suppose that \cref{eq:fjwelldefined} and \cref{eq:setBunion} hold for $j<d$ and pick $x\in [0,1]$ arbitrarily.  Then, there must exist a $(k_1,\dots,k_{j})\in\setA_{j}$ and $r_1,\dots, r_{j}\in\{1,\dots,2^{s-1}\}$ such that $x\in \setB_{r_1,\dots,r_{j}}^{k_1,\dots,k_{j}} $. We therefore have 
\begin{align}
\tilde f^{(s)}_{j+1}(x) 
&= \hat f_{j+1}^{(k_1,\dots,k_j)}\circ  H_{r_j}   \circ   L_{k_{j}}   \circ \tilde f^{(s)}_j(x) \label{eq:usestep2A}\\
&\in \hat f_{j+1}^{(k_1)}\circ  H_{r_j}   \circ   L_{k_{j}} (\setC_{r_j}^{k_j})  \label{eq:usestep2aA}\\
&\subseteq \hat f_{j+1}^{(k_1)}([0,1])\label{eq:usestep2bA}\\
&=[0,1],
\end{align}
where \cref{eq:usestep2A} follows from  \cref{useintstep1}--\cref{eq:lastincuctionstep}, which proves  \cref{eq:fjwelldefined} for $j+1$. 
Finally, note that 
\begin{align}
&\bigcup_{k_1,\dots,k_{j+1}=1}^{K} \bigcup_{r_1,\dots,r_{j+1}=1}^{2^{s-1}} \setB_{r_1,\dots,r_{j+1}}^{k_1,\dots,k_{j+1}}\\
&= \bigcup_{k_1,\dots,k_{j+1}=1}^{K} \bigcup_{r_1,\dots,r_{j+1}=1}^{2^{s-1}} \Big(\setB_{r_1,\dots,r_{j}}^{k_1,\dots,k_{j}} \cap ((\tilde f^{(s)}_{j+1})^{-1}( \setC_{r_{j+1}}^{k_{j+1}}))\Big)\\
&= \bigcup_{k_1,\dots,k_{j}=1}^{K} \bigcup_{r_1,,\dots,r_{j}=1}^{2^{s-1}} \setB_{r_1,r_2,\dots,r_{j}}^{k_1,k_{2},\dots,k_{j}} \cap \Big( \bigcup_{k_{j+1}=1}^{K} \bigcup_{r_{j+1}=1}^{2^{s-1}}  (\tilde f^{(s)}_{j+1})^{-1}( \setC_{r_{j+1}}^{k_{j+1}}) \Big)\\
&= [0,1]\cap (\tilde f^{(s)}_{j+1})^{-1}\Big( \bigcup_{k_{j+1}=1}^{K} \bigcup_{r_{j+1}=1}^{2^{s-1}}  \setC_{r_{j+1}}^{k_{j+1}} \Big)\label{eq:useBkiriA}\\
&=  [0,1]\cap(\tilde f^{(s)}_{j+1})^{-1}( [0,1]) \\
&=  [0,1],\label{eq:usewelldef1A} 
\end{align}
where \cref{eq:useBkiriA}  is by \cref{eq:setBunion} for $j$ and \cref{eq:usewelldef1A}  follows from \cref{eq:fjwelldefined} for $j+1$, which establishes   \cref{eq:setBunion} for $j+1$.\\

\emph{Step 7:}   
 To establish \Cref{eq:propB2}, fix $(k_1,\dots, k_d)\in\setA_d$ and $r_1,r_2,\dots,r_d \in\{1,\dots,2^{s-1}\}$ arbitrarily. Then, we have 
\begin{align}
\mu\big(\setC_{r_1}^{k_1}\times\setC_{r_2}^{k_2}\dots \times \setC_{r_d}^{k_d} \big)
= \frac{w_{(k_1,\dots, k_d)}}{2^{d(s-1)}}. 
\end{align}
Combining \cref{eq:measureBupper} with \cref{eq:setBunion} yields 
\begin{align}\label{eq:measureBupper2}
(\tilde f^{(s)}_d\#\colL^{(1)}|_{[0,1]})\big(\setC_{r_1}^{k_1}\times\setC_{r_2}^{k_2}\dots \times \setC_{r_d}^{k_d} \big)
=\colL^{(1)}|_{[0,1]}(\setB_{r_1,\dots,r_{j}}^{k_1,\dots,k_{j}})
= \frac{w_{k_1,\dots,k_d}}{2^{d(s-1)}}
\end{align}
upon noting that $(\tilde f^{(s)}_d\#\colL^{(1)}|_{[0,1]})(\setI)=1$ and 
\begin{align}
\sum_{k_1,\dots,k_d=1}^K \sum_{r_1,\dots,r_d=1}^{2^{s-1}} \frac{w_{(k_1,\dots, k_d)}}{2^{d(s-1)}} =1. 
\end{align}
  Next, we prove \cref{item3space}. 
\Cref{eq:propB3} follows from  \Cref{lem:kappa2} with $K2^{s-1}$ in place of $K$.  It remains to establish   \cref{eq:propB4}. To this end, note that the functions $\tilde f^{(s)}_j$ defined in 
\cref{eq:tildef1A}--\cref{eq:tildef1B}  and \cref{eq:fj1anew}--\cref{eq:fj2anew}  are piecewise linear on closed and convex sets.  \Cref{lem:lipconvex}   therefore implies 
\begin{align}\label{eq:needarea1aA}
\operatorname{Lip}(\tilde f^{(s)}_j) \leq \frac{2^{s(d-1)}N}{K}\quad\text{for $j=1,\dots,d$}, 
\end{align}
which proves \Cref{eq:propB4}. 

 It remains to establish \cref{item4space}. We first realize the individual mappings appearing in the composition of the mappings $\tilde f_j$ as ReLU neural networks.  
\begin{enumerate}
\renewcommand{\theenumi}{(\alph{enumi})}
\renewcommand{\labelenumi}{(\alph{enumi})}

\item \label{itemfjnna}Consider the affine mappings $W_1\colon\reals \to \reals^3$, $W_2\colon\reals^3 \to \reals^3$, and $W_3\colon\reals^3 \to \reals$ defined according to 
\begin{align}
W_1(x)&=
\begin{pmatrix}
2\\4\\2
\end{pmatrix}
x- 
\begin{pmatrix}
0\\2\\2
\end{pmatrix}\\[2mm]
W_2(x)&=
\begin{pmatrix}
2&-2&2\\
4&-4&4\\
2&-2& 2
\end{pmatrix}
x- 
\begin{pmatrix}
0\\2\\2
\end{pmatrix}\\[2mm]
W_3(x)&=
\begin{pmatrix}
1&-1&1
\end{pmatrix}
x,
\end{align}
respectively, 
and define, for every $s\in\naturals$, the ReLU neural network $\Phi_s\in\setN_{1,1}$ according to 
\begin{align}
\Phi_s=W_3\circ\rho\circ \underbrace{W_2\circ\rho\dots W_2\circ\rho}_{(s-1)\ \mathrm{times}} W_1 
\end{align}
with $\mathcal L(\Phi_s)=s+1$,  $\mathcal M(\Phi_s)=11s-3$, 
$\setK(\Phi_s)=\{0,1, 2,4,-1,-2,-4\}\subseteq \setD_{N,K}$, and $\setW(\Phi_s)=3$. Then, we have $\Phi_s=g_s$. 

\item \label{itemfjnnb}For  $\ell=1,\dots,K$, consider the affine mapping $L_\ell\colon \reals\to\reals$ defined according to $L_\ell(x)=Kx-\ell+1$. Further, define the affine mappings $V_1\colon \reals\to \reals^2$ and $V_\ell\colon \reals\to \reals$ according to 
\begin{align}
V_1(x)&=
\begin{pmatrix}
1\\-1
\end{pmatrix}
x\\[2mm]
V_2(x)&=
\begin{pmatrix}
K&-K
\end{pmatrix}
x -\ell+1 
\end{align}
and define, for $\ell=1,\dots,K$, the ReLU neural network $\Theta_\ell\in\setN_{1,1}$ according to 
\begin{align}
\Theta_\ell=V_2\circ\rho\circ V_1 
\end{align}
 with $\mathcal L(\Theta_\ell)=2$,  $\mathcal M(\Theta_\ell)=5$, 
$\setK(\Theta_\ell)=\{-K,\dots, K\}\subseteq \setD_{N,K}$,  and $\setW(\Theta_\ell)=2$. Then, we have $\Theta_\ell=L_\ell$. 

\item \label{itemfjnnc} Consider the mapping $\tilde f_1$ as defined in \cref{eq:tildef1A}--\cref{eq:tildef1B}.  
Further, consider the affine mappings 
$U_1\colon \reals\to \reals^{2K-1}$ and $U_2\colon \reals^{2K-1}\to \reals$ defined according to 
\begin{align} 
U_1(x)= 
\begin{pmatrix}
-x\\
x\\
x-b_1\\
x-b_1\\
\vdots\\
x-b_{K-1}\\
x-b_{K-1}
\end{pmatrix}
\end{align}
 and 
 \begin {align}
 U_2(x)= \tp{\begin{pmatrix}-x\\x/(Kw_1)\\ x/(Kw_2)\\-x/(Kw_1) \\\dots\\ x/(Kw_K) \\ - x/(Kw_{K-1})
 \end{pmatrix}},
 \end{align}
  respectively and define the ReLU neural network $\Psi_1\in\setN_{1,1}$ according to 
\begin{align}
\Psi_1=U_2\circ\rho\circ U_1 
\end{align}
 with $\mathcal L(\Psi_1)=2$,  $\mathcal M(\Psi_1)=4K-1$, and  $\setW(\Psi_1)=2K$, and 
  $\setK(\Psi_1)\subseteq \setD_{N,K}$. 
  Then, \Cref{lem:piecewiselinear} implies $\Psi_1=\tilde f_1$. 

\item \label{itemfjnnd}
Fix $(k_1,\dots, k_{j-1})\in\setA_{j-1}$ and $j=2,\dots,d$ arbitrarily, 
consider the mapping 
$\hat f_j^{(k_1,\dots, k_{j-1})}$  as defined in 
\cref{eq:hatfjA}--\cref{eq:hatfjB}. 
Further, consider the affine mappings $U^{(k_1,\dots, k_{j-1})}_1\colon \reals\to \reals^{2K-1}$ and $U^{(k_1,\dots, k_{j-1})}_2\colon \reals^{2K-1}\to \reals$ defined according to 
\begin{align} 
U^{(k_1,\dots, k_{j-1})}_1(x)= 
\begin{pmatrix}
-x\\
x\\
x-b_{1|(k_1,\dots, k_{j-1})}\\
x-b_{1|(k_1,\dots, k_{j-1})}\\
\vdots\\
x-b_{K-1|(k_1,\dots, k_{j-1})}\\
x-b_{K-1|(k_1,\dots, k_{j-1})}
\end{pmatrix}
\end{align}
 and 
 \begin {align}
 U^{(k_1,\dots, k_{j-1})}_2(x)= 
 \tp{
 \begin{pmatrix}
 -x\\
 x/(Kw_{1|(k_1,\dots, k_{j-1})})\\ 
 x/(Kw_{2|(k_1,\dots, k_{j-1})})\\
 -x/(Kw_{1|(k_1,\dots, k_{j-1})}) \\
 \vdots\\ 
 x/(Kw_{K|(k_1,\dots, k_{j-1})}) \\ 
 - x/(Kw_{K-1|(k_1,\dots, k_{j-1})})
 \end{pmatrix}},
 \end{align}
  respectively and define the ReLU neural network 
$\Psi^{(k_1,\dots, k_{j-1})}_1\in\setN_{1,1}$ according to 
\begin{align}
\Psi_j^{(k_1,\dots, k_{j-1})}=U^{(k_1,\dots, k_{j-1})}_2\circ\rho\circ U^{(k_1,\dots, k_{j-1})}_1 
\end{align}
 with $\mathcal L(\Psi_j^{(k_1,\dots, k_{j-1})})=2$,  $\mathcal M(\Psi_j^{(k_1,\dots, k_{j-1})})=4K-1$,  $\setW(\Psi_j^{(k_1,\dots, k_{j-1})})=2K$,  
 and $ \setK(\Psi_j^{(k_1,\dots, k_{j-1})}) \subseteq \setD_{N,K}$. 
  Then, \Cref{lem:piecewiselinear} implies $\Psi_j^{(k_1,\dots, k_{j-1})}=\hat f_j^{(k_1,\dots, k_{j-1})}$. 

\item \label{itemfjnne}
Fix $j\in\{2,\dots,d\}$ arbitrarily and  
consider the mapping 
$\tilde f_j^{(s)}$  as defined in \cref{eq:tildefj1}--\cref{eq:tildefj2}. 
Combining the ReLU neural networks  in  \cref{itemfjnna}--\cref{itemfjnnd} by adding a void ReLU activation function $\rho$ between all the individual ReLU networks, we can construct, for every $(k_1,\dots, k_{j_1})\in \setA_{j-1}$, a ReLU neural network $\Psi^{(s)}_{(k_1,\dots, k_{j-1})}$ with\footnote{By adding void components of the form $1\circ \rho \dots 1 \circ \rho$ of  lengths $(d-j)(5+s)$, we can make all networks to have  depth $2+(d-1)(5+s)$.}
\begin{align}
\mathcal L(\Psi^{(s)}_{(k_1,\dots, k_{j-1})})&=2+(d-1)(5+s)\\
\mathcal M(\Psi^{(s)}_{(k_1,\dots, k_{j-1})})&=(d-j)(5+s)+(j-1)(4K+11s+1)  +4K-1\\
\mathcal W(\Psi^{(s)}_{(k_1,\dots, k_{j-1})})&=2K\\  
\setK(\Psi^{(s)}_{(k_1,\dots, k_{j-1})}) &\subseteq \setD_{N,K}
\end{align} 
such that 
\begin{align}\label{eq:sumtildefj}
\tilde f_j^{(s)}= \sum_{(k_1,\dots, k_{j-1})\in\setA_{j-1}}\Psi^{(s)}_{(k_1,\dots, k_{j-1})}.
\end{align}
\item 
\label{eq:itemfjnnfpre}
Fix $j\in\{2,\dots,d\}$ arbitrarily and  
consider the mapping 
$\tilde f_j^{(s)}$  as defined in \cref{eq:tildefj1}--\cref{eq:tildefj2}.  Combining \Cref{lem:P} with \cref{itemfjnne} and realizing the sum in \cref{eq:sumtildefj} as inner product with the all-ones vector of length $K^{j-1}$,   
we can construct a  ReLU neural network $\Psi^{(s)}_{j}$ with 
\begin{align}
\mathcal L(\Psi^{(s)}_{j})&=3+(j-1)(5+s)\\
\mathcal M(\Psi^{(s)}_{j})&=\Big((d-j)(5+s)+(j-1)(4K+11s+1)+4K\Big)K^{j-1}  \\
\mathcal W(\Psi^{(s)}_{j})&=2K^{j}\\  
\setK(\Psi^{(s)}_{j}) &\subseteq \setD_{N,K}
\end{align} 
such that $\tilde f_j^{(s)}=\Psi^{(s)}_{j}$.
\item \label{itemfjnnf}
Consider the ReLU neural networks $\Psi_j^{(s)}$ constructed in  \cref{eq:itemfjnnfpre}. Then, 
the  ReLU neural network $\Phi_s=P(\Psi_1^{(s)},\dots,\Psi_d^{(s)})$ has  the desired properties in \cref{item4space} upon noting that 
\begin{align}
&\sum_{j=1}^d \Big((d-j)(5+s)+(j-1)(4K+11s+1)+4K\Big)K^{j-1}\\
&\leq 4d(3s+2K+4)K^d/(K-1)
\end{align} 
and 
\begin{align}
2 \sum_{j=1}^d K^{j} 
&= 2K(K^d-1)/(K-1) \\
&\leq 4 K^d. 
\end{align}
Finally, \cref{eq:propB4psi} follows from  \cref{eq:propB4} and  \cref{eq:corpush}. 
\end{enumerate}
\end{proof}

We next combine  \Cref{thm:approxhisto,thm:quantunif}  to obtain space-filling approximations of arbitrary Borel measures on $[0,1]^d$. 

\begin{cor}\label{cor:mupushforward}
Let $\reals^d$ be equipped with the metric $\rho(x,y)=\lVert x-y\rVert_\infty$. 
Further, let  $K\in\naturals\setminus\{1\}$ and set 
\begin{align}
 \setD_{K} =\{ a/b: a\in\mathbb{Z},b\in\naturals, \abs{a}\leq 4K^{d+1} \ \text{and}\ b\leq 4K^{d+2}\}.  
\end{align}
Then, there exist a collection $\colG^{(d)}_K\subseteq \mathcal{N}_{1,d}$ of ReLU neural networks with $\abs{\colG^{(d)}_K}\leq (12K)^{K^d}$ such that, for every  Borel probability measure   $\mu$ on $\setR^d$ with $\operatorname{spt}(\mu)\subseteq [0,1]^d$, there is a $\Sigma\in\colG^{(d)}_K$ satisfying $\Sigma([0,1])\subseteq \setI$ and  
\begin{align}
W_1\Big( \mu,\Sigma\#(\mathcal{L}^{(1)}|_{[0,1]})\Big)\leq 3/K.\label{eq:corK1}
\end{align}
Moreover, the ReLU neural networks $\Sigma\in\colG^{(d)}_K$ have the  same architecture depending only on  $d$ and $K$ 
and satisfy
\begin{align}
\mathcal L(\Sigma)&=3+6(d-1)\label{eq:corK3}\\
\mathcal M(\Sigma)&\leq 44d K^d  \\
\mathcal W(\Sigma)&\leq 4 K^d\\ 
 \mathcal B(\Sigma)&=4K^{d+1}\\
\setK(\Sigma) &\subseteq \setD_K \label{eq:corK2}\\
\operatorname{Lip}(\Sigma) &\leq 2^{d+1}K^d.
\end{align}
\end{cor}
\begin{proof}
Arbitrarily fix a  Borel probability measure   $\mu$ on $\setR^d$ with $\operatorname{spt}(\mu)\subseteq [0,1]^d$ and  let $\tilde \mu$ be a $(1/N)$-quantized uniform mixture associated with $\mu$ according to  \Cref{thm:approxhisto} with $N=4K^{d+1}$. 
Further, consider the ReLU neural network $\Sigma:=\Phi_{4K^{d+1},K}^{(1)}$ 
from \cref{item4space} in \Cref{thm:quantunif} when applied to   $\tilde \mu$  with $N=4K^{d+1}$ and   $s=1$. 
Then, we have $K([0,1])\subseteq \setI$ thanks to \cref{item2space} and  \cref{item4space} in \Cref{thm:quantunif}  and  
\begin{align}
W_1\Big( \mu,\Sigma\#(\mathcal{L}^{(1)}|_{[0,1]})\Big)
&\leq  W_1(\mu,\tilde \mu)  + W_1\Big( \tilde \mu,\Sigma\#(\mathcal{L}^{(1)}|_{[0,1]})\Big)\\
&\leq (K^d-1)/K^{d+1}  + 2/K\label{eq:useboththeorems}\\
&\leq   3/K , 
\end{align}
where \cref{eq:useboththeorems} follows from 
\Cref{thm:approxhisto} and \cref{item3space}--\cref{item4space} in 
 \Cref{thm:quantunif}. 
Moreover,  $\Sigma$ satisfies \cref{eq:corK1}--\cref{eq:corK2} with architecture depending only on  $d$ and $K$ owing to 
\cref{item4space} in \Cref{thm:quantunif}  upon noting that 
$(2K+7)/(K-1)\leq 11$. Finally, set 
\begin{align}
\setA_d=\{k\in\naturals^d :\lVert k\rVert_\infty\leq K \}
\end{align}
and note that $\Sigma$ has fixed architecture and, therefore,  is completely determined by  the set of $1/N$-quantized weights $\{w_k\}_{k\in\setA_d}$ 
of  $\tilde \mu$. 
 But
\begin{align}
\abs {\{w_k\}_{k\in\setA_d}} 
&\leq \binom{N-1}{K^d-1}\\
&\leq \binom{N}{K^d}\\
& = \binom{4K^{d+1}}{K^d}\\
&\leq ( 12K)^{K^d},\label{eq:usesterling}
\end{align}
where \cref{eq:usesterling} follows from $\binom{n}{k}\leq (en/k)^k$, which establishes $\abs{\colG^{(d)}_K}\leq (12K)^{K^d}$. 
 \end{proof}

\Cref{cor:mupushforward} implies that  the number of ReLU neural networks  required to approximate an arbitrary  Borel probability measure  $\mu$  supported on  $[0,1]^d$  through push-forward of $\mathcal{L}^{(1)}|_{[0,1]}$  with error no larger than $\varepsilon$ is upper-bounded by 
\begin{align}
(12\lceil 3/\varepsilon\rceil)^{\lceil 3/\varepsilon\rceil^d }. 
\end{align}

\appendix

\section{Properties of Lipschitz Mappings}\label{sec:LIP}
In this section, we  state  basic definitions and properties of Lipschitz  mappings required throughout the paper. We start with the definition of Lipschitz constants, gradients, and generalized Jacobian determinants  of Lipschitz mappings.   
\begin{dfn}\label{dfn:Lip}
Let $\setA\subseteq \reals^m$ and consider the Lipschitz mapping  $f\colon\setA\to \reals^n$.  
For $p\in[1,\infty]$,  we set 
\begin{align}
\operatorname{Lip}^{(p)}(f)=\inf\Bigg\{s  \in[0,\infty):  \sup_{\substack{x,y\in\setA\\ x\neq y}} \frac{\lVert f(x)-f(y)\rVert_\infty}{\lVert x-y\rVert_p} \leq s\Bigg\}
\end{align}
and let  $\operatorname{Lip}(f):=\operatorname{Lip}^{(\infty)}(f)$. 
We define the gradient of $f$ according to 
 $
\nabla f = 
(\partial_{x_1}f, \partial_{x_2}f,\dots,\partial_{x_m}f)$ at each point where $\nabla f$ exists and write 
\begin{align}\label{eq:Jacobian}
Jf= 
\begin{cases}
\sqrt{\det(\tp{(\nabla f)} \nabla f))} &\text{if $m\leq n$}\\[1mm]
\sqrt{\det( (\nabla f)\tp{\nabla f})} &\text{if $m> n$}
\end{cases}
\end{align}
for the generalized Jacobian determinant of $f$. 
\end{dfn}
Rademacher's theorem \cite[Theorem 5.1.11]{krpa08} ensures that  $\nabla f$ and $Jf$ in \Cref{dfn:Lip} are defined up to a set of Lebesgue measure zero and Lebesgue-measurable.
We next state an elementary property of Lipschitz functions on convex and closed domains that allows to upper-bound the Lipschitz by the Lipschitz constants on the individual   subdomains. 

\begin{lem}\label{lem:lipconvex}
Let $\setA, \setB_1, \dots, \setB_n\subseteq\reals^m$ be convex and closed satisfying  $\setA=\bigcup_{i=1}^n \setB_i$. Further, let $f\colon \setA\to\reals$ be Lipschitz. Then, we have 
\begin{align}\label{eq:lipfconvex}
\operatorname{Lip}^{(1)}(f)\leq \max_{i=1,\dots, n} \operatorname{Lip}^{(1)}(f|_{\setB_i}). 
\end{align}  
\end{lem}
\begin{proof}
Fix $x,y\in\setA$ with $x\neq y$ but otherwise arbitrari and set $\colL=\{ty+(1-t)x:t\in [0,1]\}$. 
We next establish the existence of $k\in\naturals$ and  $x_0,\dots,x_k\in\setL$ satisfying the following properties: 
(i) $x_0=x$, (ii) $x_k=y$, and (iii) for  $i=1,\dots, k$, there exists a $\setB_{j_i}$ such that $x_i,x_{i-1}\in \setB_{j_i}$. Set $x_0=x$.  
Owing to \Cref{lem:conv1} below, there exists an $\varepsilon_1\in (0,1]$ and a $\setB_{i_1}$ such that 
$\{ty+(1-t)x_0:t\in [0,\varepsilon_1]\}\subseteq\setB_{i_1}$. If $y\in \setB_{i_1}$, then set $x_1=y$ and we 
we  have established the statement with $k=1$.  
If $y\notin \setB_{i_1}$, then set $t_1=\max\{{t\in [0,1]: ty+(1-t)x_0\in\setB_{i_1}\}}$
and $x_1= t_1y+(1-t_1)x_0$. Now, again by invoking \Cref{lem:conv1}, 
there exists an $\varepsilon_2\in (0,1]$ and a $\setB_{i_2}$ such that 
$\{ty+(1-t)x_1:t\in [0,\varepsilon_1]\}\subseteq\setB_{2}$.  
By construction, $\setB_{i_2}\neq \setB_{i_1}$. 
If $y\in \setB_{i_2}$, then set $x_2=y$ and we have established the statement with $k=2$.  
If $y\notin \setB_{i_2}$, then set $t_2=\max\{{t\in [0,1]: ty+(1-t)x_1\in\setB_{i_2}\}}$
and $x_2= t_2y+(1-t_2)x_1$. Now, continue this way to find points $x_1,\dots, x_j$ 
and $\setB_{i_1}$, \dots $\setB_{i_j}$ with $x_0,x_1\in \setB_{i_1}$, \dots, 
$x_{j-1},x_j\in \setB_{i_j}$ and $y\notin \setB_{i_\ell}$ for $\ell=1,\dots,j-1$. 
If $y\in \setB_{i_j}$, then set $x_j=y$ and we have established the statement with $k=j$.  
 Now, the sets $\setB_{i_1},\dots,\setB_{i_j}$  satisfy $\setB_{i_s}\neq \setB_{i_{s-1}}$ for $s=1,\dots, j$ by construction, which implies that they are pairwise different  owing to the assumption that they are all convex and contain line segments of $\setL$.   
 Since there are a finite number of sets $\setB_i$,  this procedure has to stop after a finite number of steps $k\leq n$ with  $y\in \setB_{i_k}$.  
We therefore have 
\begin{align}
\abs{f(y)-f(x)}
&\leq \sum_{i=1}^k \abs{f(x_i)-f(x_{i-1})}\\ 
&\leq \sum_{i=1}^k  \operatorname{Lip}^{(1)}(f|_{\setB_{j_i}})\lVert x_i-x_{i-1}\rVert_1 \label{eq:explain1}\\
&\leq  \max_{i=1,\dots, n} \operatorname{Lip}^{(1)}(f|_{\setB_i}) \sum_{i=1}^k \lVert x_i-x_{i-1}\rVert_1 \\
&=  \max_{i=1,\dots, n} \operatorname{Lip}^{(1)}(f|_{\setB_i})  \lVert y-x\rVert_1,\label{eq:explain2}  
\end{align}
where \cref{eq:explain1} follows from (iii) above and in \cref{eq:explain2} we used 
\begin{align}
\sum_{i=1}^k \lVert x_i-x_{i-1}\rVert_1
&=\sum_{i=1}^k \lVert (t_i-t_{i-1})y-(t_i-t_{i-1})x\rVert_1\\
&= \sum_{i=1}^k(t_i-t_{i-1}) \lVert y-x\rVert_1\\
&=\lVert y-x\rVert_1, 
\end{align}
which establishes \cref{eq:lipfconvex} as $x,y$ were  arbitrary. 
\end{proof}
\
\begin{lem}\label{lem:conv1}
Let $\setA, \setB_1, \dots, \setB_n\subseteq\reals^m$ be convex and closed satisfying $\setA=\bigcup_{i=1}^n \setB_i$. Then, for every $x,y\in\setA$, there exists an $\varepsilon \in (0,1]$ and a $\setB_{i_0}$ such that $\{ty +(1-t)x: t\in [0,\varepsilon]\}\subseteq \setB_{i_0}$.  
\end{lem}
\begin{proof}
Toward a contradiction, take $x,y\in\setA$ such that the statement of the lemma is false and set 
$\setC=\bigcup_{i\in\{1,\dots n\}:x\in\setB_i} \setB_i$ and $\setD=\bigcup_{i\in\{1,\dots n\}:x\notin\setB_i} \setB_i$. 
Now, fix $k\in\naturals$ and a $\setB_i$ with $x\in \setB_i$, both  arbitrarily. 
We can hence  find  a $t_k^{(i)}\in (0,1/k]$ such that $t_k^{(i)}y +(1-t_k^{(i)})x\notin \setB_i$. 
As $\setB_i$ is convex and $x\in \setB_i$, this implies    $ty +(1-t)x\notin \setB_i$ for all  $t\in [t_k^{(i)},1]$. By the arbitrariness of $\setB_i$, this implies in turn that $x_k:=(1/k)y +(1-1/k)x\in \setD$.  Since $k$ was arbitrary, we can conclude that $x_k\in\setD$ for all $k\in\naturals$. But $x\notin \setD$ and $\lim_{k\to\infty}x_k=x$, which stands in  contradiction to the fact that $\setD$ as the finite union of closed sets is closed.  
\end{proof}

Finally, we need the following version of the area formula.  

\begin{thm}\label{thm:area}\cite[Corollary 5.1.13]{krpa08},\cite[Theorem 2.71]{amfupa00} Let $m,n\in\naturals$ with $m\leq n$. Further, 
let $\setA\subseteq \reals^m$ be Lebesgue measurable and let $f=\tp{(f_1,\dots,f_n)}\colon \reals^m \to\reals^n$ be Lipschitz. Finally, let $g\colon \setA\to[0,\infty)$ be Lebesgue measurable and define $h\colon \reals^n\to [0,\infty]$ according to 
\begin{align}
h(y)=\sum_{x\in\setA \cap f^{-1}(y)} g(x).
\end{align}
Then, $h$ is $\setH^m$-measurable and  
\begin{align}
\int_\setA g \,Jf \,\mathrm d \setL^m=\int_{\reals^n} h\,\mathrm d\setH^m
\end{align}
with $Jf$ as defined in \cref{eq:Jacobian}. 
\end{thm}

\section{Tools from General Measure Theory}\label{sec:GMT}
This appendix summarizes the main definitions and properties in general measure theory needed in this paper. 
We follow mainly the exposition in 
\cite[Section 1]{ma99}. 

\begin{dfn}\label{dfn:measure}
Consider the metric space $(\setX,\rho)$ and let  $\colP(\setX)=\{\setA:\setA\subseteq \setX\}$. A measure 
$\mu$ is a  mapping $\mu\colon \colP(\setX)\to [0,\infty]$ satisfying 
\begin{enumerate}[itemsep=1ex,topsep=1ex]
\renewcommand{\theenumi}{(\roman{enumi})}
\renewcommand{\labelenumi}{(\roman{enumi})}
\item $\mu(\emptyset)=0$;\label{item:m1}
\item $\mu(\setE)\leq \mu(\setF)$ for all $\setE,\setF\subseteq\setX$ with $\setE\subseteq\setF$; \label{item:m2}
\item $\mu\Big(\bigcup_{i\in\naturals} \setE_i \Big)\leq \sum_{i\in\naturals} \mu(\setE_i)$ for all $\setE_1, \setE_2,\dots\subseteq\setX$.\label{item:m3}
\end{enumerate}
\end{dfn} 

\begin{dfn}
Let $(\setX,\rho)$ be a metric space equipped with a measure  $\mu$. If $\setA\subseteq\setX$ satisfies $\mu(\setX\setminus\setA)=0$, then we say that $\mu$ is supported on $\setA$. 
We set 
\begin{align}
\operatorname{spt}(\mu)=\setX\setminus \bigcup \{\setV\subseteq\setX : \setV\ \text{is open and}\ \mu(\setV)=0 \}, 
\end{align}
which is the smallest closed support set. 
\end{dfn}

\begin{dfn}
Let $(\setX,\rho)$ be a metric space equipped with the measure  $\mu$.  
A set $\setA\subseteq \setX$ is called $\mu$-measurable if 
\begin{align}
\mu(\setE)= \mu(\setE\cap \setA) + \mu(\setE\setminus\setA)\quad\text{for all $\setE\subseteq \setX$.}
\end{align}
We set $\colM(\mu)=\{\setA\subseteq \setX: \setA \ \text{is $\mu$-measurable}\}$.   
\end{dfn}

\begin{thm}\cite[Theorem 1.4]{ma99}\label{thm:measzero}
Let $(\setX,\rho)$ be a metric space equipped with the measure  $\mu$. Then, $\setM(\mu)$ is a $\sigma$-algebra containing all sets of measure zero. Moreover, $\mu$ is countably additive on $\setM(\mu)$.  
\end{thm}

\begin{dfn}\label{dfn:Borel}
Let $(\setX,\rho)$ be a metric space equipped with the measure  $\mu$. Then, $\mu$ is called
\begin{enumerate}[itemsep=1ex,topsep=1ex]
\renewcommand{\theenumi}{(\roman{enumi})}
\renewcommand{\labelenumi}{(\roman{enumi})}
\item locally finite if, for every $x\in\setX$, there exists an $r\in(0,\infty)$ such that $\mu(\{y\in\setX:d(x,y)\leq r \})<\infty$;
\item Borel if $\colM(\mu)$ contains all  Borel sets;
\item  \label{item:Borelregular} Borel regular if it is Borel and for every $\setE\subseteq\setX$, there exists a Borel set $\setB\subseteq\setX$ such that $\setE\subseteq\setB$ and $\mu(\setE)=\mu(\setB)$; 
\item a Radon measure if it is Borel and satisfies the following properties:
\begin{enumerate}
\item $\mu(\setK)<\infty$ for all compact sets $\setK\subseteq \setX$;
\item $\mu(\setV)=\sup\{\mu(\setK):\setK\subseteq\setV, \setK \ \text{is compact} \}$ for all open sets $\setV\subseteq\setX$;
\item $\mu(\setA)=\inf\{\mu(\setV):\setA\subseteq\setV, \setV \ \text{is open}\}$ for all  sets $\setA\subseteq\setX$.
\end{enumerate}
\end{enumerate}
\end{dfn}
\begin{lem}\cite[Corollary 1.11]{ma99} \label{lem:Radon}
A measure on $\reals^n$ is a Radon measure if and only if it is a locally finite Borel regular measure.
\end{lem}
We have the following properties for sums and limits of Borel regular measures. 


\begin{lem}\label{lem:sum}
Let $\mu_1$ and $\mu_2$ be Borel regular measures on the metric space $(\setX,\rho)$.  Then, $\mu_1+\mu_2$ is a Borel regular measure. 
\end{lem}
\begin{proof}
It follows immediately from \Cref{dfn:measure} that $\mu_1+\mu_2$ is a  measure. Fix a Borel set $\setB$ and a set $\setE\subseteq \setX$, both  arbitrarily. Then, we have 
\begin{align}
(\mu_1+\mu_2)(\setE)
&=\mu_1(\setE) + \mu_2(\setE)\\
&=\mu_1(\setE\cap\setB) + \mu_1(\setE\setminus\setB)+\mu_2(\setE\cap\setB) + \mu_2(\setE\setminus\setB)\\
&= (\mu_1+\mu_2)(\setE\cap\setB)+(\mu_1+\mu_2)(\setE\setminus\setB),
\end{align}
which implies that $\mu_1+\mu_2$ is a Borel measure. Finally, arbitrarily fix a set $\setA\subseteq \setX$. Then, there exist Borel sets $\setB_1,\setB_2$ satisfying 
$\setA\subseteq\setB_i$ and $\mu_i(\setA)=\mu_i(\setB_i)$ for $i=1,2$. Now, set $\setB=\setB_1\cap\setB_2$. Then, we have $\setA\subseteq \setB$ and 
$\mu_i(\setA)\leq \mu_i(\setB)\leq \mu_i(\setB_i)=\mu_i(\setA)$ for $i=1,2$, which implies $(\mu_1+\mu_2)(\setA)=(\mu_1+\mu_2)(\setB)$. This establishes that  $\mu_1+\mu_2$ is  Borel regular. 
\end{proof}
\begin{lem}\label{lem:borequalmu}
Let $\mu_1$ and $\mu_2$ be Borel regular measures on the metric space $(\setX,\rho)$ satisfying $\mu_1(\setB)=\mu_2(\setB)$ for all Borel sets $\setB\subseteq\setX$. Then, $\mu_1=\mu_2$.
\begin{proof}
Arbitrary fix  $\setA\subseteq\setX$. Then, there must exist Borel sets $\setB_k\subseteq\setX$ satisfying 
$\setA\subseteq\setB$ and $\mu_k(\setB_k)$ for $k=1,2$. Set $\setB=\setB_1\cap\setB_2$. Then, we have 
\begin{align}
\mu_k(\setA)\leq \mu_k(\setB) \leq \mu_k(\setB_k) =\mu_k(\setA) \quad\text{for $k=1,2$,}
\end{align}
which implies 
\begin{align}
\mu_1(\setA)= \mu_1(\setB) =\mu_2(\setB)=\mu_2(\setA). 
\end{align}

\end{proof}
\end{lem}

\begin{lem}\label{lem:infsummeasure}
Let $(\mu_k)_{k\in\naturals}$ be a sequence of  measures  on the metric space $(\setX,\rho)$. 
Suppose that  $\mu_k(\setE)\leq \mu_{k+1}(\setE)$ for all $k\in\naturals$ and $\setE\subseteq\setX$ and define $\mu_\infty\colon\colP(\setX)\to[0,\infty]$  according to 
\begin{align}\label{eq:defmuinfty}
\mu_\infty(\setE) =\lim_{k\to\infty}\mu_k(\setE)\quad\text{for all $\setE\subseteq\setX$.}
\end{align}
Then, $\mu_\infty$ is a  measure.  
If, in addition,  $\mu_k$ is Borel (regular) for all $k\in\naturals$, so is $\mu_\infty$. \end{lem}
\begin{proof}
We first establish that $\mu_\infty$ is a measure. 
\cref{item:m1} and \cref{item:m2} in \Cref{dfn:measure} follow immediately from \cref{eq:defmuinfty}. To establish \cref{item:m3} in \Cref{dfn:measure},  arbitrary fix a collection $\{\setE_i:i\in\naturals\}$ of sets $\setE_i\subseteq\setX$ and set $\setE=\bigcup_{i\in\naturals}\setE_i$. Since 
\begin{align}
\mu_k(\setE)\leq \sum_{\ell\in\naturals} \mu_k(\setE_\ell) \leq \sum_{\ell\in\naturals} \mu_\infty(\setE_\ell)\quad\text{ for all $k\in\naturals$,} 
\end{align}
we conclude that 
\begin{align}
\mu_\infty(\setE) = \lim_{k\to\infty}\mu_k(\setE) \leq \sum_{k\in\naturals} \mu_\infty(\setE_k).
\end{align}
This establishes  \cref{item:m3} in \Cref{dfn:measure} and hence shows that $\mu_\infty$ is a measure.

Next, suppose that $\mu_k$ is a Borel measure  for all $k\in\naturals$, and  arbitrarily fix a Borel set $\setB$ and a set $\setE\subseteq \setX$. Then, we have 
\begin{align}
\mu_\infty(\setE)
&=\lim_{k\to\infty}\mu_k(\setE) \\
&=  \lim_{k\to\infty}(\mu_k(\setE\cap\setB)  + \mu_k(\setE\setminus\setB))\\
&= \lim_{k\to\infty}\mu_k(\setE\cap\setB)  + \lim_{k\to\infty} \mu_k(\setE\setminus\setB)\\
&=\mu_\infty(\setE\cap\setB)  +  \mu_\infty(\setE\setminus\setB),
\end{align}
which proves that $\mu_\infty$ is a Borel measure. 

Finally, suppose that $\mu_k$ is Borel regular for all $k\in\naturals$ and  arbitrarily fix $\setE\subseteq\setX$. Then, for every $k\in\naturals$, there exists a Borel set $\setB_k$ satisfying $\setE\subseteq\setB_k$ and $\mu_k(\setE)=\mu_k(\setB_k)$. Set $\setB=\bigcap_{k\in\naturals}\setB_k$. Then, we have $\setE\subseteq \setB$ and 
$\mu_k(\setE)\leq \mu_k(\setB)\leq \mu_k(\setB_k)=\mu_k(\setE)$ for all $k\in\naturals$,  which implies
\begin{align}
\mu_\infty(\setE)=\lim_{k\to\infty}\mu_k(\setE) = \lim_{k\to\infty}\mu_k(\setB) = \mu_\infty(\setB).
\end{align}
Therefore, $\mu_\infty$ is Borel regular. 
\end{proof}

\begin{dfn}\label{dfn:muA}
Let $(\setX,\rho)$ be a metric space equipped with the measure  $\mu$ and let $\setC\subseteq\setX$. Then, we define the restricted measure  $\mu|_\setC \colon \colP(\setX)\to[0,\infty]$  according to 
\begin{align}
\mu|_\setC(\setE)=\mu(\setE\cap\setC) \quad\text{for all $\setE\in\colP(\setX)$}
\end{align}
and the trace measure  $\mu_\setC \colon \colP(\setC)\to[0,\infty]$  as 
\begin{align}
\mu_\setC(\setE)=\mu(\setE) \quad\text{for all $\setE\in\colP(\setC)$.}
\end{align}
\end{dfn}

\begin{thm}\cite[Theorem 1.9]{ma99}\label{thm:murest}
Let $(\setX,\rho)$, $\mu$, and $\setC$ be as in \Cref{dfn:muA}. Then, the following statements hold. 
\begin{enumerate}[itemsep=1ex,topsep=1ex]
\renewcommand{\theenumi}{(\roman{enumi})}
\renewcommand{\labelenumi}{(\roman{enumi})}
\item\label{item:murest1}
$\mu|_\setC$ is a measure and $\colM(\mu)\subseteq \colM(\mu|_\setC)$. 
\item\label{item:murest2}
If, in addition, $\mu$ is Borel regular and $\setC\in \colM(\mu)$ with $\mu(\setC)<\infty$, then $\mu|_\setC$ is Borel regular. 
\end{enumerate}
\end{thm}

\begin{cor}\label{cor:tracemeasure}
Let $(\setX,\rho)$, $\mu$, and $\setC$ be as in \Cref{dfn:muA}. Then, the following statements hold. 
\begin{enumerate}[itemsep=1ex,topsep=1ex]
\renewcommand{\theenumi}{(\roman{enumi})}
\renewcommand{\labelenumi}{(\roman{enumi})}
\item\label{item:murest1t}
$\mu_\setC$ is a measure and $\{\setA\cap \setC:\setA\in \colM(\mu)\}\subseteq \colM(\mu_\setC)$. 
\item\label{item:murest2t}
If, in addition, $\mu$ is Borel regular and $\setC\in \colM(\mu)$ with $\mu(\setC)<\infty$, then $\mu_\setC$ is Borel regular. 
\end{enumerate}
\end{cor}
\begin{proof}
Let $\setA\in \colM(\mu)$ and $\setE\subseteq \setC$. Then, \cref{item:murest1} of \Cref{thm:murest} yields 
\begin{align}
\mu_\setC(\setE)=\mu|_\setC(\setE) = \mu|_\setC(\setE\cap\setA) +  \mu|_\setC(\setE\setminus\setA)
=\mu_\setC(\setE\cap(\setA\cap\setC))+  \mu_\setC(\setE\setminus(\setA\cap\setC)),
\end{align}
which implies $\setA\cap\setC\in\colM(\mu_\setC)$ and establishes \cref{item:murest1t}. Next, suppose that $\mu$ is Borel regular and $\setC\in \colM(\mu)$ with $\mu(\setC)<\infty$. Since $\colM(\mu)$ contains all Borel sets of $\setX$, $\colM(\mu_\setC)$ contains all Borel sets of $\setC$ by \cref{item:murest1t} so that $\mu_\setC$ is a Borel measure. Finally, let $\setE\subset\setC$. Since $\mu|_\setC$ is Borel regular owing to  \cref{item:murest2} of \Cref{thm:murest}, there exists a Borel set $\setB\subseteq\setX$ such that $\setE\subseteq\setB$ and $\mu|_\setC(\setE)=\mu|_\setC(\setB)$,  which implies 
\begin{align}
\mu_\setC(\setE)= \mu(\setE)= \mu|_\setC(\setE)=\mu|_\setC(\setB)= \mu(\setB\cap\setC)= \mu_\setC(\setB\cap\setC).  
\end{align}
But $\setB\cap\setC$ is a Borel set in $\setC$ and $\setE\subseteq \setB\cap\setC$, which proves that  $\mu_\setC$ is Borel regular. 
\end{proof}

\begin{dfn}
Let $(\setX,\rho)$  and $(\setY,\sigma)$ be  metric spaces and consider the  mapping
  $f\colon\setX\to\setY$. 
\begin{enumerate}[itemsep=1ex,topsep=1ex]
\renewcommand{\theenumi}{(\roman{enumi})}
\renewcommand{\labelenumi}{(\roman{enumi})}
\item  $f$  is called Borel if $f^{-1}(\setB)$ is a Borel set in $\setX$ for all open sets $\setB$ in $\setY$;
\item if    $\mu$ is a measure on $\setX$,  
then  $f$ is called $\mu$-measurable if $f^{-1}(\setV)\in \colM(\mu)$ for all open sets $\setV\subseteq\setY$.  
\end{enumerate}
\end{dfn}

\begin{dfn}\label{dfn:mu}
Let $(\setX,\rho)$ be a metric space equipped with the measure  $\mu$ and  let $f\colon\setX\to[0,\infty]$ be $\mu$-measurable. 
We define $ f\mu\colon \colP(\setX)\to[0,\infty]$ according to 
\begin{align}\label{eq:fmu}
(f\mu)(\setE) =\inf\Bigg\{\int_{\setB}f\mathrm d \mu :\setB\in\colM(\mu), \setE\subseteq \setB \Bigg\}. 
\end{align}
\end{dfn}

\begin{lem}\label{lem:fmu}
Let $(\setX,\rho)$, $\mu$, and $f$ be as in \Cref{dfn:mu}. Then, $f\mu$ is a measure satisfying  $f\mu\ll\mu$ and $\colM(\mu)\subseteq \colM(f\mu)$. Moreover, we have 
\begin{align}\label{eq:fmuC}
(f\mu)|_\setA = f(\mu|_\setA)\quad\text{for all $\setA\in\colM(\mu)$.} 
\end{align}  
If, in addition, $\mu$ is Borel (regular), so is $f\mu$. 
\end{lem}
\begin{proof}
We start by establishing that $f\mu$ is a measure. 
\cref{item:m1}  and \cref{item:m2} in \Cref{dfn:measure} follow immediately from \cref{eq:fmu}. To establish \cref{item:m3} in \Cref{dfn:measure}, arbitrarily  fix an collection $\{\setE_i:i\in\naturals\}$ of sets $\setE_i\subseteq\setX$ and set $\setE=\bigcup_{i\in\naturals}\setE_i$. Then, we have 
\begin{align}
(f\mu)(\setE) 
&=\inf \Bigg\{\int_{\setB}f\,\mathrm d \mu :\setB\in\colM(\mu), \setE\subseteq \setB \Bigg\}\\
&\leq \inf \Bigg\{\int_{\bigcup_{i\in\naturals}\setB_i}f\,\mathrm d \mu :\setB_i\in\colM(\mu), \setE_i\subseteq \setB_i\ \text{for all $i\in\naturals$} \Bigg\}\\
&\leq \inf \Bigg\{ \sum_{i\in\naturals}\int_{\setB_i} f\,\mathrm d \mu :\setB_i\in\colM(\mu), \setE_i\subseteq \setB_i\ \text{for all $i\in\naturals$} \Bigg\}\label{eq:applybartle}\\
&= \sum_{i\in\naturals} \inf \Bigg\{ \int_{\setB_i} f\,\mathrm d \mu :\setB_i\in\colM(\mu), \setE_i\subseteq \setB_i\ \Bigg\}\\
&= \sum_{i\in\naturals}  (f\nu)(\setE_i), 
\end{align}
where in \cref{eq:applybartle} we applied \cite[Corollary 4.9]{ba95}.  

We next prove that $f\mu\ll\mu$. To this end, consider $\setA\subseteq\setX$ with  $\mu(\setA)=0$. Since this implies $\setA\in\colM(\mu)$, we have  $(f\mu)(\setA)=\int_{\setA} f\,\mathrm d \mu = \int f \ind{\setA} \,\mathrm d \mu$. 
But $f \ind{\setA}$ is identically zero on the complement of $\setA$   and  $\mu(\setA)=0$.   
We therefore conclude that $(f\mu)(\setA)=0$ thanks to \cite[Corollary 4.10]{ba95}. Next, we establish that $\colM(\mu)\subseteq \colM(f\mu)$. To this end, fix $\setA\in \colM(\mu)$ and $\setE\subseteq \setX$, both arbitrary. Then, we have 
\begin{align}\label{eq:Mmu1}
(f\mu)(\setE) \leq (f\mu)(\setE\cap\setA) + (f\mu)(\setE\setminus\setA)
\end{align}
owing to \cref{item:m3} in \Cref{dfn:measure}, and, by \cref{eq:fmu}, we obtain 
\begin{align}
(f\mu)(\setE)
& = \inf\Bigg\{\int_{\setB}f\,\mathrm d \mu :\setB\in\colM(\mu), \setE\subseteq \setB \Bigg\}\label{eq:Mmu2}\\
&= \inf\Bigg\{\int_{\setB\cap\setA}f\,\mathrm d \mu +\int_{\setB\setminus\setA}f\,\mathrm d \mu:\setB\in\colM(\mu), \setE\subseteq \setB \Bigg\}\label{eq:applybartle2}\\
&\geq \inf\Bigg\{\int_{\setB}f\,\mathrm d \mu:\setB\in\colM(\mu),  \setE\cap\setA \subseteq \setB \Bigg\}\\
&\ \ \ + \inf\Bigg\{\int_{\setB}f\,\mathrm d \mu:\setB\in\colM(\mu),  \setE\setminus\setA \subseteq \setB \Bigg\}\\
&= (f\mu)(\setE\cap\setA) + (f\mu)(\setE\setminus\setA), \label{eq:Mmu3}
\end{align}
where in \cref{eq:applybartle2} we applied \cite[Corollary 4.9]{ba95}. Combining \cref{eq:Mmu1} with \cref{eq:Mmu2}--\cref{eq:Mmu3} and using that $\setE$ was arbitrary establishes $\setA\in\setM(f\mu)$.

Now,  \cref{eq:fmuC} follows from 
\begin{align}
(f\mu)|_{\setA}(\setE)&= \inf\Bigg\{\int_{\setB}f\mathrm d \mu :\setB\in\colM(\mu), \setE\cap\setA\subseteq \setB \Bigg\}\\
&=\inf\Bigg\{\int_{\setB\cap\setA}f\mathrm d \mu :\setB\in\colM(\mu), \setE\subseteq \setB \Bigg\}\\
&=\inf\Bigg\{\int_{\setB}f\mathrm d \mu|_{\setA} :\setB\in\colM(\mu), \setE\subseteq \setB \Bigg\}\\
&=f(\mu|_\setA)(\setE)\quad\text{for all $\setE\subseteq\setX$.}
\end{align}
If $\mu$ is Borel, so is $f\mu$  since $\colM(\mu)\subseteq \colM(f\mu)$.

 Finally, suppose that $\mu$ is Borel regular and  arbitrarily fix $\setE\subseteq\setX$. Then, there exists a Borel set $\setB$ satisfying $\setE\subseteq\setB$ and $\mu(\setE)=\mu(\setB)$. Now, \cref{eq:fmu} implies $(f\mu)(\setE)\leq (f\mu)(\setB)$. 
It remains to establish  $(f\mu)(\setE)\geq (f\mu)(\setB)$. To this end, 
 arbitrarily fix $\setA\in\colM(\mu)$ such that $\setE\subseteq\setA$. Then,  we have 
\begin{align}
\mu(\setE)\leq \mu(\setB\cap\setA)\leq \mu(\setB)=\mu(\setE), 
\end{align} 
which implies $\mu(\setB\cap\setA)=\mu(\setB)$ and, in turn, using that $\setA\in\colM(\mu)$, we get 
\begin{align}
\mu(\setB\setminus\setA)=\mu(\setB)- \mu(\setB\cap\setA)=0. 
\end{align} 
 We thus have $(f\mu)(\setB\setminus\setA)=0$ thanks to $f\mu\ll\mu$ so that  
\begin{align}\label{eq:BC}
(f\mu)(\setB)=(f\mu)(\setB\setminus\setA) + (f\mu)(\setB\cap\setA)= (f\mu)(\setB\cap\setA)  \leq (f\mu)(\setA) 
\end{align}
owing to  $\setA\in\colM(f\mu)$. Since $\setA$ was  arbitrary, \cref{eq:BC} implies 
$(f\mu)(\setE)\geq (f\mu)(\setB)$. 
\end{proof}

\begin{dfn}\label{dfn:pushmu}
Let $(\setX,\rho)$ and $(\setY,\sigma)$  be  metric spaces, let  $\mu$ be a measure on $\setX$, and let  $f\colon\setX\to\setY$ be $\mu$-measurable. 
Then, we define the push forward measure $f \# \mu \colon \colP(\setY)\to [0,\infty]$ according to 
\begin{align}\label{eq:pushmu}
f \# \mu(\setE) = \mu(f^{-1}(\setE))\quad\text{for all $\setE\subseteq\setY$.}
\end{align}
\end{dfn}

\begin{lem}\label{lem:borelpush}
Let $(\setX,\rho)$,  $(\setY,\sigma)$, $\mu$,  and $f$ be as in \Cref{dfn:pushmu}. Then, $f \# \mu$ is a measure, and $\setA\in\colM(f \# \mu)$  provided that $f^{-1}(\setA)\in\colM(\mu)$. If, in addition, 	$\mu$ is a Borel measure and $f$ is Borel measurable, then $f \# \mu$ is a Borel measure. 	
\end{lem}
\begin{proof} It follows  immediately from \cref{eq:pushmu}  that $f \# \mu$ is a measure by verifying the properties in \Cref{dfn:measure}. Next, assume that  $\setA\subseteq\setY$ is such that  $f^{-1}(\setA)\in\colM(\mu)$. Then, we have 
\begin{align}
f \# \mu(\setE)
&=\mu(f^{-1}(\setE))\\
&= \mu(f^{-1}(\setE)\cap f^{-1}(\setA)) +\mu(f^{-1}(\setE)\setminus f^{-1}(\setA))\\
&= \mu(f^{-1}(\setE\cap \setA)) +\mu(f^{-1}(\setE\setminus\setA))\\
&= f \# \mu (\setE\cap \setA)+  f \# \mu (\setE\setminus \setA)\quad\text{for all $\setE\subseteq\setY$}, 
\end{align}
which implies $\setA\in\colM(f \# \mu)$. Finally, suppose that $\mu$ is a Borel measure and  arbitrarily fix a Borel set $\setB\subseteq \setY$.  
Now, suppose that   $f$ is Borel measurable and  arbitrarily fix a Borel set $\setB\subseteq\setY$. Then, $f^{-1}(\setB)$ is a Borel set and, in particular, $\mu$-measurable. Thus, $\setB\in\colM(f \# \mu)$, which establishes that $f \# \mu$ is a Borel measure. 
\end{proof}

\begin{thm}\cite[Theorem 1.18]{ma99}\label{thm:radonpush}
Let $(\setX,\rho)$ and $(\setY,\sigma)$ be  separable metric spaces. If  $f\colon\setX\to\setY$ is continuous and  $\mu$ is a Radon measure on $\setX$ with compact support $\operatorname{spt}(\mu)$, then $f \# \mu$ is a Radon measure.  Moreover, $\operatorname{spt}(f \# \mu)=f(\operatorname{spt}(\mu))$. 
\end{thm}

\begin{thm}\cite[Theorem 1.20]{ma99}\label{thm:radonpull}
Let $(\setX,\rho)$ and $(\setY,\sigma)$ be metric spaces and assume that $f\colon\setX\to\setY$ is a continuous surjection. Then, for every Radon measure $\nu$ on $\setY$, there exists a Radon measure $\mu$ on $\setX$ such that  $\nu=f \# \mu$. 
\end{thm}

\begin{cor}\label{lem:munu}
Let $\setA\subseteq\reals^m$ be compact and $f\colon \setA\to\reals^n$ Lipschitz. Further, assume that 
$\nu$ is a  Radon measure on $\reals^n$  supported on the $m$-rectifiable set $\setE=f(\setA)$. Then, 
there exists a 
Radon measure $\mu$ on $\setA$  such that $\nu=f\#\mu$. 
\end{cor}
\begin{proof}
Set $\setE=f(\setA)$. Now, $\nu$ is a   locally finite Borel regular measure by \Cref{lem:Radon}. 
Therefore, since $\setE$ is compact, we have $\nu(\setE)<\infty$. 
Hence, by \cref{item:murest2t} in  \Cref{cor:tracemeasure}, 
$\nu_\setE$ is Borel regular. 
Using again \Cref{lem:Radon} therefore   establishes that 
$\nu_\setE$ is a Radon measure. 
Therefore, 
 by \Cref{thm:radonpull}, there exists a Radon measure $\mu$ on $\setA$ such that $\nu_\setE =f\# \mu$. Finally, we have 
\begin{align}
\nu(\setB)&=\nu(\setB\cap\setE)= \nu_\setE(\setB\cap\setE)= \mu (f^{-1}(\setB\cap\setE))\\
&= \mu (f^{-1}(\setB))= (f\#\mu)(\setB)
\end{align} 
for all Borel sets $\setB\subseteq\reals^n$, which establishes $\nu=f\#\mu$ thanks to \Cref{lem:borequalmu}. 
\end{proof}

\begin{dfn}
Let $(\setX,\rho)$ be a metric space.  For every $s\in [0,\infty)$, the $s$-dimensional Hausdorff measure $\colH^s\colon \colP(\setX)\to [0,\infty]$ is defined by 
\begin{align}
\colH^s(\setE) =\lim_{\delta\to 0} \colH_\delta^s(\setE)
\end{align}
with 
\begin{align}
\colH_\delta^s(\setE) =\frac{\pi^{s/2}}{2^s\Gamma(1+s/2)} \inf\Big\{ \sum_{i\in\naturals}\diam^s(\setE_i):\setE\subseteq\bigcup_{i\in\naturals}\setF_i, \diam(\setE_i)\leq \delta \Big\}
\end{align}
for all $\setE\subseteq\setX$ and $\delta\in(0,\infty)$.
\end{dfn}
The Hausdorff measures $\colH^s$ are Borel regular \cite[Corollary 4.5]{ma99}. 
\begin{lem}
Let $(\setX,\rho)$ be a metric space equipped with the Hausdorff measure $\colH^s$ and assume that $\setA\subseteq \setX$ is $\colH^s$-measurable with $\colH^s(\setA)<\infty$.  Then, $\colH^s|_{\setA}$ is a Radon measure. 
\end{lem}
\begin{proof}
$\colH^s|_{\setA}$ is Borel regular by \cref{item:murest2} in \Cref{thm:murest} and, therefore, a  Radon measure 
owing to \Cref{lem:Radon}. 
\end{proof}

\begin{lem} \cite[Proposition 2.49 and Theorem 2.53]{amfupa00}\cite[Lemma 3.2]{evga14}\label{lem:Hmeasure}
The Hausdorff measures $\colH^s$ on $\reals^n$ have the following properties: 
\begin{enumerate}[itemsep=1ex,topsep=1ex]
\renewcommand{\theenumi}{(\roman{enumi})}
\renewcommand{\labelenumi}{(\roman{enumi})}
\item For every $s_1,s_2\in[0,\infty)$ with $s_1<s_2$ and $\setE\subseteq \reals^n$, $\colH^{s_2}(\setE)>0$ implies $\colH^{s_1}(\setE)=\infty$;
\item  \label{eq:itemHLL} If $f\colon \reals^m\to\reals^n$ is Lipschitz, then 
\begin{align}
\colH^{s}(f(\setE))\leq \operatorname{Lip}(f)^s\colH^{s}(\setE)\quad\text{for all $s\in[0,\infty)$ and $\setE\subseteq\reals^m$;}
\end{align}
\item \label{eq:itemHL}If $f\colon \reals^m\to\reals^n$ is Lipschitz and $\setA\subseteq\reals^m$ is Lebesgue measurable, then $f(\setA)\in \colM(\colH^{m})$; 
\item $\colH^{n}(\setB)=\colL^{n}(\setB)$ for all Borel sets $\setB\subseteq\reals^n$, where $\colL^{n}$ denotes Lebesgue measure.\label{eq:itemHLeb}  
\end{enumerate}
\end{lem}

\begin{dfn}\cite[Definition 2.5.1]{krpa08}\label{dfn:HD}
Let $(\setX,\rho)$ be a metric space. The Hausdorff dimension of the set $\setE\subseteq \setX$ is defined as
\begin{align}
\dim_{\mathrm H}(\setE) = \sup\{s\in[0,\infty]:\colH^s(\setE)>\infty\}. 
\end{align}
If $\mu$ is a measure on $\setX$, then we set $\dim_{\mathrm H}(\mu)=\dim_{\mathrm H}(\operatorname{spt}(\mu))$. 
\end{dfn}

\begin{lem}\cite[Definition 2.5.1]{krpa08}\label{lem:HD}
Let $(\setX,\rho)$ be a metric space and let $\setE\subseteq \setX$. Then, the Hausdorff dimension  satisfies 
\begin{align}
\dim_{\mathrm H}(\setE)  
&= \sup\{s\in[0,\infty]:\colH^s(\setE)=\infty\}\\
& =\inf\{s\in[0,\infty]:\colH^s(\setE)<\infty\}\\
&=\inf\{s\in[0,\infty]:\colH^s(\setE)=0\}. 
\end{align}
\end{lem}

\section{Properties of Wasserstein Distance}\label{sec:WD}
This section summarizes the main definitions and properties of Wasserstein distance needed in the paper.  We mainly  follow  the exposition in \cite[Chapter 7]{amgisa08}. 
\begin{dfn}\label{dfn:Wasserstein}
Let $(\setX,\rho)$ be a  separable complete metric space. Denote by $\setP(\setX)$ the set of all Borel probability measures  $\setX$.  
A coupling $\pi$ between $\mu,\nu\in \setP(\setX)$ is a Borel probability measure on $\setX\times\setX$ satisfying 
\begin{align}
\pi(\setA,\setX)&=\mu(\setA)\quad\text{for all $\setA\subseteq\setX$}\\
\pi(\setX,\setB)&=\nu(\setB)\quad\, \text{for all $\setB\subseteq\setX$.}
\end{align}
The set of all couplings between $\mu,\nu\in \setP(\setX)$ is denoted by $\Pi(\mu,\nu)$. 
For every $p\in[1,\infty)$, we set 
\begin{align}
\setP_p(\setX) =\Bigg\{ \mu\in\setP(\setX):\ \exists x_0\in\setX\ \text{with}\ \int \rho^{\,p}(x_0,\cdot)  \mathrm d \mu <\infty\Bigg\}. 
\end{align}
For every  $p\in[1,\infty)$ and  $\mu,\nu\in \setP_p(\setX)$, the $p$-Wasserstein distance  between $\mu$ and $\nu$ 
is 
\begin{align}\label{eq:Wp}
 W_p(\mu,\nu)= \Bigg( \inf_{\pi \in\Pi(\mu,\nu)}\int \rho^{\,p} \,\mathrm d \pi\Bigg)^{1/p}.    
\end{align}
Finally, if $\mu,\nu$ are finite measures with $\alpha:=\mu(\setX)=\nu(\setX)$, we set 
\begin{align}\label{eq:Wp2}
 W_p(\mu,\nu)= \alpha^{1/p}W_p(\mu/\alpha,\nu/\alpha). 
 \end{align}
\end{dfn}

\begin{lem}\label{lem:boundWboundedS}
Let $(\setX,\rho)$ be a  separable complete metric space and assume that  $\setS\subseteq \setX$ is bounded. 
Then, we have $ W_p(\mu,\nu)\leq \diam(\setS)$ for all $\mu,\nu\in \setP_p(\setX)$ satisfying $\operatorname{spt}(\mu)\subseteq \setS$ and $\operatorname{spt}(\nu)\subseteq \setS$. 
\begin{proof}
Arbitrarily fix $\pi\in \Pi(\mu,\nu)$. Then, $\pi|_{\setS\times\setS}$ is also a coupling, and we have 
\begin{align}
W^p_p(\mu,\nu) 
\leq   \int \rho^{\,p} \,\mathrm d \pi|_{\setS\times\setS}
\leq   \diam^p(\setS). 
\end{align}
\end{proof}
\end{lem}

We have the following relation between $p$-Wasserstein distances. 
\begin{lem}\label{lem:equiv}
Let $(\setX,\rho)$ be a separable complete metric space and let   $p,q\in[1,\infty)$ with $q>p$. 
Further, assume that $\operatorname{spt}(\mu)\subseteq \setS$ and $\operatorname{spt}(\nu)\subseteq \setS$. 
Then, we have 
\begin{align}\label{eq:Wequiv}
W_p( \mu, \nu)\ \leq\ \, W_q(\mu,\nu)\ \leq\ \diam^{1-p/q}(\setS) W^{p/q}(\mu,\nu)\quad\text{for all $\mu,\nu\in  \setP_p(\setX)$}. 
\end{align}
\end{lem}
\begin{proof}
 Jensen's inequality \cite[Theorem 2.3]{peprto92} applied to  the convex function $\phi\colon(0,\infty)\to (0,\infty)$,  $t\to t^{q/p}$ yields  $W_q( \mu, \nu)\geq W_p( \mu, \nu)$
 and 
 \begin{align}
 W^q_q( \mu, \nu) &\leq \diam^{q-p}(\setS) \inf_{\pi\in\Pi( \mu, \nu)} \int   \rho^{\,p} \mathrm d \pi|_{\setS\times\setS}\\
 & = \diam^{q-p}(\setS)\,W^p_p( \mu, \nu). 
 \end{align}
\end{proof}

\Cref{lem:equiv} implies that for bounded separable complete metric spaces,  $p$-Wasserstein distances are  comparable in the sense of  \cref{eq:Wequiv}. For the  case  $p=1$, we have the following Kantorovich-Rubinstein dual characterization. 

\begin{lem}\cite[Equation (7.1.2)]{amgisa08} \label{thm:dualW1}
Let $(\setX,\rho)$ be a separable complete metric space.    Then, we have 
\begin{align}
W_1(\mu,\nu) = \sup_{\substack{\psi\colon\setX\to\reals\\ \operatorname{Lip}(\psi)\leq 1}}
\Bigg\{\int \psi \,\mathrm d \mu - \int \psi \,\mathrm d \nu\Bigg\}\quad\text{for all $\mu,\nu\in \setP_1(\setX)$.} 
\end{align}
\end{lem}

We need the following properties of  $1$-Wasserstein distance under the push forwards of a measure. 

\begin{lem}\label{lem:pushW1}
Let $(\setX,\rho)$ and $(\setY,\sigma)$ be  separable complete metric spaces and let $p\in[1,\infty)$ and $\alpha\in(0,\infty)$. Further, let $\mu\in\setP_p(\setX)$ and suppose that 
 $f,g\colon \setX\to\setY$ are both $\mu$-measurable.   Then, we have 
 \begin{align}
W_1(f\#(\alpha\mu),g\#(\alpha\mu)) \leq \alpha \int \sigma( f,g)\, \mathrm d\mu.
\end{align}
\end{lem}
\begin{proof}
Since 
\begin{align}
W_1(f\#(\alpha\mu),g\#(\alpha\mu)) = W_1(\alpha(f\#\mu),\alpha (g\#\mu))= \alpha W_1(f\#\mu,g\#\mu)
\end{align}
and 
\begin{align}
\int \sigma( f,g)\, \mathrm d(\alpha \mu) = \alpha \int \sigma( f,g)\, \mathrm d\mu 
\end{align}
we can  assume, without loss of generality, that $\alpha=1$. 
We have 
\begin{align}
W_1(f\#\mu,g\#\mu) 
&= \sup_{\substack{\psi\colon\setY\to\reals\\ \operatorname{Lip}(\psi)\leq 1}}\Bigg\{\int \psi \,\mathrm d (f\#\mu) - \int \psi \,\mathrm d (g\#\mu)\Bigg\}\label{eq:dualW1}\\
&=\sup_{\substack{\psi\colon\setY\to\reals\\ \operatorname{Lip}(\psi)\leq 1}}\Bigg\{\int \psi\circ f \,\mathrm d \mu - \int \psi\circ g \,\mathrm d \mu\Bigg\}\\
&\leq\sup_{\substack{\psi\colon\setY\to\reals\\ \operatorname{Lip}(\psi)\leq 1}}\Bigg\{\int \abs{\psi\circ f -\psi\circ g}\,\mathrm d \mu \\
&\leq \int \sigma( f,  g)\,\mathrm d \mu,  \label{eq:dualW2}
\end{align}
where \cref{eq:dualW1} follows by application of \Cref{thm:dualW1}. 
\end{proof}

\begin{lem}\label{lem:pushW2}
Let $(\setX,\rho)$ and $(\setY,\sigma)$ be a separable complete metric spaces and let  $\alpha\in(0,\infty)$. Further, let $\mu,\nu\in\setP_1(\setX)$ be Borel measures and  assume that 
 $f\colon \setX\to\setY$ is Lipschitz. Then, 
 \begin{align}\label{eq:towhowpushW2}
 W_1(f\#(\alpha\mu),f\#(\alpha\nu)) \leq \operatorname{Lip}(f) W_1(\alpha\mu,\alpha\nu). 
 \end{align}
\end{lem}
\begin{proof}
Since
\begin{align}
W_1(f\#(\alpha\mu),f\#(\alpha\nu)) =W_1(\alpha (f\#\mu),\alpha (f\#\nu)) = \alpha W_1(f\#\mu,f\#\nu)
\end{align}
and 
\begin{align}
 W_1(\alpha\mu,\alpha\nu)=\alpha W_1(\mu,\nu),
\end{align}
we can assume, without loss of generality, that $\alpha =1$. Next, define $g\colon\setX\times\setX\to\setY\to\setY$ according to 
\begin{align}
g(x,y) = (f(x),f(y)). 
\end{align}
Now, arbitrarily fix  $\pi\in \Pi(\mu,\nu)$. Then, we have 
\begin{align}
(g\# \pi)(\setA\times \setY)= \pi(f^{-1}(\setA)\times\setX) = \mu(f^{-1}(\setA)) =(f\#\mu)(\setA) 
\end{align}
and 
\begin{align}
(g\# \pi)( \setY\times \setA)= \pi(\setX\times f^{-1}(\setA)) = \nu(f^{-1}(\setA)) =(f\#\nu)(\setA),  
\end{align}
which implies  $g\# \pi\in \Pi(f\#\mu,f\#\nu)$ so that 
\begin{align}
W_1(f\#\mu,f\#\nu) 
&\leq \int \sigma\, \mathrm d (g\# \pi)\\
&=\int \sigma(f,f)\, \mathrm d  \pi\\
&\leq \operatorname{Lip}(f)  \int \rho\, \mathrm d  \pi, 
\end{align}
which establishes the desired result as $\pi$ was  arbitrary. 
\end{proof}

Finally, we need the following result for sequences of measures. 

\begin{lem}\label{eq:lemWrest}
Let $(\setX,\rho)$ be a separable complete metric space satisfying $\sup_{x,y\in\setX} \rho(x,y) <\infty$. Further,  let $(\mu_i)_{i\in\naturals}$ and $(\nu_i)_{i\in\naturals}$ be sequences of Borel regular probability  measures on $\setX$ and   let $(a_i)_{i\in\naturals}$ be a sequence in $[0,1]$ satisfying $\sum_{i\in\naturals}a_i=1$. Then, the measures 
\begin{align}\label{eq:musum}
\mu=\sum_{i\in\naturals}a_i\mu_i
\end{align}
and 
\begin{align}\label{eq:nusum}
\nu=\sum_{i\in\naturals}a_i\nu_i
\end{align}
are Borel regular, in $\setP_1(\setX)$,  and satisfy 
\begin{align}
W_1(\mu,\nu) \leq \sum_{i\in\naturals} a_i W_1(\mu_i,\nu_i).
\end{align}
\end{lem}
\begin{proof}
Set $C=\sup_{x,y\in\setX} \rho(x,y)$ and note that 
 $ C<\infty$, which is by assumption,  implies   $\mu_i, \nu_i\in \setP_1(\setX)$ for all $i\in\naturals$.  
Now, \Cref{lem:infsummeasure}  implies that $\mu$ and $\nu$ are both Borel regular measures. 
Moreover, $\mu,\nu\in \setP_1(\setX)$ since
\begin{align}
\mu(\setX) = \sum_{i\in\naturals}a_i \mu_i(\setX) =1 
\end{align}
and, for fixed $x_0\in\setX$, we have  
\begin{align}
\int \rho(x,x_0) \,\mathrm d\mu \leq  C \mu(\setX) =C 
\end{align}
and likewise for $\nu$. Now, for every $i\in\naturals$,  arbitrarily fix  $\pi_i\in \Pi(\mu_i,\nu_i)$. Then, again by \Cref{lem:infsummeasure}, 
\begin{align}\label{eq:pisum}
\pi=\sum_{i\in\naturals}a_i\pi_i
\end{align}
is a Borel regular measure, where  convergence in \cref{eq:pisum} is setwise.  
Moreover, we have 
\begin{align}
\pi (\setA\times \setX)= \sum_{i\in\naturals}a_i\pi_i(\setA\times \setX) = \sum_{i\in\naturals}a_i\mu_i(\setA) =\mu(\setA)
\end{align}
and 
\begin{align}
\pi ( \setX\times \setA)= \sum_{i\in\naturals}a_i\pi_i( \setX\times \setA) = \sum_{i\in\naturals}a_i\nu_i(\setA) =\nu(\setA),  
\end{align}
which implies  $\pi\in \Pi(\mu,\nu)$.  Since the coupling $\pi_i$ was  arbitrary, we conclude that 
\begin{align}
W_1(\mu,\nu) 
&\leq \inf_{\substack{\pi_i \in\Pi(\mu_i,\nu_i)\\ \text{$i\in\naturals$}}}\int \rho \,\mathrm d \Big(\sum_{i\in\naturals} a_i \pi_i\Big)\\
&= \sum_{i\in\naturals}  a_i \inf_{\pi_i \in\Pi(\mu_i,\nu_i)}\int \rho \,\mathrm d  \pi_i \label{eq:convergencemeasure}\\
&= \sum_{i\in\naturals} a_i W_1(\mu_i,\nu_i), 
\end{align}
where   \cref{eq:convergencemeasure} follows from \cref{eq:pisum} and approximation of $\rho$ through simple functions. 
\end{proof}

\section{Auxiliary Results}\label{sec:aux}

 \begin{lem}\label{lem:equivMdim}
 Let $(X,d)$ be a  metric space and set 
 \begin{align}
 M_1(\setX,d):=\inf\{s\in(0,\infty): \liminf_{\varepsilon\to 0 }H_\varepsilon(\setX,d)\varepsilon^s<\infty\}
 \end{align}
 and
 \begin{align}
M_2(\setX,d)=\inf\{s\in(0,\infty): \liminf_{\varepsilon\to 0 }e^{H_\varepsilon(\setX,d)}e^{-\varepsilon^{-s}}<\infty\}.
 \end{align}
Then, we have $M_1(\setX,d)= M_2(\setX,d)$.  
 \end{lem}
\begin{proof}
We first establish $M_1(\setX,d)\geq M_2(\setX,d)$. 
Toward a contradiction, suppose that $M_1(\setX,d)< M_2(\setX,d)$ and pick $s,t\in (M_1(\setX,d),M_2(\setX,d))$ with $s<t$. 
Since  $s > M_1(\setX,d)$, there must  exist a sequence $(\varepsilon_n)_{n\in\naturals}$ in $(0,1]$ with $\lim_{n\to\infty}\varepsilon_n=0$ and a $C\in(0,\infty)$ such that 
\begin{align}
H_{\varepsilon_n}(\setX,d)    \leq C\varepsilon_n^{-s}\quad\text{for all $n\in\naturals$.}
\end{align}
Since $t>s$, this implies that there exists an $N\in\naturals$ such that 
\begin{align}
e^{ H_{\varepsilon_n}(\setX,d)}e^{-\varepsilon^{-t}} \leq e^{C\varepsilon_n^{-s}}e^{-\varepsilon_n^{-t}}\leq 1\quad\text{for all $n\geq N$,}
 \end{align}
 which yields 
 \begin{align}
 \liminf_{\varepsilon\to 0 }e^{ H_{\varepsilon_n}(\setX,d)}e^{-\varepsilon^{-t}}<\infty
 \end{align}
 and, in turn, $M_2(\setX,d)\leq t$. 
This  stands  in contradiction to $t< M_2(\setX,d)$. 

Next, we prove $M_2(\setX,d)\geq M_1(\setX,d)$. Toward a contradiction, suppose that $M_2(\setX,d)< M_1(\setX,d)$ and pick $s,t\in (M_2(\setX,d),M_1(\setX,d))$ with $s<t$. Since  $s > M_2(\setX,d)$, there must  exist a sequence $(\varepsilon_n)_{n\in\naturals}$ in $(0,1]$ with $\lim_{n\to\infty}\varepsilon_n=0$ and a $C\in(0,\infty)$ such that 
\begin{align}
e^{H_{\varepsilon_n}(\setX,d)}    \leq Ce^{\varepsilon_n^{-s}}\quad\text{for all $n\in\naturals$.}
\end{align}
Since $t>s$, this implies that there exists an $N\in\naturals$ such that 
\begin{align}
H_{\varepsilon_n}(\setX,d)\varepsilon^{t}\leq  \Big(\frac{\log(C)}{\log(e)}+\varepsilon_n^{-s} \Big)\varepsilon_n^{t} \leq 1\quad\text{for all $n\geq N$,}
 \end{align}
 which yields 
 \begin{align}
 \liminf_{\varepsilon\to 0 }H_{\varepsilon_n}(\setX,d)\varepsilon^{t}<\infty
 \end{align}
 and, in turn, $M_1(\setX,d)\leq t$. 
This  stands  in contradiction to $t< M_1(\setX,d)$. 

\end{proof}

\begin{lem}\label{lem:mhatm}
Let $n,N\in\naturals$ with  $N\geq n$, set $\delta=1/N$, and let  $m_1,\dots,m_n\in[0,1]$ with $\sum_{k=1}^nm_k=1$. Then, there exist $\hat w_1,\dots,\hat w_n\in \delta\naturals$ satisfying  $\sum_{k=1}^n \hat w_k=1$ and 
\begin{align}\label{eq:tildemxxx}
\sum_{k=1}^n \abs{\hat w_k - w_k}\leq  4 \delta (n-1). 
\end{align} 
\end{lem}
\begin{proof}
Set $\hat p_k =\delta \lfloor w_k/\delta \rfloor$ for $k=1,\dots,n-1$ and $\hat p_n=1-\sum_{k=1}^{n-1} \hat p_k$. By construction, we have $\hat p_k\in\delta\naturals_0$ with $\sum_{k=1}^n \hat p_k =1$.  Moreover, the $\hat p_k$ satisfy 
\begin{align}\label{eq:hatm}
\sum_{k=1}^n \abs{\hat p_k - w_k}
&=  \abs{\hat p_n - w_n}+\sum_{k=1}^{n-1} \abs{\hat p_k - w_k}\\
&= \abs{\sum_{k=1}^{n-1} (\hat p_k -w_k) }+\sum_{k=1}^{n-1} \abs{\hat p_k - w_k}\\
&\leq2\sum_{k=1}^{n-1} \abs{\hat p_k - w_k}\\
&= 2\delta\sum_{k=1}^{n-1} \abs{  \lfloor w_k/\delta \rfloor  - w_k/\delta}\\
&\leq 2\delta (n-1). \label{eq:hatm2}
\end{align}
Next, set 
\begin{align}
\setC=\{k\in\{1,\dots,n\}: \hat p_k=0\}, 
&&
\setC^\bot=\{1,\dots,n\}\setminus\setC,
\end{align}
 $\alpha=\abs{\setC}\delta$, and $\beta=\abs{\setC^\bot}\delta$.  
Now, note that 
\begin{align}
\sum_{k\in\setC^\bot} \hat p_k =1 \geq n\delta =\alpha+\beta. 
\end{align}
We can therefore subtract a total mass $\alpha$ from the weights $\{\hat p_k: k\in \setC^\bot\}$ in such a way that they  still have positive mass. Concretely,  
we can find $\alpha_1,\alpha_2\dots,\alpha_{\abs{\setC^\bot}}\in \delta\naturals_0$ satisfying 
\begin{align}\label{eq:Cset0}
\sum_{k\in\setC^\bot}\alpha_k=\alpha
\end{align}
such that 
\begin{align}\label{eq:Cset2}
\hat w_k := \hat p_k-\alpha_k\in \delta \naturals\quad\text{for all $k\in\setC^\bot$.}
\end{align}
Setting 
\begin{align}\label{eq:Cset1}
\hat w_k=\delta \quad\text{ for all $k\in \setC$,} 
\end{align} 
we have $\hat w_k\in \delta\naturals$ for $k=1,\dots,n$ by construction. Moreover, 
\begin{align}
\sum_{k=1}^n\hat w_k
&= \sum_{k\in\setC}^n\hat w_k+ \sum_{k\in\setC^\bot}^n\hat w_k\\
&=  \alpha + \sum_{k\in\setC^\bot}^n(\hat p_k-\alpha_k)\\
&= \sum_{k\in\setC^\bot}^n\hat p_k\\
&=1
\end{align}
 and 
\begin{align}\label{eq:tildem}
\sum_{k=1}^n \abs{\hat w_k - w_k} 
&\leq \sum_{k=1}^n \abs{\hat w_k - \hat p_k}+  \sum_{k=1}^n \abs{\hat p_k -  w_k}\\
&\leq 2\delta (n-1)  +2\abs{\setC}\delta\label{eq:usehatm}\\
&\leq 4\delta (n-1), 
\end{align}
where in  \cref{eq:usehatm} we used \cref{eq:hatm}--\cref{eq:hatm2} and  \cref{eq:Cset0}--\cref{eq:Cset1}.
\end{proof}
\begin{lem}\label{lem:piecewiselinear}
For $K\in\naturals$, let $w_1,\dots,w_K\in(0,1]$ with $\sum_{k=1}^Kw_k=1$ and set $a_0=1/(Kw_1)$ and  $b_0=0$. 
For $k=1,\dots,K$, set  $a_k=1/(Kw_{k+1})-1/(Kw_{k})$, $b_{k}=\sum_{j=1}^{k}w_{j}$, and
\begin{align}
\setK_{k}=
\begin{cases} 
[b_{k-1}, b_{k})& \text{if $k<K$}\\
[b_{k-1}, \infty)& \text{if $k=K$}.  
\end{cases}
\end{align}
Define $f\colon \reals\to \reals$ according to 
\begin{align}
f(x)=
\begin{cases}
x&\text{for all $x\in (-\infty,0)$}\\
(x-b_k)/(Kw_k) +{k}/{K}& \text{for all $x\in \setK_k$ and $k=1,\dots, K$}
\end{cases}
\end{align}
and $g\colon \reals\to \reals$ as 
\begin{align}
g(x)= -\rho(-x)+\sum_{k=1}^{K} \rho(x-b_{k-1})a_{k-1}, 
\end{align}
where $\rho$ is as per \Cref{dfn:ReLUrho}. 
Then, $f=g$.  
\end{lem}
\begin{proof}
We have 
\begin{align}
g(x)=-\rho(-x) = x =f(x)\quad\text{for all $x\in (-\infty,0)$}   
\end{align}
Next, we establish that  $f(x)=g(x)$ for all $x\in[0,\infty)$. 
To this end,  arbitrarily fix $k\in\{1,\dots,K\}$ and note that 
\begin{align}
g(x)=  x \sum_{j=1}^{k}a_{j-1} - \sum_{j=1}^{k} a_{j-1} b_{j-1}\quad\text{for all $x\in\setK_k$.}\label{eq:inwk}
\end{align}
Now, 
\begin{align}
\sum_{j=1}^{k}a_{j-1} =1/(Kw_k)\label{eq:use1wk}
\end{align}
and 
\begin{align}
&K\sum_{j=1}^{k}a_{j-1}b_{j-1}\label{eq:use1wka} \\
&=w_1(1/w_2-1/w_1  )+(w_1+w_2)(1/w_3-1/w_2  )\\
&\ \ \ +\dots+ (w_1+w_2+\dots+w_{k-1})(1/w_k-1/w_{k-1})\\
&= 1-k +(w_1/w_2+(w_1+w_2)/w_3+\dots+(w_1+\dots +w_{k-1})/w_k)\\
&\ \ \ - (w_1/w_2+(w_1+w_2)/w_3+\dots+(w_1+\dots +w_{k-2})/w_{k-1})\\
&= 1-k + b_{k-1}/(w_k).\label{eq:use2wk}
\end{align}
Using \cref{eq:use1wk} and \cref{eq:use1wka}--\cref{eq:use2wk} in \cref{eq:inwk} yields 
\begin{align}
g(x)&= (x- b_{k-1})/(Kw_k)+(k-1)/K\\
     & =(x- b_{k})/(Kw_k)+k/K \\
     &=f(x) \quad\text{for all $x\in\setK_k$}.  
\end{align}
\end{proof}

\section{Properties of the Sawtooth Function}\label{sec:sawtooth}
We start with the definition of the sawtooth  function.

\begin{dfn}\label{dfn:g}
The sawtooth function $g\colon \reals \to [0,1]$ is defined as  
\begin{align}
g(x)=
\begin{cases}
2x &\text{if $x\in[0,1/2)$ }\\
2(1-x) &\text{if $x\in[1/2,1]$}\\
0&\text{else.} 
\end{cases}
\end{align}
For every $s\in\naturals$, we set 
\begin{align}\label{eq:gs}
g_s= \underbrace{g\circ g\circ \dots \circ g}_{\text{$s$ times}}. 
\end{align}
\end{dfn} 

We next state  elementary properties of $g_s$ needed in the paper. 

\begin{lem}\label{lem:sawtooth}
For every $s\in\naturals$, we  have $g_s(x)=0$ for all $x\in\reals\setminus (0,1)$ and 
\begin{align}\label{eq:sawtooth}
g_s(x)= \begin{cases}
 2^sx - \big\lfloor 2^sx\big\rfloor   &\text{if $\lfloor 2^s x\rfloor$ is even}\\
  1-2^sx + \big\lfloor 2^sx\big\rfloor &\text{if $\lfloor 2^s x\rfloor$ is odd}
\end{cases}\quad \text{for all $x\in [0,1]$.}
\end{align}
\end{lem}
\begin{proof}
Since $g_s(x)=0$ whenever $g(x)=0$, we conclude that $g_s(x)=0$ for all $x\in\reals\setminus [0,1]$.  
We prove  \cref{eq:sawtooth} by induction. 
For $s=1$, \cref{eq:sawtooth} follows from $\lfloor 2x\rfloor =0$ for all $x\in[0,1/2)$ and  $\lfloor 2x\rfloor =1$ for all  $x\in[1/2,1]$   together with the definition of $g=g_1$. Now, assume that \cref{eq:sawtooth} holds for $s-1$ and set $\setI_k=[(k-1)/2^s,k/2^s)$  for $k\in\naturals$. Then, for every $m\in\naturals$, we have 
\begin{align}
x\in \setI_{4m} & \Rightarrow \lfloor 2^{s-1}x\rfloor = 2m-1, \lfloor 2^sx\rfloor = 4m-1,\ \text{and $g_{s-1}(x) \leq 1/2$}\\
x\in \setI_{4m-1} & \Rightarrow \lfloor 2^{s-1}x\rfloor = 2m-1, \lfloor 2^sx\rfloor = 4m-2,\ \text{and $g_{s-1}(x) > 1/2$}\\
x\in \setI_{4m-2} & \Rightarrow \lfloor 2^{s-1}x\rfloor = 2m-2, \lfloor 2^sx\rfloor = 4m-3,\ \text{and $g_{s-1}(x) \geq 1/2$} \\
x\in \setI_{4m-3} & \Rightarrow \lfloor 2^{s-1}x\rfloor = 2m-2, \lfloor 2^sx\rfloor = 4m-4,\ \text{and $g_{s-1}(x) < 1/2$,}  
\end{align}
which implies
\begin{align}
g_s(x)= 
\begin{cases}
2 g_{s-1}(x) = 4m-2^sx  = 1 -2^sx  + \lfloor 2^sx\rfloor &\text{if $x\in \setI_{4m}$} \\
2- 2g_{s-1}(x)= 2^sx -4m+2 = 2^sx - \lfloor 2^sx\rfloor &\text{if $x\in \setI_{4m-1}$} \\
2- 2g_{s-1}(x)=  2^sx -4m+4 = 1 -2^sx  + \lfloor 2^sx\rfloor  &\text{if $x\in \setI_{4m-2}$} \\
2g_{s-1}(x)=2^sx -4m+4=  2^sx - \lfloor 2^sx\rfloor &\text{if $x\in \setI_{4m-4}$.} 
\end{cases}
\end{align}
\end{proof}
\begin{lem}\label{lem:sawtooth1}
For every $s\in\naturals$, we  have 
\begin{align}\label{eq:sawtooth2}
g_s=\sum_{k=1} ^{2^{s-1}} h_k
\end{align}
with $h_k\colon \reals\to [0,1]$, $h_k(x)=g(2^{s-1}x-k+1)$ for $k=1,\dots,2^{s-1}$.  The mappings $h_k$
 satisfy $h_k(x)=0$  for all $x\notin \setF_k:=((k-1)/2^{s-1}, k/2^{s-1})$. 
Moreover, for every  mapping $f\colon [0,1]\to \reals$ satisfying $f(0)=0$, we have 
\begin{align}\label{eq:sawtooth3}
f\circ g_s = \sum_{k=1} ^{2^{s-1}} f\circ h_k. 
\end{align}
\end{lem}
\begin{proof}
It follows immediately from \Cref{dfn:g} that $h_k(x)=0$  for all $x\notin \setF_k$. 
Moreover, by \Cref{lem:sawtooth}, we have 
\begin{align}
g_s(k/2^{s-1})= 2k-\lfloor 2k\rfloor =0\quad\text{ for  $k=0,\dots,2^{s-1}$.}  
\end{align}
Since 
\begin{align}
[0,1]=\Bigg(\bigcup_{k=1}^{2^{s-1}}\setF_k\Bigg) \cup \Bigg(\bigcup_{k=0}^{2^{s-1}}\{k/2^{s-1}\}\Bigg),
\end{align}
it hence   suffice to establish $h_k(x)=g_s(x)$ for all $x\in\setF_k$ and $k=1,2,\dots, 2^{s-1}$. Now,  arbitrarily fix $k\in\{1,2,\dots, 2^{s-1}\}$ and split up $\setF_k$ according to  $\setF_k^{(1)}=((2k-2)/2^{s}, (2k-1)/2^s)$ and 
$\setF_k^{(2)}=[(2k-1)/2^s, 2k/2^{s})$ so that $\setF_k=\setF_k^{(1)}\cup \setF_k^{(2)}$. since 
$\lfloor 2^s x\rfloor= {2k-2}$ for all $x\in \setF_k^{(1)}$ and 
 $\lfloor 2^s x\rfloor= {2k-1}$ for all $x\in \setF_k^{(2)}$, we have 
\begin{align}
h_k(x)=
\begin{cases}
2^{s}x-2k+2 =2^{s}x - \lfloor 2^s x\rfloor &\text{for all $x\in \setF_k^{(1)}$}\\
2k -2^{s}x  =1- 2^{s}x + \lfloor 2^s x\rfloor  &\text{for all $x\in \setF_k^{(2)}$, }
\end{cases}
\end{align} 
which implies $h_k(x)=g_s(x)$ for all $x\in\setF_k$  owing to \Cref{lem:sawtooth}. 
To establish \cref{eq:sawtooth3},  arbitrarily fix $k\in\{1,2,\dots, 2^{s-1}\}$ . Then, 
we have 
\begin{align}
f(g_s(x)) 
&=f\Bigg( \sum_{\ell=1} ^{2^{s-1}}h_\ell(x)\Bigg)\label{eq:fgs1}\\
&= f(h_k(x))\label{eq:fgs2} \\
&= \sum_{k=1} ^{2^{s-1}} f(h_k(x))\quad\text{for all $x\in\setF_k$,} \label{eq:fgs3}
\end{align}
where \cref{eq:fgs1} follows from \cref{eq:sawtooth2},  \cref{eq:fgs2}  is by  the fact that $h_k(x)=0$ for all $x\notin \setF_k$, and in \cref{eq:fgs3} used $f(0)=0$ and again  $h_k(x)=0$ for all $x\notin \setF_k$. 
\end{proof}

\bibliographystyle{habbrv}
\bibliography{references}

\begin{thebibliography}{10}
\expandafter\ifx\csname url\endcsname\relax
  \def\url#1{\texttt{#1}}\fi
\expandafter\ifx\csname doi\endcsname\relax
  \def\doi#1{\burlalt{doi:#1}{http://dx.doi.org/#1}}\fi
\expandafter\ifx\csname urlprefix\endcsname\relax\def\urlprefix{URL }\fi
\expandafter\ifx\csname href\endcsname\relax
  \def\href#1#2{#2}\fi
\expandafter\ifx\csname burlalt\endcsname\relax
  \def\burlalt#1#2{\href{#2}{#1}}\fi

\bibitem{adfo03}
R.~A. Adams and J.~F. Fournier.
\newblock {\em {S}obolev {S}paces}, volume 140 of {\em Pure and Applied
  Mathematics}.
\newblock Elsevier/Academic Press, Amsterdam, Netherlands, 2nd edition, 2003.
\newblock \doi{https://doi.org/10.1016/S0079-8169(13)62896-2}.

\bibitem{amfupa00}
L.~Ambrosio, N.~Fusco, and D.~Pallara.
\newblock {\em Functions of {B}ounded {V}ariation and {F}ree {D}iscontinuity
  {P}roblems}.
\newblock Oxford Univ. Press, Oxford, UK, 2000.
\newblock \doi{https://doi.org/10.1093/oso/9780198502456.001.0001}.

\bibitem{amgisa08}
L.~Ambrosio, N.~Gigli, and G.~Savar\'e.
\newblock {\em {G}radient {F}lows in {M}etric {S}paces and in the {S}pace of
  {P}robability {M}easures}.
\newblock Birkh\"auser, 2nd edition, 2008.
\newblock \doi{https://doi.org/10.1007/978-3-7643-8722-8}.

\bibitem{ap76}
T.~M. Apostol.
\newblock {\em {I}ntroduction to {A}nalytic {N}umber {T}heory}.
\newblock Springer, New York, NY, 1976.
\newblock \doi{https://doi.org/10.1007/978-1-4757-5579-4}.

\bibitem{arbo17}
M.~Arjovsky and L.~Bottou.
\newblock Towards principled methods for training generative adversarial
  networks.
\newblock In {\em Proceedings of the International Conference on Learning
  Representations}, pages 1--15, 2017.

\bibitem{bate18}
B.~Bailey and M.~J. Telgarsky.
\newblock Size-noise tradeoffs in generative networks.
\newblock In {\em Advances in Neural Information Processing Systems},
  volume~31, pages 6490--6500, 2018.

\bibitem{ba93}
A.~R. Barron.
\newblock Universal approximation bounds for superpositions of a sigmoidal
  function.
\newblock {\em {IEEE} {T}rans. {I}nform. {T}heory}, 39(3):930--945, Mar. 1993.
\newblock \doi{https://doi.org/10.1109/18.256500}.

\bibitem{ba95}
R.~G. Bartle.
\newblock {\em {T}he Elements of Integration and {L}ebesgue Measure}.
\newblock Wiley, New York, NY, 1995.
\newblock \doi{https://doi.org/10.1002/9781118164471}.

\bibitem{elpegrbo21}
D.~Elbr\"achter, D.~Perekrestenko, P.~Grohs, and H.~B\"olcskei.
\newblock Deep neural network approximation theory.
\newblock {\em {IEEE} {T}rans. {I}nform. {T}heory}, 67(5):2581--2623, May 2021.
\newblock \doi{https://doi.org/10.1109/TIT.2021.3062161}.

\bibitem{evga14}
L.~C. Evans and R.~F. Gariepy.
\newblock {\em {M}easure {T}heory and {F}ine {P}roperties of {F}unctions}.
\newblock CRC Press, New York, NY, revised edition, 2015.
\newblock \doi{https://doi.org/10.1201/b18333}.

\bibitem{fa14}
K.~Falconer.
\newblock {\em Fractal Geometry}.
\newblock Wiley, New York, NY, 3rd edition, 2014.

\bibitem{fed69}
H.~Federer.
\newblock {\em {G}eometric Measure Theory}.
\newblock Springer, New York, NY, 1969.
\newblock \doi{https://doi.org/10.1007/978-3-642-62010-2}.

\bibitem{gopomixuwaozcobe14}
I.~J. Goodfellow, J.~Pouget-Abadie, M.~Mirza, B.~Xu, D.~Warde-Farley, S.~Ozair,
  A.~Courville, and Y.~Bengio.
\newblock Generative adversarial networks.
\newblock In {\em Advances in Neural Information Processing Systems}, volume~3,
  pages 1--9, Jun. 2014.

\bibitem{gukupe20}
I.~G\"uhring, G.~Kutyniok, and P.~Petersen.
\newblock Error bounds for approximations with deep {ReLU} networks in
  {$W^{s,p}$} norms.
\newblock {\em Analysis and Applications}, 18(5):803--859, 2020.
\newblock \doi{https://doi.org/10.1142/S0219530519410021}.

\bibitem{hidare97}
G.~E. Hinton, P.~Dayan, and M.~Revow.
\newblock Modeling the manifolds of images of handwritten digits.
\newblock {\em IEEE Trans. Neural Netw.}, 8(1):65--74, Jan 1997.
\newblock \doi{https://doi.org/10.1109/72.554192}.

\bibitem{hokr24}
R.~Hong and A.~Kratsios.
\newblock Bridging the gap between approximation and learning via optimal
  approximation by {ReLU} {MLPs} of maximal regularity.
\newblock {\em arXiv:2409.12335v1}, 2024.

\bibitem{lz12}
A.~J. Izenman.
\newblock Introduction to manifold learning.
\newblock {\em WIREs Comput. Stat.}, 4(5):439--446, 2012.
\newblock \doi{https://doi.org/10.1002/wics.1222}.

\bibitem{kawaiw24}
S.~Karnik, R.~Wang, and M.~Iwen.
\newblock Neural network approximation of continuous functions in high
  dimensions with applications to inverse problems.
\newblock {\em J. Comput. Appl. Math.}, 438(5):1--20, Mar. 2024.

\bibitem{kiwe14}
D.~P. Kingma and M.~Welling.
\newblock Auto-encoding variational {B}ayes.
\newblock In {\em Proceedings of the International Conference on Learning
  Representations}, pages 1--14, 2014.

\bibitem{kl12}
B.~Kloeckner.
\newblock A generalization of {H}ausdorff dimension applied to hilbert cubes
  and {W}asserstein spaces.
\newblock {\em Journal of Topology and Analysis}, 4(2):203--235, 2012.
\newblock \doi{https://doi.org/10.1142/S1793525312500094}.

\bibitem{ko57}
A.~Kosi\'nski.
\newblock A proof of an {A}uerbach-{B}anach-{M}azur-{U}lam theorem on convex
  bodies.
\newblock {\em Colloquium Mathematicae}, 4(2):216--218, 1957.
\newblock \doi{https://doi.org/10.4064/cm-4-2-216-218}.

\bibitem{krpa08}
S.~G. Krantz and H.~R. Parks.
\newblock {\em Geometric Integration Theory}.
\newblock {B}irkh{\"a}user, Boston, MA, 2008.
\newblock \doi{https://doi.org/10.1007/978-0-8176-4679-0}.

\bibitem{legemariar17}
H.~Lee, R.~Ge, T.~Ma, A.~Risteski, and S.~Arora.
\newblock On the ability of neural nets to express distributions.
\newblock In {\em Proceedings of Machine Learning Research}, volume~65, pages
  1--26, Jul. 2017.

\bibitem{letawi94}
W.~E. Leland, M.~S. Taqqu, W.~Willinger, and D.~V. Wilson.
\newblock On the self-similar nature of {E}thernet traffic (extended version).
\newblock {\em IEEE/ACM Trans. Netw.}, 2(1):1--15, Feb 1994.
\newblock \doi{https://doi.org/10.1109/90.282603}.

\bibitem{lufahe98}
H.~Lu, Y.~Fainman, and R.~Hecht-Nielsen.
\newblock Image manifolds.
\newblock In {\em Applications of Artificial Neural Networks in Image
  Processing III}, volume 3307, pages 52--63, Apr. 1998.
\newblock \doi{https://doi.org/10.1117/12.304659}.

\bibitem{ma99}
P.~Mattila.
\newblock {\em Geometry of Sets and Measures in {E}uclidean Spaces: Fractals
  and Rectifiability}.
\newblock Cambridge Univ. Press, Cambridge, UK, 1995.
\newblock \doi{https://doi.org/10.1017/CBO9780511623813}.

\bibitem{mc34}
E.~J. McShane.
\newblock Extension of range of functions.
\newblock {\em Bulletin of the American Mathematical Society}, 40(12):837--842,
  1934.

\bibitem{ngjo02}
A.~Ng and M.~Jordan.
\newblock On discriminative vs. generative classifiers: {A} comparison of
  logistic regression and naive {B}ayes.
\newblock In {\em Advances in Neural Information Processing Systems},
  volume~14, 2001.

\bibitem{pahubo24}
Y.~Pan, C.~Hutter, and H.~B\"olcskei.
\newblock Metric-entropy limits on nonlinear dynamical system learning.
\newblock {\em Information Theory, Probability and Statistical Learning: A
  Festschrift in Honor of Andrew Barron}, 2026 (to appear).

\bibitem{peebbo21}
D.~Perekrestenko, L.~Eberhard, and H.~B\"olcskei.
\newblock High-dimensional distribution generation through deep neural
  networks.
\newblock {\em Partial Differential Equations and Applications}, 2(5):1--44,
  May 2021.
\newblock \doi{https://doi.org/10.1007/s42985-021-00115-6}.

\bibitem{peprto92}
J.~E. Pe\v{c}ari\'{c}, F.~Proschan, and Y.~L. Tong.
\newblock {\em {C}onvex {F}unctions, {Partial} {Orderings}, and {S}tatistical
  {A}pplications}.
\newblock Academic Press Inc., Boston, MA, 1992.

\bibitem{rikostbo23}
E.~Riegler, G.~Koliander, D.~Stotz, and H.~B\"olcskei.
\newblock Recovery of matrices with low description complexity.
\newblock {\em Sampling Theory, Signal Processing, and Data Analysis}, 2026
  (submitted).

\bibitem{soze98}
N.~Sochen and Y.~Y. Zeevi.
\newblock Representation of colored images by manifolds embedded in higher
  dimensional non-{E}uclidean space.
\newblock In {\em {P}roc. {IEEE} {I}nt. {C}onf. on {I}mage {P}rocess.}, pages
  166--170, Oct 1998.
\newblock \doi{https://doi.org/10.1109/ICIP.1998.723450}.

\bibitem{yaliwa22}
Y.~Yang, L.~Zheng, and Y.~Wang.
\newblock On the capacity of deep generative networks for approximating
  distributions.
\newblock {\em Neural Networks}, 145:144--154, 2022.
\newblock \doi{https://doi.org/10.1016/j.neunet.2021.10.012}.

\end{thebibliography}
\end{document}
The following result states that on convex open domains, the set of Lipschitz functions  equals the set of local Sobolev functions. 
\begin{thm}\label{lem:WLIP}
Let $\setO\subseteq \reals^m$ be open and convex and consider the  function   $f\colon \setO \to \reals$. Then, the following properties hold. 
\begin{enumerate}
\renewcommand{\theenumi}{(\roman{enumi})}
\renewcommand{\labelenumi}{(\roman{enumi})}
\item\label{eq:item1lipW}
$f$ is Lipschitz if and only if $f\in W^{1,\infty}(\setO^\prime)$ for all open and bounded sets $\setO^\prime\subseteq\setO$;
\item \label{eq:item2lipW} if $f$ is Lipschitz, then 
 \begin{align}
\operatorname{Lip}^{(1)}(f)= \lVert \lVert \nabla f\rVert_\infty \rVert_{L_\infty(\setO)}.   
\end{align} 
\end{enumerate}
\end{thm}
\begin{proof}
The proof follows along the same lines as that  of  \cite[Theorem 3.31]{kixx}.  
Suppose that  $f$ is Lipschitz.  Then, $\nabla f$ exists (Lebesgue) almost everywhere  thanks to  Rademacher's theorem \cite[Theorem 5.1.11]{krpa08}. 
Moreover,  we have 
\begin{align}
\abs{\tp{e_i} \nabla f(x)}
&= \lim_{t\to 0} \frac{\abs{ f (x+t e_i) -f(x)}}{t} \label{eq:diffquitient} \\
&\leq  \operatorname{Lip}^{(1)} ( f)\quad\text{for almost all $x\in\setO$ and $i=1,\dots, m$,}
\end{align}
which implies 
\begin{align}\label{eq:Lip1}
\lVert \lVert \nabla f\rVert_\infty \rVert_{L_\infty(\setO)} \leq \operatorname {Lip}^{(1)}(f)
\end{align}
 and, in turn,  
$f\in W^{1,\infty}(\setO^\prime)$ for all open and bounded sets $\setO^\prime\subseteq\setO$. 

To establish the converse, suppose that $f\in W^{1,\infty}(\setO^\prime)$ for all open and bounded sets $\setO^\prime\subseteq\setO$. Arbitrarily fix $x_0,y_0\in \setO$ and pick an   open,  bounded, and convex set $\setO^\prime\subseteq\setO$  satisfying $x_0,y_0\in \setO^\prime$. Further,  arbitrarily fix a   $p\in(m,\infty)$. Since 
 $f\in W^{1,\infty}(\setO^\prime)$ and $\setO^\prime$ is bounded, \cite[Theorem 2.14]{adfo03} (applied to $f$ and the weak derivative of $f$)  implies 
$f\in W^{1,p}(\setO^\prime)$. Moreover, $\setO^\prime$ is a Lipschitz domain \cite[Definition 2.2]{dele04} owing to \cite[Lemma 2.3]{dele04} 
   so that \cite[Part II of Theorem 4.12]{adfo03} implies that $f$ has a 
representative that is  uniformly  continuous on $\setO^\prime$. Let  $\hat f$ be the zero extension of  that representative outside $\setO^\prime$. We next construct a  mollification of $\hat f$. To this end, 
set
\begin{align}
\eta (x)=
\begin{cases}
c\, e^{-1/(1-\lVert x\rVert_2)}&\text{if $\lVert x\rVert_2 <1$}\\
0&\text{else,}
\end{cases}
\end{align}
where $c\in(0,\infty)$ is chosen such  that $\int \eta \,\mathrm d\colL^m =1$. Further, 
for every $\varepsilon \in (0,\infty)$, set   $\eta_\varepsilon=\eta(\,\cdot\,/\varepsilon)/\varepsilon^m$, 
$\setO^\prime_\varepsilon=\{x\in\setO^\prime: \sup_{y\in \partial \setO}\lVert x-y\rVert_2>2\varepsilon\}$,  and 
$f_\varepsilon= \hat f\ast \eta_\varepsilon$, where $\ast$ denotes convolution.  Now,   
\cite[Item (e) of Theorem 2.29]{adfo03} implies that  $f_\varepsilon$ converges uniformly to $f$ on $\setO^\prime$. 
Next, note that 
\begin{align}
(\lVert \nabla f_\varepsilon \rVert_\infty)(x) 
&= \max_{i=1,\dots,m}  \abs{ \partial_{x_i} f_\varepsilon }(x)\label{eq:convder0} \\
&= \max_{i=1,\dots,m}  \abs{ (D _{x_i} f) \ast \eta_\varepsilon }(x)\label{eq:convder}  \\
&\leq \Big(\big(\max_{i=1,\dots,m} \abs{ D _{x_i} f }\big) \ast \eta_\varepsilon\Big)(x) \\
&= (\lVert D f\rVert_\infty\ast \eta_\varepsilon)(x)\quad\text{for all $x\in\setO^\prime_\varepsilon$}, \label{eq:convderA}
\end{align}
where \cref{eq:convder} is by \cite[Theorem 1.19]{kixx} with $D_{x_i} f$  denoting the weak derivative of $f$  and in \cref{eq:convderA} we set $Df=\tp{(D_{x_1}f,\dots,D_{x_m}f )}$. 
We therefore have 
\begin{align}
\lVert  \lVert \nabla f_\varepsilon \rVert_\infty  \rVert_{L_p(\setO^\prime_\varepsilon)}
&\leq \lVert  \lVert D f \rVert_\infty  \rVert_{L_p(\setO_\varepsilon^\prime)}  \lVert  \eta_\varepsilon \rVert_{L_1(\setO^\prime_\varepsilon)} \label{eq:convder2}\\
&\leq \lVert  \lVert Df \rVert_\infty  \rVert_{L_p(\setO)}  \quad\text{for all $\varepsilon \in(0,\infty)$,} \label{eq:convder3}
\end{align} 
where in \cref{eq:convder2} we applied Young's inequality for convolutions \cite[Corollary 2.25]{adfo03} in combination with \cref{eq:convder0}--\cref{eq:convderA}.  
 We can thus upper-bound 
 \begin{align}
\lVert  \lVert \nabla f_\varepsilon \rVert_\infty \rVert_{L_\infty(\setO^\prime_\varepsilon)} 
&=\lim_{p\to\infty}\lVert  \lVert \nabla f_\varepsilon \rVert_\infty  \rVert_{L_p(\setO^\prime_\varepsilon)}\label{eq:weakboundD}\\ 
&\leq \lim_{p\to\infty} \lVert  \lVert Df \rVert_\infty  \rVert_{L_p(\setO)}\label{eq:weakboundD1}\\
&\leq \lVert  \lVert D f   \lVert_\infty  \rVert_{L_\infty(\setO)}  \quad\text{for all $\varepsilon \in(0,\infty)$,} \label{eq:weakboundD2}
\end{align}
where in \cref{eq:weakboundD} and \cref{eq:weakboundD2} we applied \cite[Theorem 2.14]{adfo03} and  
\cref{eq:weakboundD1} follows from \cref{eq:convder2}--\cref{eq:convder3}  upon noting that 
 $p$ was arbitrary,  
Next, note that $\setO^\prime_{\varepsilon}$  is convex for all $\varepsilon\in (0,\infty)$. In fact, arbitrarily fix $\varepsilon\in(0,\infty)$ and set, for $z\in \setO^\prime$, $\setB(z,2\varepsilon)=\{x\in\setO^\prime: \lVert x-z\rVert_2\leq  2\varepsilon\}$. Arbitrarily pick $u,v\in \setO^\prime_{\varepsilon}$ and $t\in (0,1)$ and set $w=tu+(1-t)v$. Since every point in $\setB(w,2\varepsilon)$ is a convex combination of points in $\setB(u,2\varepsilon)$ and $\setB(v,2\varepsilon)$, convexity of $\setO^\prime$ implies $\setB(w,2\varepsilon)\subseteq \setO^\prime$ and, in turn, $w\in \setO^\prime_{\varepsilon}$ since $\setO^\prime$ is an open set.  
Now, let $\varepsilon_0\in (0,\infty)$ be sufficiently small so that $x_0,y_0\in\setO^\prime_{\varepsilon_0}$. 
Then, we have 
\begin{align}
\abs{f_\varepsilon(y_0) -f_\varepsilon(x_0)} 
&=\abs{\int_0^1 \tp{\nabla f_\varepsilon (tx_0 +(1-t)y_0) } (x_0-y_0)\, \mathrm d t }\label{eq:fundthm}\\
&\leq  \lVert  \lVert \nabla f_\varepsilon \rVert_\infty \rVert_{L_\infty(\setO^\prime_\varepsilon)}\lVert x_0-y_0\rVert_1\quad\text{for all $\varepsilon \in(0,\varepsilon_0)$},\label{eq:fundthm2}  
\end{align}
where \cref{eq:fundthm} follows from the fundamental theorem of calculus \cite[Proposition 1.6.41]{ta11} and 
in \cref{eq:fundthm2} we use the fact that $x_0,y_0\in\setO^\prime_{\varepsilon}$ implies $tx_0 +(1-t)y_0\in\setO^\prime_{\varepsilon}$ thanks to convexity of $\setO^\prime_{\varepsilon}$. 
 Taking the limit $\varepsilon\to 0 $ in \cref{eq:fundthm}--\cref{eq:fundthm2} and using  uniform convergence of $f_\varepsilon\to f$ on $\setO^\prime$ and  
\cref{eq:weakboundD}--\cref{eq:weakboundD2}, it follows that $\abs{f(y_0) -f(x_0)}\leq \lVert  \lVert D f \rVert_\infty \rVert_{L_\infty(\setO)}\lVert x_0-y_0\rVert_1$. Since $x_0,y_0$ were assumed to be arbitrary, we can conclude that $f$ is Lipschitz  with 
$\operatorname{Lip}^{(1)}(f)\leq \lVert  \lVert D f \rVert_\infty \rVert_{L_\infty(\setO)}$. We can replace $Df$ by $\nabla f$ thanks to Rademacher's theorem \cite[Theorem 5.1.11]{krpa08} so that  
\begin{align}\label{eq:Lip2}
\operatorname{Lip}^{(1)}(f)\leq \lVert  \lVert \nabla f \rVert_\infty \rVert_{L_\infty(\setO)}, 
\end{align}
which implies that $f$ is Lipschitz. Finally,  \cref{eq:item2lipW} follows from \cref{eq:item1lipW}, \cref{eq:Lip1}, and \cref{eq:Lip2}.

\end{proof}
